\documentclass[a4paper,11pt]{report} % [titlepage]{article}

\usepackage[T1]{fontenc}
\usepackage[utf8]{inputenc}
\usepackage{lmodern}

\usepackage{a4wide}
\usepackage[left=2.5cm, right=2.5cm, vmargin=2.5cm]{geometry}
\usepackage{fancyhdr}
\usepackage{amsmath} % e.g. column vectors (pmatrix)
\usepackage{amssymb} % \Beta (uppercase)
\usepackage{amsfonts} % Denoting sets; namely R & N.
\usepackage{amsthm} % Theorems

\usepackage[lined,boxed,algochapter]{algorithm2e} % srodowisko algorithm do typesettingu algorytmow.
\SetAlgoInsideSkip{smallskip} % Padding algorytmu w ramce.
\SetAlCapSkip{10pt} % Odleglosc miedzy boxem algorytmu, a caption.
\SetAlCapFnt{\small} % Czcionka slowa "Alogorithm" w caption.
\SetAlCapNameFnt{\small} % Czcionka opisu w caption.
\SetAlgoCaptionSeparator{ –} % Znak odstepu miedzy "Algorithm 4.1" ZNAK "caption".

\usepackage{wrapfig}
\usepackage{graphicx}
\usepackage{subfig} % \subfloat

\usepackage{pbox} % pbox

\usepackage[margin=10pt,font=small,labelfont=bf,
            labelsep=endash]{caption}

\usepackage{mdframed}
\newmdenv[leftline=false,rightline=false]{topbottom}

\usepackage{rotating} % sidewaysfigure

\usepackage{fancyvrb} % SaveVerbatim

\usepackage{varwidth} % Centrowanie verbatima.

\usepackage{multirow} % Tabulars
\usepackage{multicol}
\usepackage{xcolor,colortbl}

\usepackage[hyphens]{url}

\usepackage[breaklinks,unicode=true]{hyperref} % hyperlinks in the toc
\hypersetup{ % Fix for unicode chars in meta data.
  pdfauthor={Adrian {\L}a{\'n}cucki},
  pdftitle={GGP with Advanced Reasoning and Board Knowledge Discovery}
}

\usepackage{breakurl} % url wrapping

\usepackage{verbatim} % verbatiminput

\usepackage{enumitem} % indent the description environment
\setdescription{leftmargin=\parindent,labelindent=\parindent} % indent the description environment

\newenvironment{definition}[1][Definition]{\begin{trivlist}
\item[\hskip \labelsep {\bfseries #1}]}{\end{trivlist}}

\newenvironment{problem}[1][Problem]{\begin{trivlist}
\item[\hskip \labelsep {\bfseries #1}]}{\end{trivlist}}

\newcommand*\colvec[3][]{
    \begin{pmatrix}\ifx\relax#1\relax\else#1\\\fi#2\\#3\end{pmatrix}
}

\newtoks\rowvectoks
\newcommand{\rowvec}[2]{%
  \rowvectoks={#2}\count255=#1\relax
  \advance\count255 by -1
  \rowvecnexta}
\newcommand{\rowvecnexta}{%
  \ifnum\count255>0
    \expandafter\rowvecnextb
  \else
    \begin{pmatrix}\the\rowvectoks\end{pmatrix}
  \fi}
\newcommand\rowvecnextb[1]{%
    \rowvectoks=\expandafter{\the\rowvectoks&#1}%
    \advance\count255 by -1
    \rowvecnexta}
    
\newtheorem{theorem}{Theorem}[section]
\newtheorem{lemma}[theorem]{Lemma}

\makeatletter
\let\ps@plain\ps@empty
\makeatother

\begin{document}

\begin{titlepage}
%\newpage
%\thispagestyle{empty}
\hspace*{1.5cm}
\begin{minipage}[l]{\linewidth}
	\begin{minipage}[c]{.15\linewidth}
	    \includegraphics[width=2cm]{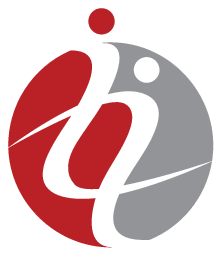}
	\end{minipage}\hfill
	\begin{minipage}[l]{.9\linewidth}
	    {\large \textsc{University of Wrocław\\
	    Faculty of Mathematics and Computer Science\\
	    Institute of Computer Science}}
	\end{minipage}
	
	\vspace{5cm}
	
	\vspace{0.5cm}
	
	\textsc{\huge GGP with Advanced Reasoning}
	\vspace{0.2cm}
	\\
	\textsc{\huge and Board Knowledge Discovery}
	\vspace{0.5cm}
	\\
	\textsc{by}
	\vspace{0.5cm}
	\\
	\textsc{\Large Adrian Łańcucki}
	\vspace{8cm}
	\\
	\hspace*{8.5cm}\textsc{M. Sc. Thesis}
	\vspace{0.5cm}
	\\
	\hspace*{8.5cm}\textsc{Supervisor:}\\
	\hspace*{8.5cm}\textsc{Dr Piotr Lipiński}
\end{minipage}
\vfill
\hspace*{6cm}\textsc{{\large Wrocław 2013}}
\end{titlepage}

%\newpage
%\thispagestyle{empty}
%\mbox{}

\newpage

\begin{abstract}
Quality of General Game Playing (GGP) matches suffers from slow state-switching and weak knowledge modules. Instantiation and Propositional Networks offer great performance gains over Prolog-based reasoning, but do not scale well. In this publication mGDL, a variant of GDL stripped of function constants, has been defined as a basis for simple reasoning machines. mGDL allows to easily map rules to C++ functions. 253 out of 270 tested GDL rule sheets conformed to mGDL without any modifications; the rest required minor changes. A revised (m)GDL to C++ translation scheme has been reevaluated; it brought gains ranging from 28\% to 7300\% over YAP Prolog, managing to compile even demanding rule sheets under few seconds. For strengthening game knowledge, spatial features inspired by similar successful techniques from computer Go have been proposed. For they required an Euclidean metric, a small board extension to GDL has been defined through a set of ground atomic sentences. An SGA-based genetic algorithm has been designed for tweaking game parameters and conducting self-plays, so the features could be mined from meaningful game records. The approach has been tested on a small cluster, giving performance gains up to 20\% more wins against the baseline UCT player. Implementations of proposed ideas constitutes the core of GGP Spatium - a small C++/Python GGP framework, created for developing compact GGP Players and problem solvers.
\end{abstract}

%\newpage
%\thispagestyle{empty}
%\mbox{}

% \setcounter{page}{4}
\setcounter{page}{3}
\tableofcontents

%When we write programs that "learn", it turns out we do and they don't.
%- \textit{unknown source}

\chapter{Introduction}

General Game Playing~\cite{gdl_specification} might be considered the first successful project, aiming at development of multi-game playing agents. With GGP courses being taught at universities, annual competitions and hundreds of publications, it is slowly making it's way as one of key fields in the AI domain. Perhaps the most exciting thing about GGP is the versatility - a GGP agent is assumed to learn a particular game and conduct a play based solely on a rule sheet, often under demanding time constraints. No human intervention is allowed.
\\

However, the flexibility comes at a price. No universal method of automatically conducting efficient computations based on the rule sheet exists. In the era of customizing the hardware to meet the needs of a problem, attempting heavy computations on a system, whose performance is far behind human-written code, seems largely missed. Fast pace of GGP matches also discourages development of thorough systems. It is usually the case that only simple, statistical data about the game is gathered. GGP agents rely on the recently developed MCTS technique~\cite{progressiveMCTSGo08}. Intense research in the field of MCTS-based computer Go allowed to, while breaking record after record, create agents of highly professional skill~\cite{grandChallenge12}. Such phenomenon did not occur in GGP, despite the same underlying MCTS method.
\\

This thesis addresses the aforementioned issues. A revised scheme~\cite{cpp09} of translating GDL game rules sheets to C++ code is evaluated and compared with other, existing methods. Spatial feature-based knowledge, inspired by that employed in computer Go~\cite{KNearestPatternsGo05, InriaMCTSPatternGo06, progressiveMCTSGo08}, is proposed for games featuring generalized boards. Evolutionary knowledge mining algorithm, allowing to find the features and fine-tune agent's parameters, completes the proposed approach.
\\

As the ability to generate efficient reasoning machines on-the-fly is crucial for GGP, such methods are constantly researched. The existing include: usage of Prolog engines, translation to OCaml~\cite{ocaml11} or C++ source code, game instantiation~\cite{instantiating10}, Propositional Networks~\cite{propnets09}. The last two give good results but, due to time and memory complexity, are currently limited to simple games only. Translation to source code seems promising in the long run; for that reason, this thesis introduces a small change to Game Description Language, bringing it closer to Datalog. With said change, the code generation approach will be simplified, and C++ code generation can be reevaluated.
\\

The state-of-the-art methods of extracting game knowledge from expert plays in Go rely on matching moves with various formations on the board. Those formations usually constitute piece arrangements, or express certain characteristics of a move, easily perceived by human players. The characteristics share essentially one thing: they rely on the concept of a board. Although experimental methods of automatic recognition of board structures and metrics in GGP exist~\cite{auto-heuristic-construction06, deducing-piece-args-kaiser-2007, distanceFeatures2012}, it seems the best to supply board information as meta data directly in the rule sheet, written in GDL itself. Such extension to GDL will be proposed, along with a system of a few spatial features and meta facts, definable on an arbitrary game board. They will aid in abstracting information about intended moves, and can be applied in the Go-like fashion to GGP.
\\

To generate meaningful, self-play game records, an evolutionary algorithm based on SGA~\cite{Goldberg:1989:GAS:534133} will be introduced. It will aid in developing various agent parameters and iteratively mining features. The computations are targeted for small clusters, as many matches are expected to be carried in the process. To run the evolutionary algorithm, time constrains have to be greatly relaxed.
\\

The approach is rather complex. Evaluation required implementing many tools; thus, GGP Spatium, a C++/Python GGP framework has been created. Because of a tepid interpretation of some limitations imposed by the GGP specification, GGP Spatium with it's capabilities constitutes also a testbed for general problem-solvers. It can be used to take advantage of the vast number of GGP resources (rule sheets, game servers, various agents, etc.) in otherwise not related research.
\\

This thesis is constructed in the following way: Chapter~\ref{chap:sota} will summarize recent techniques in multi-game playing, focusing on MCTS with knowledge, also in computer Go. Chapter~\ref{chap:problem} will define the problem considered in this thesis. In Chapter~\ref{chap:keyIdeas}, the approach created for this thesis to answer the problem will be presented. Chapter~\ref{chap:testingPlatform} will summarize technical details of GGP Spatium - the implementation of said approach. Finally, results obtained during tests will be presented in Chapter~\ref{chap:tests}, with Chapter~\ref{chap:summary} concluding the thesis.

\chapter{GGP/MCTS and Artificial Intelligence} % State of the art} 
\label{chap:sota}

In this chapter, an overview of state of the art in artificial intelligence in games is presented. The chapter begins with a brief history of AI. The main focus is on multi-game playing, and utilization of Monte Carlo methods supported with game knowledge in conducting plays. While the scope is narrowed to notable work relevant to the subject, a broader view on AI game playing, yet still concise, can be found in~\cite{Mandziuk10}.
\\

\section{AI Game Playing}

The idea of artificial intelligence comes from the 50s, and is nearly as old as computer science itself. The natural approach to research in the field of AI was to break up the problem into simpler and smaller ones. This is how AI game playing, among the other AI-related ones, emerged as a problem. Academic research focused mainly on developing one-game programs for board games like chess, checkers, variations of Tic-tac-toe, or later Go.
\\

It is not clear when the rise in artificial intelligence in games happened. However, for many, it is the year 1997, when Deep Blue, IBM's supercomputer with both software and hardware dedicated to playing chess, won with then world master Garry Kasparov. 
\\

\section{Multi-Game Playing} 

First notable Meta-Game Playing (Multi-Game Playing) systems come from the first half of the 1990s. The term itself was coined by Barney Pell~\cite{pell92}, though at the time, he was not the only one interested in the subject. Early Multi-Game Playing systems like SAL or Hoyle~\cite{hoyle01, sal93} were designed to play two-player, perfect-information board games – the kind that fits well with minimax tree searching scheme, which was then the state of the art method.
\\

No popular framework was available until the year 2005, when General Game Playing (GGP) originated at Stanford University. The whole idea of GGP revolves around a single document~\cite{gdl_specification}, created and maintained by the Stanford Logic Group, which covers important technical details such as:
\begin{itemize}
  \item Game Description Language (GDL) for describing general games, 
  \item architecture of a General Game Playing System,
  \item mechanics of a single play,
  \item communication protocol.
\end{itemize}

GGP is continuously gaining popularity, largely due to the aforementioned standardization, and also active popularization in form of frequent, open competitions. The most prominent one, the AAAI Competition~\cite{aaaiComp05}, is an annual competition which features a large money prize and gathers enthusiasts and scientists. Yet another factor, that contributed to GGP's success, is it's timing – the system has been created almost in parallel with development of the UCT algorithm, applicable to a wide class of problems, and flexible in applying modifications such as utilization of various forms of game knowledge.
\\ 

Although GGP systems are meant to learn the game, they usually do not do that. Designed to collect and immediately use ad-hoc data about the game, they rarely improve on previous plays, because the game is assumed not likely to appear again. For instance, rather than to keep game-specific data between consecutive plays,~\cite{knowledgeTransfer07} presents the idea of knowledge transfer. It is a form of extracting "general knowledge" from previously conducted plays and using it in future plays with, possibly different yet unknown, games.
\\

Numerous different directions in research in GGP were taken; non-trivial ones include using automated theorem proving for formulating simple statements about the game, and detecting sub-games and symmetries in games~\cite{symmetry09}. On the other hand, feature and evaluation function discovery, described in Subsection~\ref{subsec:evalFuns}, is popular among the top GGP agents. There are many ideas for enhancing the players by analyzing game rules, but they all seem to share one flaw: their impact is hard to evaluate, as many a time they tend to better plays in some games while getting in the way in others.
\\

Of course, the ultimate panacea for game-playing programs is computational efficiency. That is why speeding up the key components of agents also attracts significant attention, specifically efficient parallelization~\cite{Centurio10, parallel11} and optimization of reasoning engines, described in greater detail in Section~\ref{sec:reasoning}.

\section{Rollout-Based Game Playing}

Monte Carlo methods have been long unpopular in AI game playing, mostly because of poor results in comparison to minimax-like searching. For instance, an attempt to apply simulated annealing to Go, resulted in obtaining only a novice-level player~\cite{Brugmann93montecarlo}.
\\

A major breakthrough came in 2006 when UCT~\cite{Kocsis06banditbased}, a rollout-based Monte Carlo planning algorithm, was proposed and successfully applied to Go~\cite{mogoFirstGoUct06}. Later, UCT was generalized to a new tree-search method named Monte Carlo Tree Search~\cite{progressiveMCTSGo08}.
\\

To understand what makes UCT stand out, UCB1 has to be introduced beforehand.

\subsection{UCB1 and UCT}
\label{subsec:ucb}

UCB1 is a policy proposed as a solution to the following problem, formulated during World War~II:
\\
 
\textit{Suppose we have a $k$-armed bandit; each arm (lever) pays out with an unknown distribution. What strategy should be used to maximize the payout of multiple pulls?}
\\

The problem is also known as an exploration/exploitation dilemma, question of a philosophical origin, quite similar to the famous secretary problem~\cite{wiki:secretaryProblem}.
\\

UCB states for Upper Confidence Bound. As the name suggests, UCB1 policy~\cite{multiarmed02} gives an upper bound on expected regret (which occurs when not pulling the optimal lever). More specifically, for a $k$-armed bandit with arms of arbitrary reward distributions $P_1,\ldots,P_k$ with support in $[0;1]$, the expected regret after $n$ pulls is bounded by
\[ \left[ 8 \sum_{i:\mu_1<\mu^*}\left( \frac{\ln n}{\Delta_i} \right) \right] +\left(1+\frac{\pi^2}{3} \right)\left(\sum^k_{j=1}\Delta_j\right) ,\]
where $\mu_1,\ldots\mu_k$ are the expected values of $P_1,\ldots,P_k$ and $\mu^*$ is any maximal element in $\left\{\mu_1,\ldots,\mu_k\right\}$.
\\

The natural strategy, when having some information about the levers, would be to pull the most promising one constantly. Let $I=\left\{1,\ldots,k\right\}$ be the set of lever's indexes. The lever to be pulled in turn $k$ is being determined by the following formula
\begin{equation} \label{eq:ucb}
argmax_{i\in I}\left( v_i + C \sqrt{\frac{\ln n_p}{n_i}}\right),
\end{equation}
where $v_i$ denotes an average payoff associated with lever $i$, $n_p$ is the total number of pulls and $n_i$ is the number of pulls on lever $i$. The expression $\sqrt{\frac{\ln n_p}{n_i}}$ is called the UCB bonus. It's purpose is to accumulate over time for long ignored levers. With the Formula~\ref{eq:ucb}, the most promising lever is being chosen.
\\

UCT stands for UCB1 Applied to Trees. It is an anytime heuristic for making decisions by the analysis of decision trees. The algorithm works in a loop, where each iteration consists of three phases: shallow-search, random simulation and back-propagation. Additionally, each iteration marks a path from the tree root to a leaf, and back to the root again. 
\\

During the shallow search, next node is being selected using the UCB1 policy. Each node is being treated as a $k$-armed bandit, with it's children being the levers. 
\\

UCB1 is being used to select the most promising node. When the maximum depth is reached, a random simulation takes place from the chosen node to a leaf. Next up, a value of the leaf is back-propagated as a payoff for selected arms, that is, all the selected moves on the path from the current game state.

\subsection{Monte Carlo Tree Search}

Monte Carlo Tree Search is a generalization of UCT, where the shallow-search phase (called a selection phase) strategy follows a \textit{selection strategy} like UCB1, but not necessarily.
\\ 

\begin{figure}[h]
  \centering
      \includegraphics[width=1.0\textwidth]{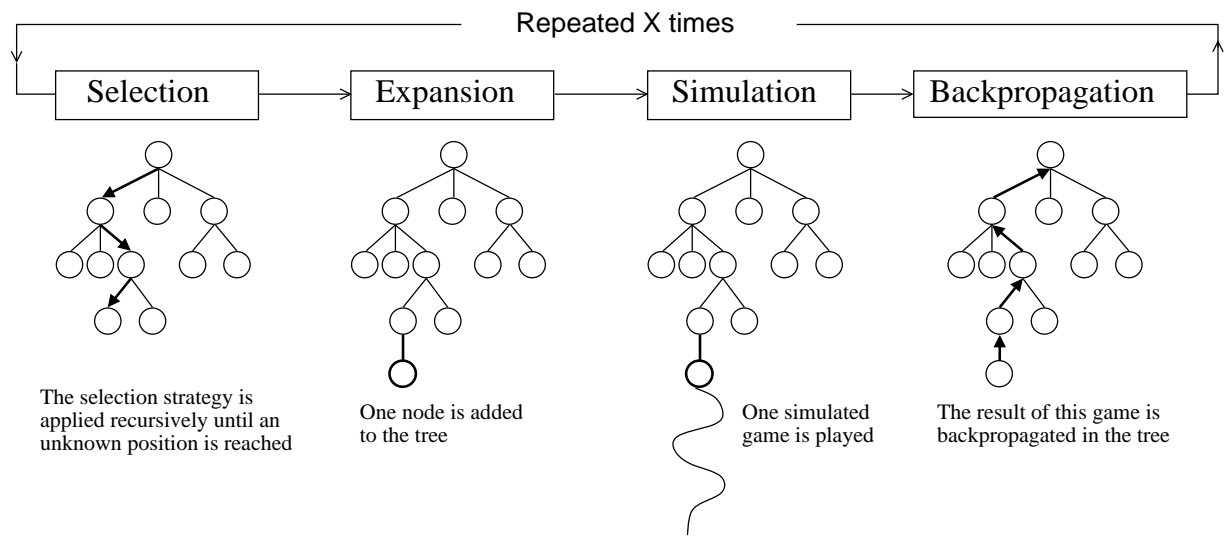}
  \caption{Outline of Monte-Carlo Tree Search~\cite{progressiveMCTSGo08}. Tree nodes are being scored to reflect average results of random playouts. To conserve resources, most promising playouts are pre-selected in the selection phase. The main loop might be interrupted at any given time.}
  \label{fig:mcts}
\end{figure}

The class of problems for which MCTS performs well is not yet thoroughly investigated, as it is dependent on the chosen selection strategy. Because of it's probabilistic nature, it is generally not suitable for single player deterministic games, which are essentially search problems. 
It performs well for a wide variety of other ones, specifically in dynamic environments.
It should be noted that bare MCTS, or even UCT variant of tree search, needs to be fine-tuned to a problem.
\\

Recent applications of MCTS to continuous, real-time environments with no easily distinguishable end states, give promising results, i.e. to Ms.~Pac-Man~\cite{mctsPacMan11} or Tron~\cite{tron10}. Another notable application of MCTS was obtaining automatic discretization of a continuous domain in Texas Holdem flavor of Poker~\cite{poker11}.

\section{The Game Description Language}

AI or random driven, almost every turn-based game playing agent has the logic component (also called the reasoner). It enables i.e. inference of consecutive game states or legal move sets. The reasoner has to be efficient, for it utilizes a great share of computational power through repeated switching of game states during a single play.
\\

The most straightforward solution for a multi-game player would be to generate the logic component on the fly, based on game rules. GGP uses Game Description Language for this purpose~\cite{gdl_specification}. GDL, a variant of Datalog, which is again a subset of Prolog, uses first-order logic for describing games. The class of games coverable by GDL\footnote{It is enough to include classic board games and even discretized versions of some arcade games.}
contains all finite, sequential, deterministic and perfect information games. In other words, the language describes any game for which a finite state automaton exists, where the automaton state corresponds to a single game state; a sample automaton is presented in Figure~\ref{fig:ga}.
\\

It is obvious that, due to the combinatorial explosion of their sizes, bare or even compressed automatons as structures are not feasible for passing information about the game. Therefore GDL is being used to define the automaton, by providing initial/terminal states and the transition function.
\\

\begin{figure}[h]
  \centering
      \includegraphics[width=0.6\textwidth]{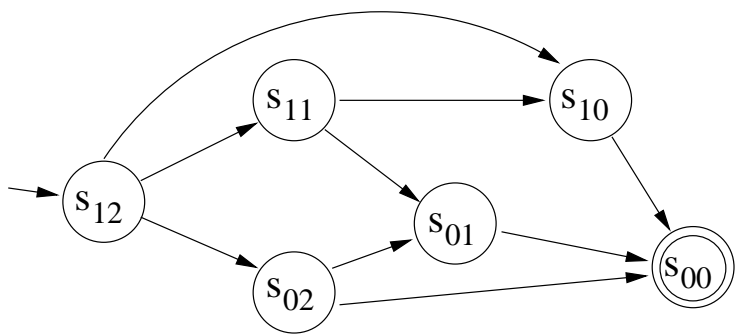}
  \caption{A sample game automaton for the game \textit{Nim} with initial stacks $(1,2)$. Each automaton state represents a game state, with $s_{12}$ being a start state, and $s_{00}$ an accept state. Transition to the next state occurs when players make their moves.}
  \label{fig:ga}
\end{figure}

While GDL is defined in terms of syntax and semantics, KIF~\cite{kifSpec} has been chosen to provide the syntax. KIF specification covers technical details, i.e. it provides a syntax in the modified BNF notation and defines the allowed charset. Sample excerpts from GDL rule sheets are presented in Subsection~\ref{subsec:prolog}.
\\

GDL rule sheets are, in a way, closely mimicking rule books used by human players. After all, real-life handbooks do not define GAs either; inference rules are written in natural language and extended with broad comments, with the rules relying on the common knowledge (like arithmetic or well-known conventions) simply omitted.

\subsection{Limitations}
\label{subsec:limitations}
For the time being, GDL can describe only finite, sequential, deterministic, and perfect information games. Though it is a wide class of problems, efforts have been made to extend it even further.
\\

One simple way to add a primitive support for non-deterministic games is to add a dummy random player to the game, whose only moves would be random calls separated from the rest of the game. Such player could, for instance, shuffle the cards or throw the dice. The GDL specification~\cite{gdl_specification} states that, whenever a player fails to connect, the gamemaster should carry on supplying random moves for the absent player. This way the dummy player could serve as an interface for the gamemaster to actually make random calls. However, most of today's players "safely" assume, that other players are opponents, and the dummy player would certainly be considered one. This would result in distorted perception of the dummy player's move distribution. In other words, such agents might continuously expect the worse number on the dice, worse possible cards dealt, etc.
\\

A GDL extension has been proposed~\cite{GDLExtension10}, adding a full support for both nondeterminism and information obscurity in a clear manner. The extension was a basis for GDL-II language~\cite{gdl2-11}, allowing to express games like Poker or Monopolly. It seems fairly popular\footnote{At the time of writing, Dresden GGP Server~\cite{dresdenGGPServer} holds $21$ GDL-II rule sheets, which account for $7\%$ of the total.}; however, this publication focuses on GDL alone.
\\

\begin{figure}[h]
  \centering
      \includegraphics[width=0.5\textwidth]{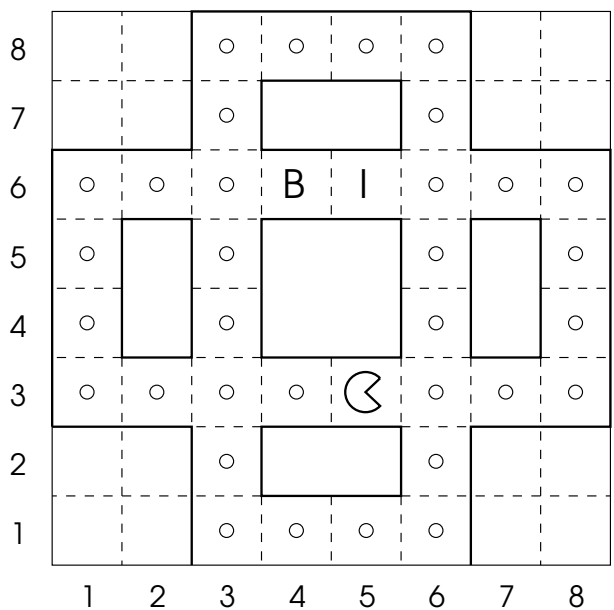}
  \caption{Initial maze setup of \emph{pacman3p.kif}~\cite{dresdenGGPServer} - a quantized, GDL variation of the popular Ms.~Pac-Man arcade game. The rules have been heavily simplified in respect to the original version, so the game could be easily expressed in GDL. Tokens $B$ and $I$ represent the pieces of two players whose goal is to catch the third one. Every player moves at the same speed of one square per move.}
  \label{fig:pacman}
\end{figure}

Computer representation of continuous domains is a matter of proper quantization. The same applies to a game-state space; though without specific language constructs, simple versions of a few “continuous” games have been expressed in GDL. However, accurate implementations of such games are not yet feasible with GDL. Ms. Pac-Man's GDL version is a good example of such problem. The simplified maze used in it is depicted in Figure~\ref{fig:pacman}.
\\

The game involves three players; one controls Pac-Man and the other two control the ghosts. The play takes place on $8\times 8$ grid; both the player and the ghosts share the same speed of 1 square per turn. This is not the case in the real Ms. Pac-Man; the speed of ghosts varies from 40\% to 95\% of Pac-Man's top speed with many factors contributing~\cite{PacManDossier}, not to mention Pac-Man's speed changing as well.
\\

A simple approach might be to quantify the state space in such way, that the turn would change only junctions. However, with such infrequent board refresh rate, the monsters might seem as randomly changing their position, what in turn could affect the learning engine. Developing reliable Ms. Pac-Man GDL rules could result in a heavily polluted state space.
\\

Another shortcoming of GDL is lack of basic arithmetic. The common way of bypassing this dilemma is defining simple arithmetic operators only when necessary, as it is shown in Figure~\ref{fig:PrologAddition}. Apart from doing unnecessary computations, it also makes the size of the overhead dependent on the data - in the example, complexity of operation $a+b$ would be $O(f(b))$, instead of $O(1)$. Thus, adding large numbers might result in unexpected slowdowns, i.e. in the middle of the game, when $b$ becomes suddenly large.
\\

\begin{figure}[htbp]
\begin{SaveVerbatim}{VerbPrologAddition}
(++ 1 0 1)
(++ 1 1 2)
(++ 1 2 3)
(++ 1 3 4)
(++ 1 4 5)
   ...
(++ 1 9 10)

(<= (++ 2 ?x ?z)
    (++ 1 ?x ?y)
    (++ 1 ?y ?z))
(<= (++ 4 ?x ?z)
    (++ 2 ?x ?y)
    (++ 2 ?y ?z))
(<= (++ 3 ?x ?z)
    (++ 1 ?x ?y)
    (++ 2 ?y ?z))
\end{SaveVerbatim}
	
	\centering
    \setlength{\fboxsep}{5mm}
    \fbox{\BUseVerbatim{VerbPrologAddition}}
    \caption{Definition of simple addition operation in GDL, $++: \{1\ldots4\}\times\{0\ldots9\}\mapsto\{1\ldots10\}$, allowing to add $1, 2, 3$ or $4$ to a positive integer with the resulting sum limited to 10.}
    \label{fig:PrologAddition}
\end{figure}

\section{GGP-Specific Knowledge}

The following section explores various interesting attempts of enriching GGP agents with game knowledge. The usual learning methods might not fit with GGP, either because of limited time, or uncertainty of the target game.

\subsection{Data Mining}

	An interesting attempt has been made in order to explore the possibility of replacing game tree search algorithm completely by a decision tree in~\cite{hSheng11}. The agent tries to gather some basic knowledge about the game by analyzing self-play history; using statistical analysis, sub-goals, crucial to winning or loosing, are being identified. Those are, in fact, just regular ground atomic sentences (facts) like \texttt{board(1,3,x)}. In the mining phase, the C4.5 Algorithm~\cite{c4.593}, which is an extension of the ID3 algorithm, is being used to create a sub-goal decision tree. Each sub-goal has a statistical winning ratio associated with it. The tree's job is to, for a given game state (a collection of facts) as an input, classify it to one of the outcome buckets. The approach seems to be useful for games with a high branching factor.

\subsection{Patterns}

	Though moderately popular in AI Game Playing, pattern recognition and matching has not attracted much attention of the GGP community. Usually, patterns are hard-coded, or at least their learning is supervised; any of that is not feasible in GGP. For discovery and application of quality patterns is heavily resource consuming, the documented attempts are simple. In~\cite{GIFL11} GIFL (GGP Feature Learning Algorithm) was proposed. It is an algorithm, which tries to find patterns described as predicate sets with associated moves and expected outcomes. GIFL finds two types of patterns: offensive and defensive, which are correlated with success and failure, respectively.
\\

GIFL is based on random playouts. Whenever a move $M$ leads from state $S_{i}$ to a winning, terminal state $S_t$ with payout $p_{S_t}$, state $S_i$ is being examined. Facts are being removed from that state, one at a time, and after every removal, the agent checks if applying move $M$ still guarantees payout equal to $p_{S_t}$.
\\

This way, the algorithm approximates the minimal subset of predicates that contributes to winning. Patterns are also being recognized in other states along the path in similar fashion, though they are being weighted according to the formula:
	\[ 100\cdot V^{level-1}, \]
where $V$ is an empirically set constant, and $level$ denotes the state's level in the tree.

\subsection{Payout Association}

The easiest approach, and perhaps the most natural one (considering representation of games in GGP), is associating statistical quality values with facts. For instance, CadiaPlayer utilizes the Move-Average Sampling Technique~\cite{CadiaSearchControl11}, where an action lookup table holds $\mathcal{Q}_h(a)$ value for each encountered move. It is an average payout of game during which the move was made. The values are updated with each subsequent playout. The rationale is to identify profitable actions, independent from game states.
\\

Similar approach was taken in~\cite{knowledgeGenGGP08}, albeit move payout was also dependent on a position during the simulation. Simply speaking, for board games with homogeneous pieces, this technique allows to discover the most favorable positions. The mentioned paper also explored a similar concept of evaluating facts, describing game states.
\\

Those simple ideas are not always new; in fact, some of them have been reused, as good enough for GGP. For instance, in 1993 Br\"{u}gmann~\cite{Brugmann93montecarlo} described a Monte Carlo scheme for evaluating possible moves.

\subsection{Recognition of Boards and Metrics}
\label{subsec:boardDiscovery}

As already mentioned in Subsection~\ref{subsec:mgdlTest}, GDL lacks basic arithmetic and, in consequence, obscures relations between different components of the game, that are naturally perceived by human players. A prime example of such connection is a concept of a board. It is absent in GDL; without it, facts which describe the board are mixed with those which describe temporary scores, count moves, etc. and form a loosely coupled set. But perceiving a board requires perceiving the distances between the pieces, and that is where having a metric defined is crucial. It is matter of discussion if tying facts together to form an abstract structure would be particularly useful, especially for games that do not feature a rectangular board, or even any board or pieces at all. Despite that, some of the top GGP agents use techniques of inference of metric features and even primitive recognition of board-like structures.
\\

One notable approach was taken in~\cite{auto-heuristic-construction06}, later used also in~\cite{Schiffel07automaticconstruction}. It consisted of identifying syntactic structures through analysis of game rule sheet. All ternary predicates have been assumed to describe two dimensional boards of a grid structure. More specifically, each ternary predicate described a board under one condition. Two arguments were assumed to be the coordinates and one to be the piece. The condition was that one of the arguments could never have two different values simultaneously (what would mean that there are two different pieces occupying the same field). The property also allowed to assume which argument was the one describing the field, and was verified by running self-play simulations.
\\

Kaiser~\cite{deducing-piece-args-kaiser-2007} used a similar approach of discovering board-like structures through self-play. Facts were grouped by relation, based on their functor/arity, i.e. \texttt{cell/3} formed a single group. All constants were labeled with unique numbers. Each relation's argument position (in case of \texttt{cell/3} there are three) was examined in terms of variance throughout the game. Retaining the values throughout the game would yield a 0 variance, contrary to frequently changing the values from state to state, referred to as mobility. Lastly, a motion detection heuristic settled which mobile arguments qualified as those denoting the pieces. The approach works best with rule sheets denoting empty fields in states, common in games where player moves a piece (chess, checkers) rather than puts a new one (Tic-tac-toe, Othello).
\\

The aforementioned papers note that, due to unpredictable ordering and obfuscation of predicates describing the board, those cannot be recognized by simply following naming patterns like \texttt{(cell ? ? ?)}. However, they may constitute a light an easy to implement alternative. Apparently, this approach is not universal as it relies heavily on the GDL representation of the game. Again, this is a clear shortage of GDL, since basic metadata supplied with a rule sheet would rule out all the potential errors from badly recognized metrics.

Board-like structures are typically used in conjunction with metrics. Below are presented a few examples of such, coming from different contexts.
\\

Manhattan distance on the board was used in~\cite{Schiffel07automaticconstruction} as a way to obtain new state features. Two kinds of features were proposed:
\begin{itemize}
\item Manhattan distance between each pair of pieces,
\item sum of pair-wise Manhattan distances.
\end{itemize}

Yet another idea was used in~\cite{distanceFeatures2012}. It suggested “proving” the metric by analysis of game rules. By inspecting relations between fluents, distance between two arbitrary predicates may be obtained (in terms of turns), that is, how many turns have to pass between the last occurrence of one term and the first one of the other. Taking chess for example, fact \texttt{board(8,1,rook)} can be obtained from \texttt{board(1,1,rook)} with a single move, whereas \texttt{board(8,8,rook)} requires at least two moves. The distances were calculated using the formula

\[ \delta(s,f)=\frac{\Delta(s,f)}{\Delta_{max}(f)}, \]

where $\Delta(s,f)$ denotes a shortest distance between a fluent $f$ and any of the fluents in state $s$, while $\Delta_{max}(f)$ is the longest possible distance $\Delta(s,f)$ that is not infinite. The paper states that this type of metric is always possible to obtain but not always feasible due to memory- and time- consuming computations.
\\
 
The impact of such a metric on the computations may be harder to foresee. 
In example, an introductory chess tutorial~\cite{chessTutorial04} suggests keeping the rooks close to the board center. The concept of a center relies on the Euclidean metric; it could be defined as a place equally distant from the boundaries. The definition of the board center through the “rook metric” however is not that straightforward. Of course at present, no real algorithm relies on tips written in natural language, albeit some designers of algorithms do.
\\

\cite{CluneHeuristicFunctions07} proposed the \textit{symbol distance}, an interesting way of capturing relations that follow syntactical patterns. Simply put, each binary relation is assumed to introduce an ordering of involved object constants through an undirected graph:
\begin{itemize}
\item each object constant $c$ creates a vertex $v_a$, 
\item two constants bound together with a functor (like \texttt{rel(c,d)}) introduce an edge $\{v_c,v_d\}$,
\item the distance is defined as the shortest path between the corresponding vertices.
\end{itemize}
The approach works very well with predicates like \texttt{succ}:
\begin{verbatim}
    (succ 1 2)
    (succ 2 3)
    (succ 3 4)
       ...
    (cell 1 2 b)
    (cell 3 4 b)
\end{verbatim}
Finally, the distance between fluents is a sum of distances between object constants on positions identified as distance-relevant. In the above example, if only the first two arguments would be identified as positional arguments, then the distance between \texttt{(cell 1 2 b)} and \texttt{(cell 2 4 b)} would be equal to the sum of the shortest paths $v_1\rightsquigarrow v_3$ and $v_2\rightsquigarrow v_4$, which would be $2+2=4$.

\subsection{Evaluation Functions}
\label{subsec:evalFuns}

	The point of using UCT in AI Game Playing is to have a generic way to evaluate game states. However, having a quality evaluation function at hand can lead to better performance of a UCT player. Experiments carried in~\cite{amazonsDiscover08}, where UCT algorithm was fitted nearly as a replacement for a minimax engine in the Game of Amazons agent, yielded a twofold benefit. Firstly, evaluation function can save time on random simulations - intermediate states evaluated with high certainty can be treated like terminal states. However, as the paper underlines, it might not pay off to evaluate every state along the path. Evaluation function might be used just once after a certain number of moves (counted from beginning of the game, or the beginning of the simulation). The second benefit is the ability to do forward pruning in games having great branching factor, where a high percentage of poorly evaluated moves might be discarded from building the initial UCT tree.
\\

A few of the leading GGP Agents incorporate different ideas for generating evaluation functions~\cite{fluxplayer07, CluneHeuristicFunctions07}. The first one, for instance, scored game rules' atoms with fuzzy logic, so they could serve as a basis for scoring complicated formulas. However, neither of them used UCT, as the corresponding agents were equipped with iterative-deepening depth-first-search algorithms.
\\

Lastly, an interesting approach was taken in~\cite{neuralNets09}, where a neural network was utilized for state evaluation. A propositional domain theory, obtained from the game description, was passed to $C-IL^2P$ algorithm~\cite{garcez02}, which returned a ready-to-use network. The approach has been further improved in~\cite{neuralNets11}.

\section{MCTS with Knowledge}
\label{sec:goKnowledge}

Recent Go playing agents, which leverage MCTS, are a great resource for pattern recognition and matching methods, tailored specifically to Monte Carlo simulations. Those agents give a good overview of what can be expected of pattern systems in terms of design, usage and performance. Though most of the ideas from Go are not generally applicable to every game, as they are mostly based on expert games and specific Go features, knowledge-free research in that area has been also carried.
\\

	Suiting MCTS to one particular game gives a finer control over minor tweaks to the MCTS method. Kloetzer et al.~\cite{MCAmazons07} summarized (the paper includes further references) the typical improvements as :
\begin{itemize}
\item using knowledge in the random games,
\item changing the behavior of UCT or using other techniques in the tree-search part, 
\item changing the behavior at the final decision of the move, by pruning moves.
\end{itemize}

In the following sections, selected state-of-the-art approaches to refining MCTS are presented. Figure~\ref{fig:go-quality} shows some of the results obtained by programs utilizing similar methods.

\begin{figure}[h]
  \centering
      \includegraphics[width=0.9\textwidth]{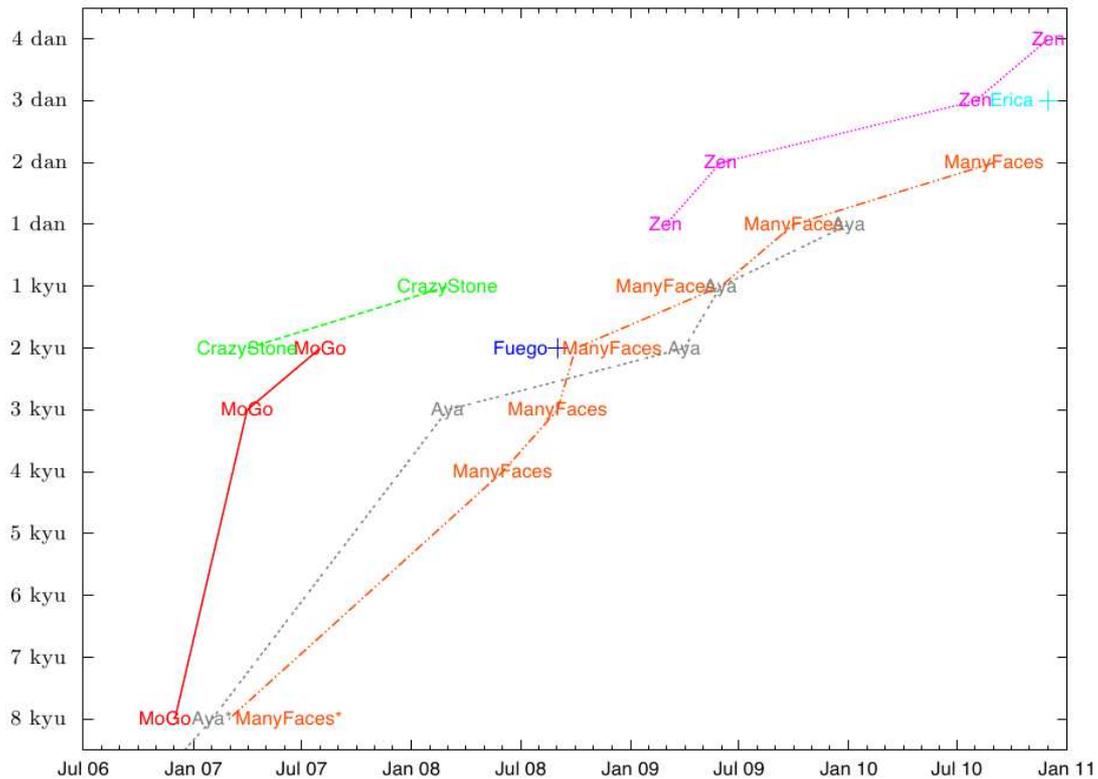}
  \caption{Dates at which several strong MCTS programs on the Kiseido Go Server achieved the given ranks~\cite{grandChallenge12}. Pre-MCTS versions are marked with an asterisk\*.}
  \label{fig:go-quality}
\end{figure}

\subsection{Features and Spatial Patterns}

Many mature examples of pattern exploration and matching come from chess and Go, games complicated enough, to resist any “lighter” computational intelligence methods. Both were, and still are, challenging. Patterns usually consist of spatial arrangements of pieces or game-specific features.
\\

In 2002, \cite{ComputerGo02}~noted, that mimicking human pattern skills seems out of reach on modern hardware, for the computational power required for efficient processing of patterns during game play is too great a burden. Only simple patterns are feasible to explore.
\\

$3\times 3$ spatial patterns, being a reasonable trade-off between size and carried information, are popular among Go researchers~\cite{bouzyGoPatterns03, progressiveMCTSGo08, goDiversity11, goMovePatternTrends12, InriaMCTSPatternGo06}. Such patterns are usually piece arrangements on a $3\times 3$ area around the field where a move is to be made. There are roughly $3^8$ of those (with extra patterns for moves close to the edges). \cite{bouzyGoPatterns03, progressiveMCTSGo08}~report a significant improvement when using $3\times 3$ patterns, and the latter suggests future increase in the size of patterns as greater computational power and memory sizes will become available.
\\

Similar approach, inspired by solutions from Moyo Go Studio~\cite{MoyoGoStudio12} and others, was taken in~\cite{BayesianPatternRankingGo06}, where patterns were also build around the field where a move was to be made. Nested sequences of patterns were considered, ordered by their sizes. Simply speaking, the larger the matched patterns (in terms of number of fields), the greater predictive power was associated with it. The smallest patterns were matched last, as the most universal and least predictive ones. Additional local features, bit-coded, were also added to the patterns. Those were mostly Go-specific, but a distance of the move to the board edge was also included.
\\

In~\cite{KNearestPatternsGo05} application of the common, K-nearest-neighbor pattern representation to Go was investigated. Again, the pattern was build around the field where a move was to be made, and the $k$ nearest fields to it (in respect to the Euclidean metric), that contained game pieces or board edges were relevant to the pattern and stored. The Bayesian properties of patterns were used to filter out meaningless ones: pattern $p$ was kept in the database if $P(i|p) > 0.01$, where $P(i|p)$ is the probability of playing a move on intersection $i$, provided pattern $p$ matches on $i$. A corpus of 2000 professional players' 19~x~19 Go games has been analyzed in order to mine significant spatial features. The generated databases (varying in size from 8,000 to 85,000 for $k$ from 6 to 15) were used in the Indigo Go program.
\\
	
	Among the popular features used in Go agents, that do not rely on Go-specific knowledge, are distance of piece to the nearest board edges, distance to the last move, and the move before the last move. All three rely on the Euclidean metric; the last two, called “proximity features”, depend on locality of Go~\cite{InriaMCTSPatternGo06}, an assumption that it usually pays out to make a move close to the intersection, where the last move was made.

\subsection{Knowledge Incorporation}
\label{subsec:incorporation}

With knowledge prepared beforehand, a simulation-based engine can be enhanced, for instance, by:
\begin{itemize}
\item influencing probability distribution when picking next move,
\item using the knowledge in the expert system manner (matching a rule would result in picking a move without any randomness)
\item narrowing down the choice of moves,
\item search-free playing when no reliable simulation data is available,
\item relying more on knowledge than average move payouts, until the payouts become reliable.
\end{itemize}

K-nearest-neighbor databases, mentioned in the previous section for the Indigo program, were used in two ways: for choosing a move in the opening, and for preselecting moves for the MC module during regular play. The latter is more relevant to the subject; during each turn, the knowledge module was used only to preselect a fairly small number $ns$ of moves (reported $ns=7$) for further evaluation by the Monte Carlo module. Mined k-nearest-neighbor patterns were used to select $nk$ best moves, leaving $ns-nk$ to be selected by the regular knowledge module. In the experiments conducted by the authors, the best results were achieved for $nk=2$.
\\

MoGo~\cite{InriaMCTSPatternGo06}, one of the top Go players, is a great example of how UCT can make use of identified patterns.	 The claim presented in the paper was that the quality of random simulations is poor, because pure randomness results in mostly meaningless playouts, skewing the estimated move value. It is worth noting that the meaningless playout effect might not be that obvious in simpler games of lower branching factors. A search-free method, of matching against predefine rule set, was used during random simulations. The first matching rule selected the move:
\begin{itemize}
  \item If the last played move was an Atari (and the stones could be saved), one of the saving moves was chosen randomly.
  \item The 8 fields surrounding the last move were compared against the encoded patterns. If some patterns were matched, one of the matching moves was chosen randomly. The patterns represented common Go formations.
  \item Any move that captured stones on the whole board was chosen.
  \item Finally, if all the previous heuristics failed, a truly random move was played.
 \end{itemize}

Progressive strategies, proposed in~\cite{progressiveMCTSGo08}, aim at lowering the costs of using knowledge in Monte Carlo Tree Search and, though tested with Go, are game-independent. The first one, progressive bias, introduces a modification to the UCT formula (designations remain the same as given in Formula~\ref{eq:ucb}):

\begin{equation}
\label{formula:progressiveUCT}
k=argmax_{i\in I}\left( v_i + C\sqrt{\frac{\ln n_p}{n_i}}+f(n_i)\right)
\end{equation}

The point is to rely on the search knowledge, while there is not enough data available from simulations. It follows from the formula that progressive bias converges over time to regular UCT.
\\

The second method, “progressive unpruning”, tries to cope with high branching factor of tree nodes. UCT tries to be fair and, by default, initially visits each node's child at least once. For hundreds or thousands of children this might not be possible withing the given timespan. In progressive pruning, after a fixed number visits to the node, the children are being pruned, according to the domain knowledge. Afterwards, the children nodes are becoming gradually “unpruned”. The scheme is depicted in Figure \ref{fig:widening}.

\begin{figure}[h]
\centering
\begin{tabular}{cc}
\subfloat[Moves are played according to the \textbf{simulation strategy}. All moves can be played.]
{\includegraphics[width=0.47\textwidth]{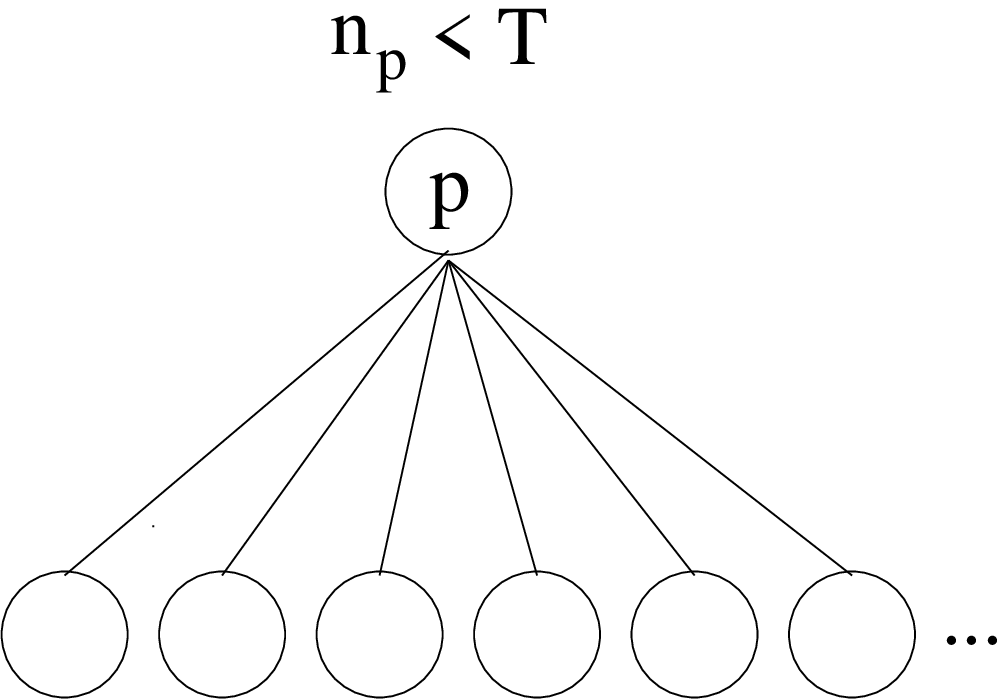}} 
\hspace{0.5cm}
\subfloat[The domain knowledge is called. Most of the moves are prunned.]
{\includegraphics[width=0.47\textwidth]{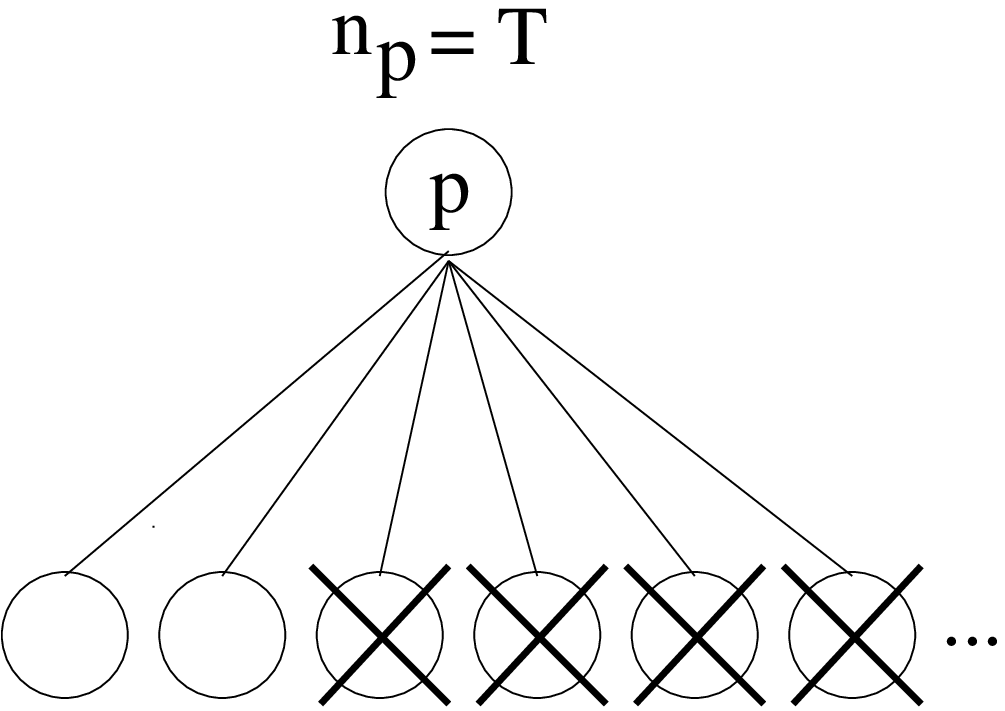}} \\
\\ %\vspace{0.2cm}\\
\subfloat[Moves are played accordint to the \textbf{selection strategy amongst the unpruned moves}. Moves are progressively unprunned.]
{\includegraphics[width=0.47\textwidth]{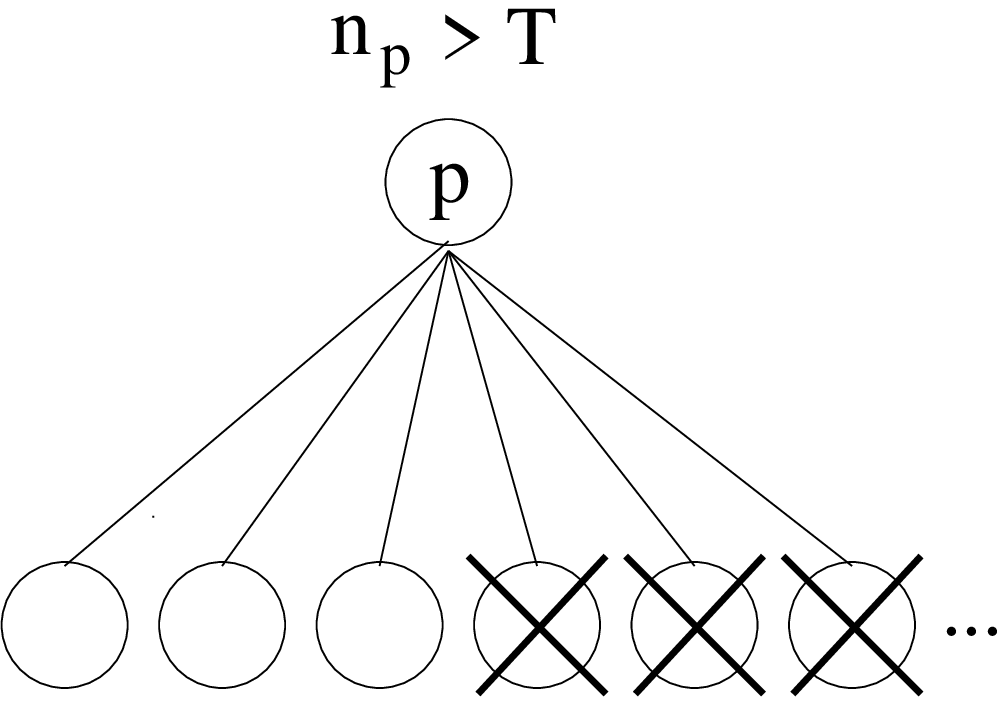}} 
\hspace{0.5cm}
\subfloat[Moves are played accordint to the \textbf{selection strategy amongst the unpruned moves}. Moves are progressively unprunned.]
{\includegraphics[width=0.47\textwidth]{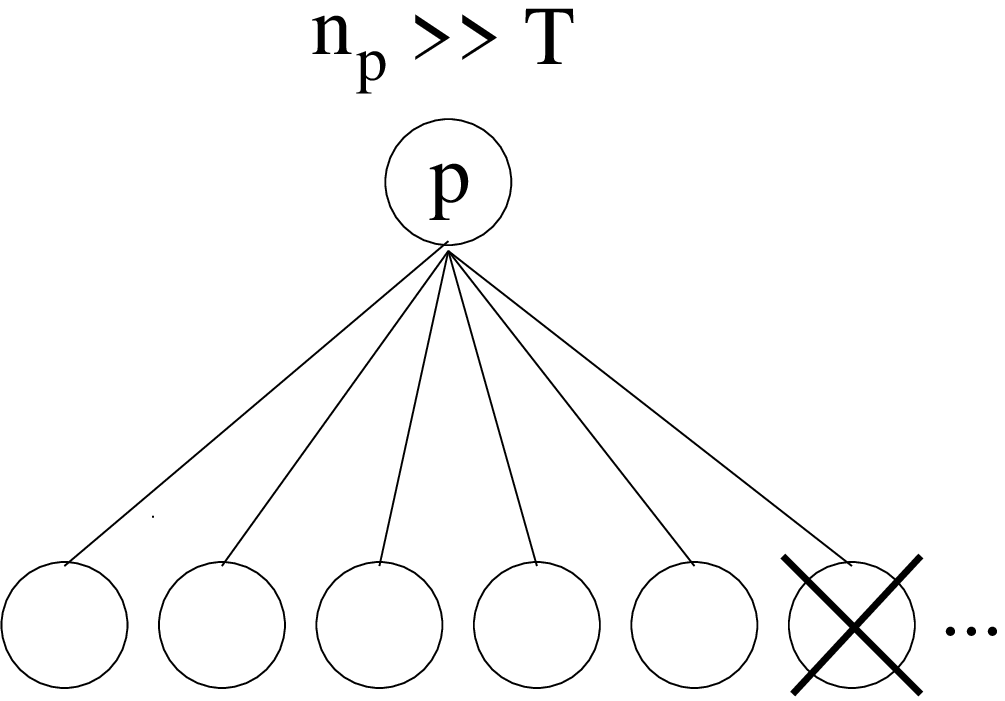}} \\
\end{tabular}
\caption{Progressive unprunning~\cite{progressiveMCTSGo08}}
\label{fig:widening}
\end{figure}

Cadiaplayer employs simulations affected by the previously mentioned MAST tables of move payouts. Instead of random, the moves are chosen according to the Boltzmann distribution

\[ \mathcal{P}(a)=\frac{e^{\mathcal{Q}_h(a)/\tau}}{\sum^n_{b=1} e^{\mathcal{Q}_h(b)/\tau}} .\]

\paragraph{Diversity}

An interesting extension of the knowledge-based MCTS scheme was proposed in~\cite{goDiversity11}. The paper investigates the influence of diversity in decision making, concept taken from social studies, where a team of diverse people is believed to cope better with difficult problems than a team of highly talented, but similarly-skilled individuals. As it goes for GGP agents, a multi-agent MCTS Go system was proposed; in each playout phase, a move to be made is being suggested by a randomly chosen agent from the agent database. The agents are equipped with identical pattern and feature sets, only weighted differently. It is an interesting phenomenon, that each of those agents alone may have worse results than a diverse group of agents (that is, group of agents performing good and relatively poor). Another thing noted is, that diversity alone does not guarantee a good performance - a diverse set has to be selected carefully. Thus, a greedy algorithm was used to select such set. Even better results were obtained when the agents were initially ordered.

\paragraph{Elo rankings}

	An inquiring idea was explored in~\cite{EloGo07}. The author suggested that patterns, once extracted, can be further evaluated using the popular Elo ranking system~\cite{elo1978rating}. Each pattern is considered a separate contestant, and selecting a move during a play is a team win of all the participating patterns. Though patterns were obtained with an expert knowledge (game records), the learning process is game-independent. As the author suggests, the whole system can be easily extended to other games. The patterns are being used with UCT in two ways. Firstly, the most lightweight features (in terms of computation time) direct the random search phase, for they provide probability distributions for moves. Secondly, the full set of features is used to prune the Monte-Carlo Search Tree.

\section{GDL Reasoning Engines}
\label{sec:reasoning}

\subsection{Prolog}
\label{subsec:prolog}

Different implementations of Prolog, due their ease of use and almost natural translation from KIF to Prolog, quickly became the  standard components handling the logic part of GGP agents. The vast majority of top GGP contestants also do, or used to, handle inference with Prolog: Ary~\cite{Ary10}, Fluxplayer~\cite{fluxplayer07}. Cadia~\cite{CadiaPlayer09} reports using YAP Prolog, Centurio~\cite{Centurio10} used ${}ECL^iPS^e$ as a base logic engine, which was further extended with generated Java reasoning code.
\\

The Knowledge Interchange Format specification~\cite{kifSpec} covers a Prolog-like, infix syntactic variant of KIF, which facilitates a mellow switch from infix to prefix notation. For example, rule
\begin{verbatim}
    (<= (legal xPlayer (play ?i ?j x))
        (true (control xPlayer))
        (emptyCell ?i ?j))
\end{verbatim}
could be translated as
\begin{verbatim}
    legal(xPlayer, play(I, J, X)) :- state(control(xPlayer)),
                                     emptyCell(I, J).
\end{verbatim}
Technical details of the translation may be dependent on a particular Prolog implementation.
\\

Bj{\"o}rnsson and Finnsson have created useful macros (as Prolog rules)~\cite{cadiaplayerSrc} to simplify the communication between the Prolog engine and the agent. Exemplary rules from the set are:
%  A few of those rules were shown on~\ref{fig:CadiaMacros}.
 
\begin{verbatim}
    distinct( _x, _y ) :- _x \= _y.
    or( _x, _y ) :- _x ; _y.
    or( _x, _y, _z ) :- _x ; _y ; _z.
    
    state_make_move( _p, _m ) :- assert( does( _p, _m ) ),
                                 bagof( A, next( A ), _l ),
                                 retract( does( _p, _m ) ),
                                 retractall( state( _ ) ),
                                 add_state_clauses( _l ).
    
    state_peek_next( _ml, _sl ) :- add_does_clauses( _ml ),
                                   bagof( A, next( A ), _l ),
                                   retractall( does( _, _ ) ),
                                   remove_duplicates( _l, _sl ).
    
    state_is_terminal :- terminal.
\end{verbatim}

Prolog engines are, first of all, used on a per-query basis; agent makes single queries about legal moves or next game state, provided a move $M$ was made. It is possible to push the macro approach a little bit further, by defining a rule for a whole random simulation (that is, all the way to the terminal state), or even implementing UCT in Prolog. 
\\

Performance of Prolog varies greatly between implementations; numerous benchmarks made on simple problems are widely available
\cite{prologPerformance1, prologPerformance2}. Demoen and Nguyen~\cite{prologBench2001} points out, however, some difficulties in comparing Prolog implementations. The advantage of one system over the others varies from test to test, and depends on the type of queries which are to be made. Though slightly out-of-date, the study revealed a few factors affecting performance: conformance to the ISO standard (which supposedly introduces an overhead), system robustness and subtle implementation details.
\\

The last thing to note about using Prolog is that, although fairly easy to interface, Prolog engines are terribly slow when processing GDL rule sheets in comparison to engines custom-written for a particular game.~\cite{Ary08} reports a full game of Othello taking 2 seconds. In preliminary experiments carried for this publication, a single, random, 200-move game of chess took nearly 1,5s on a 2GHz processor with YAP Prolog. This is an overwhelming number\footnote{In comparison, Deep Blue (1997) could calculate 200 million positions per second~\cite{deepBlue}.}, and makes chess impossible to play under reasonable time constrains.

\subsection{Source Code Generation}
\label{subsec:compilation}

The set of possible queries to the logic component, given by the specification, might be considered fixed. So is the set of inference rules and possible functors given by the rule sheet. Prolog engines introduce an overhead on account of their flexibility, which allows to alter rules and facts and make various queries. On the other hand, Prolog systems are usually heavily specialized and use specific heuristics to speed up the process.
\\

For a particular game, a static program compiled after analysis of the games' rule sheet, should be sufficient to traverse the game automaton. A few sources note performance gains from translating GDL rule sheets to source code on a per-game basis. The code is later compiled and serves as a static query engine.
\\
% There is no need for a reasoning engine complex enough to supports alternation of the mentioned elements during runtime. 

Waugh~\cite{cpp09} proposed a way of generating C++ source code for a single GDL rule sheet. For each rule and fact there was a corresponding C++ function generated. Rule's body, consisted of nested \texttt{for} loops, each loop for one literal. Because invoking a rule is essentially querying for unbound variables appearing in it's head, every query returned a list of tuples - possible values of those variables. Moreover, each function was overloaded, depending on what combination of constants and variables was to be fed as input. Two interesting optimizations were also considered:
\begin{itemize}
\item Query memorization, in form of caching results of functions that do not depend on \texttt{does} predicates. Results of such functions remain unchanged throughout the whole game.
\item Reordering of literals in rules' bodies by the number of unknown variables, so that more variables would be bound in subsequent queries.
\end{itemize}
Compared to YAP Prolog, reported performance gains varied between 60\% and 1760\%, depending on the game. As the flaws are concerned, long\footnote{Source files as large as of megabyte size were reported.} source code files were reported for the most complicated games, and the volume of source code influences compilation times. Additionally, performance profiling of the system revealed a huge overhead on memory management, partially due to poor memory management of used C++ Standard Library containers.
\\

Saffidine and Cazenave~\cite{ocaml11} proposed another method of translating rule sheets to OCaml source code. Rule sheets would undergo a chain of transformations from GDL to intermediate languages - variations of GDL.

\begin{itemize}
	\item By logical transformations, rules were transitioned so \texttt{distinct} predicates would never appear negated. Rules were put into Disjunctive Normal Form and divided into sub-rules at disjunctions. Obtained rules were in Mini-GDL language, a subset of GDL, yet equally expressive.
	\item Rules were further decomposed to the normal form, in which there are at most two literals on the right hand side. Because the decomposition might have broken the safety property (which yields that every variable in a negated literal should appear in an earlier positive literal), negated literals were moved when necessary. Under an assumption that the original author of the rule sheet might have ordered the literals in an optimal way, contrary to aforementioned C++ code generation method, positive literals were not moved. Obtained rules were in Normal-GDL language.
	\item Finally, the rule sheet underwent inversion, where each functions constant $f$ was associated with rules possible to be triggered in it's body. Resulting rule sheet in Inverted-GDL was ready for translation to the target programming language.
\end{itemize}

The authors have translated the processed rules sheet to OCaml, and interfaced it with functions interfacing the program as a GA (game automaton). Because bottom-up evaluation was chosen, the system kept a fact database and allowed for adding, searching and unifying facts. Reported performance gains of this method varied between 24\% and 146\% for a set of tested games.

\subsection{Propositional Automata}

Both approaches described in previous sections, that is using Prolog reasoning or Prolog-like reasoning with pre-compiled code, employ a top-down evaluation. Propositional Automata are somewhat similar to bottom-up evaluation trees; but instead of setting the facts and propagating the effects they have towards the root of the tree, a PA retains it's internal state between updates (turns). The input of a PA consists only of the unpredictable game facts - representing the moves of the players. The updates are triggered in transition nodes.
\\

A Propositional Automaton (PA)~\cite{propnets09} is a representation for discrete, dynamic environments. Simply put, it allows to infer the next state of the environment from the previous one, meeting the requirements for a GGP reasoning engine. A PA is based on a simpler structure, called Propositional Network. A Propositional Network (PN) is a structure which resembles a logic circuit. It is a directed bipartite graph, nodes of which fall into one of the three categories:
\begin{itemize}
  \item boolean gates,
  \item transitions,
  \item propositions, further divided into:
  \begin{itemize}
    \item $P_I$ - input propositions (with no entering edges),
    \item $P_V$ - view propositions (with no leaving edges),
    \item $P_B$ - base propositions (leading connected to boolean gates and transitions).
  \end{itemize}
\end{itemize}

Proposition nodes correspond to facts about the environment; they take on boolean values. As for GGP, input propositions roughly correspond to \texttt{does} facts, and base propositions to the other facts taking part in the reasoning process.
\\

A transition node is an identity function. It serves only for synchronization purposes. During each turn, the changes propagate not from the inputs, but rather from identity nodes. In other words, a transition node serves as a dam, taming the changes until the next time slot.
\\

Evaluating the network is similar to evaluating a logic circuit. Assuming that the network already has a valid internal state, the input propositions are supplied with new boolean values, and transitions fire off, propagating changes throughout the network. An exemplary PN is shown in Figure~\ref{fig:propnet}.
\\

\begin{figure}[h]
  \centering
  \includegraphics[width=0.9\textwidth]{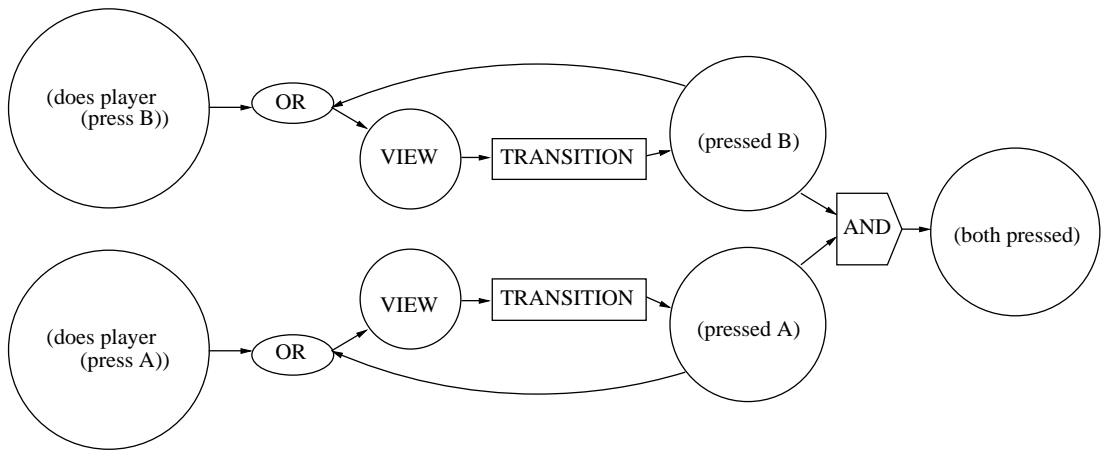}
  \caption{Propositional Network representing the physics
of a simple game~\cite{propnets09}. Each player can press either of buttons A and B which,
once pressed, remain in a pressed state.}
  \label{fig:propnet}
\end{figure}

Definition of a PN as a graph does not provide means for conducting a reasoning. For this purpose, it is being wrapped as a Propositional Automaton. A PA~\cite{propnets09} is a triple $<N, M^0_B, l>$, where $N$ is a Propositional Network, $M^0_B$ is an initial truth assignment, and $l$ is a legality function mapping base truth assignments to finite sets of input truth assignments $l:m_B\mapsto 2^{m_I}$. For GGP, initial truth assignment would be based on the initial game setup, and the legality function on the \texttt{legal} predicate.
\\

The formal definitions of PNs and PAs and algorithms for generating them were presented in great detail in~\cite{propnets09}.
\\

In GGP, PAs are very robust when compared to Prolog engines. Chmiel~\cite{gajusz12} presented a performance comparison of PAs against SWI Prolog run on the same machine. The paper also reports inability to generate PAs for more complicated games like chess and checkers, due to enormous size of their PNs. Results are presented in Figure~\ref{tab:propnetResults}.
\\

\begin{table}[h]
\caption{Propositional Net performance~\cite{gajusz12}. The PN failed to build for chckers and chess.}
  \centering
	\begin{tabular}{ |c||c||c|>{\columncolor[gray]{.9}}c| }
	% \begin{tabular}{ |c||c||c|c| }
	\hline
	\multirow{2}{*}{Game} & SWI Prolog & \multicolumn{2}{c|}{Propositional Net} \\ %\cline{2-6}
	 & states/s & states/s & \% of SWI \\ \hline
	\textit{blocks.kif} & 2537 & 23504 & 926\\
	\textit{hanoi.kif} & 1547 & 2190 & 142\\
	\textit{tictactoe.kif} & 1237 & 12580 & 1017 \\
	\textit{checkers.kif} & 146 & - & - \\
	\textit{chess.kif} & 41 & - & - \\
	\hline
	\end{tabular}
  
  \label{tab:propnetResults}
\end{table}

\subsection{Game Instantiation}
\label{subsec:instantiation}

	Game instantiation, that is, removing all variables from the rule sheet in such way, that the original semantics would be preserved, was proposed as an obvious way to speed up GDL inference in~\cite{instantiating10}. The idea is based on an assumption that using variables may add unnecessary complexity, whereas processing instantiated input, which is more like brute-force searching, is actually faster. According to the conducted experiments, the factor of Prolog engine speed growth ranged from being a few times to even 250 times faster than the search on uninstantiated input.
	\\
	
	The algorithm is shown as Algorithm~\ref{alg:instantiation}. The critical step of instantiation is calculation of supersets of reachable state atoms, moves and axioms. As initially proposed, this can be done using either Prolog or dependency graphs. After that, formulas are being instantiated. Some of the instantiations might actually be redundant and leading to conflicting formulas. For this reason, the reachable supersets are being post-processed, and obtained instantiations' validity is being checked. An extra step is to calculate groups of mutually exclusive state atoms, so that the resulting rule sheet could be processed more efficiently.
\\

\begin{algorithm}[h]
 \SetAlgoLined
 \KwData{GDL rule sheet}
 \KwResult{Instantiated GDL version of the rule sheet}
 \BlankLine
	1. Parse the GDL input.\\
	2. Create the disjunctive normal form of the bodies of all formulas.\\
	3. Calculate the supersets of all reachable atoms, moves and axioms.\\
	4. Instantiate all formulas.\\
	5. Find groups of mutually exclusive atoms.\\
	6. Remove the axioms (by applying them in topological order).\\
	7. Generate the instantiated GDL output.\\
\caption{Game instantiation~\cite{instantiating10}. The point of the algorithm is to convert the rule sheet to an equivalent form, stripped of the variables.}
\label{alg:instantiation}
\end{algorithm}

The benefit of this approach is it's compatibility – instantiated input might be written back in KIF, resulting in a rule sheet that every GGP agent  can process without any further modifications. However, instantiation might not always be feasible to carry out. For some games from Dresden GGP Server\footnote{\url{http://ggpserver.general-game-playing.de/}}, the C++ instantiator ran out of time or memory. Despite that, instantiation has been utilized in the Gamer GGP Agent~\cite{GamerGGPAgent11}.

\section{Perspective of Further Development}

Presented methods approach the problem of (multi) game playing from entirely different angles. Despite using MCTS methods, both in GGP and in recent computer Go, there is a tremendous gap in performance of those systems. Naturally, Go has the advantage of the knowledge being carefully tested by researchers. But the used methods are also far apart. Another thing that strikes at a closer look is the appalling performance of existing GGP reasoning systems.
\\

Research in the area of GGP continues in different directions. But some of them seem to arise when trying to cope with minor flaws of the system. Perhaps GGP is mature enough to be reevaluated, to direct future research towards more exciting methods focusing solely on playing, possibly drawing from other, successful MCTS-based projects.

\chapter{Problem Definition}
\label{chap:problem}

\section{Motivation}

It is clear that MCTS technique has become a standard in meta-gaming and is a valuable approach for games with high branching factor. However, some argue that the AAAI GGP Competition does not encourage CI methods and actually learning the game. It might be valuable to relax a bit restrictions imposed by the GGP specification; overcome some of it's flaws in order to develop stronger agents. However, the main assumption of GGP should be retained: learning without any human supervision.
% While GGP programs are getting stronger by each year, many authors point out that they are still noticeably weaker than dedicated systems and humans.
\\

MCTS has been widely applied to problems and games perceived as hard, with Go being the prime example of such. Large branching factor rules out conventional, tree-searching methods for such games. A great deal of research has been made in terms of adapting and equipping MCTS with hand-coded knowledge, or mined from expert game records in Go. Usually, in chess, the Game of Amazons, etc., the knowledge and learning are carried beforehand, and the agent is being optimized for fast state-switching. GGP has a different philosophy: learning and playing the game are almost parallel, so it's hard to expect robust state-switching. Those two activities might be better off separate.
\\

The Go approach of game record analysis and mining spatial features, presented particularly in~\cite{goMovePatternTrends12, KNearestPatternsGo05, BayesianPatternRankingGo06}, seems applicable to GGP with slight modifications. Such features usually rely on Euclidean metric. Numerous attempts to automatic discovery of metrics in GGP have been presented in Subsection~\ref{subsec:boardDiscovery}. It would be more reliable to simply add board meta data to the game's rule sheet and continue with mining the features. Performance might not be that good with a generic-feature approach as it is in Go, but it might still suffice to strengthen the agent against weak UCT players in a multi-game environment.
\\

%Such a system of developing game records and mining features would be closer to that from Go and would realize the above. Additionally, such system is easy to further extend with new classes of features.

Additionally, even with quality game knowledge at hand, the appalling performance of Prolog-based reasoning engines might compromise the agent's ability to play. As stated in Subsection~\ref{subsec:prolog}, such engines, while the easiest to employ, are the main reason of inefficiency of GGP agents. The approach from~\cite{cpp09} of translating GDL to C++, which makes as a reasonable trade-off between performance and resource utilization, is well worth reevaluating. Especially, when code generation might be greatly simplified with a slight and not that harmful change to GDL.
\\

All the mentioned ideas pose as an opportunity for creating a system, where developing the game knowledge would be a separate process carried automatically prior to the match, unlike in GGP where it always happens while the match goes on. During an actual play, a lean, robust agent could then use the knowledge to it's advantage, focusing only on the tree search.

% GGP methods are also being criticized fro not relying on "real knowledge". There is a common opinion in the field of AI that, while we could hypothetically mimic some human learning mechanisms, it would be a waste of computational power because specific algorithms and methods not inspired by living organisms would perform much better.

% In other words: maybe some generic learning methods, that do not perform well enough against dedicated systems, could perform well against weak UCT players.

\section{The Problem}

This publication pursues the goal of making a design, implementation and evaluation of a lean, versatile GGP player. Such player should be able to learn the game before an actual GGP match, but still without any human supervision. Thus, some limitations imposed by GGP may be relaxed in order to achieve the goal.

% In a nutshell, the main aim of this publication was to design a GGP system working under a variation of the GGP specification with relaxed constrains, yet backward-compatible with the original specification. The system had to be robust, with strongly developed knowledge modules, relying on a heavy pre-computing phase, so a more interesting kind of knowledge could be developed. Lastly, the agent part of the system had to be equipped with a fast GA traversal mechanism, realized through a custom-compiled reasoner. 

% The key objective of this publication was to design a strong GGP agent by extending the bare design, common to all leading GGP agents. As stated in Section~\ref{sec:agentArchitecture}, such an agent employs UCT search combined with an external Prolog engine. Strengthening the agent consists of three main parts: speeding up the logic component, mining the game knowledge and using the knowledge in plays. 

\section{Discussion}

% preferably by using the approach of using pseudo-random playouts in UCT, which proved succesfull in Go.

A lean player is, in principle, an agent conserving the resources - that is available memory and computational power. CPU conservation was achieved by efficient reasoning engine, generated ad-hoc from the rule sheet.

Two different concepts of generating C++ source code were examined. Both designs, reliant on ordinary arrays and custom data structures instead of STL containers, involved reserving the whole memory at startup, to mitigate the management overhead mentioned in the original paper. The newly designed system was also modular, in the sense that it easily allowed swapping the underlying containers or key code-generating functions; the two new mentioned versions shared most of the source code.

To keep things simple, a small change to the GDL specification was made, which resulted in a new language: sharing the same syntax but slightly different semantics. While GDL allows and encourages using complex terms as arguments for predicates, a quick analysis of a set of popular GDL games revealed that they are not being used in practice even in complicated rule sheets. % , was further adapted to the simplified version of GDL and tested after further re-engineering. 
Thus, the simplified version of GDL constitutes a solid theoretical base for the system, as well as it results in a clear and comprehensible source code.

The idea from Holt~\cite{holt08} of using transposition tables was extended, to adapt to possible, demanding memory constrains. Instead of randomly deleting states, the refined design presented in this work combined transposition tables with priority lists, allowing to keep the most interesting states intact. Thus, MCTS limited with small transposition table, would build the game tree selectively, adapting to the most pormising branches.
\\

As to learning the game beforehand, an attempt to apply some successful ideas from the Go community was made (described in Section~\ref{sec:goKnowledge}), to general games. It required to analyze game records, find simple features describing the board and assign them weights, to gather basic knowledge about the game. The features were inspired by those which are and are not Go-specific, and rely only on the Euclidean metric (and structure of the board). To achieve this, GDL has been modified with a small, backward-compatible extension. It supplied the basic meta informations about the game board and additional semantics of the rule sheet to the agent, remaining transparent to the agents not supporting it.
\\

The features were mined by analysis of a game records database. To generate quality ones, a multi-agent system was designed, driven by an evolutionary approach. The agents improved upon both their knowledge and their plays.
\\

To strengthen the knowledge module, some of the inconclusive features were backed with frequent itemsets of state facts and meta-facts. Because the games (and thus their states), which can be thought of as Apriori baskets, can be both winning or loosing, an approach similar to that in~\cite{negativeAssociationRules04} was needed in order to gather game facts associated with winning (but not loosing) and vice versa. A variation of the Apriori Algorithm~\cite{Agrawal:1994:FAM:645920.672836} was prepared, which adapted to those needs and fit well with the expected game data.
\\

% As an addition to this publication, the implementation of all of the described ideas was prepared. It has been used in tests conducted in chapter~\ref{chap:tests}. It has been released under GNU General Public License v3~\cite{gpl}.

\chapter{Proposed Approach}
\label{chap:keyIdeas}

The new approach which this thesis proposes, merely outlined in the previous chapter, is presented in detail in this chapter. It has been created as a solution for the stated problem, and might be considered the core of this publication. First, a revised scheme of translating GDL rule sheets to C++ code is presented and formally backed with an alternation to GDL. Next, a small extension to GDL is described along with a kind of game knowledge, which the extension allows to obtain. Lastly, a way of extracting spatial knowledge through self-play, including necessary algorithms, is presented.

\section{Revised GDL to C++ Translation Scheme}

	Game instantiation or inference through Propositional Automata results in far more efficient reasoning than with Prolog. However, those methods are bulky both time- and memory-wise, and do not work for certain complex games, due to exceeding time and memory limits. The code generation approach was chosen for this publication as a compact alternative, not subjected to memory-overflowing issues. The realizations of this approach mentioned in Subsection~\ref{subsec:compilation} performed well, while leaving room for some obvious improvements.
\\

In this chapter, design of a refined GDL/KIF to C++ translation scheme based on~\cite{cpp09}, created for this publication, is described. First, the concept of a reasoning tree is clarified. By considering sample queries, the mGDL language is outlined and defined as a subset of GDL, much easier to process. Then, examples are presented of how the translation of mGDL/KIF rules to C++ routines can proceed. Furthermore, two versions of the refined code generation approach, based on different reasoning algorithms, are explained. The system has been designed in such way, that the underlying data containers and output functions could be changed with little effort, thus both versions share most of the source code.
\\

\subsection{Reasoning Trees}

A typical way to resolve a Datalog/Prolog query is to evaluate the associated reasoning tree made of the rules and facts. It can be done with a bottom-up or a top-down approach. The tree can be build recurrently:
\begin{enumerate}
\item Begin with a query sentence $\phi_q$ as a root.
\item If $\phi$ is a sentence:
	\begin{itemize}
	\item Find a set $C$ of rules and facts $\phi$ unifies with.
	\item If $|C| > 1$ then add an \texttt{OR} vertex as a child of $v$, with elements from $C$ as it's children and apply recurrently to those children.
	\item If $|C| = 1$ then make the element from $C$ a child of $v$ and apply recurrently to that child.
	\end{itemize}
\item If $\phi$ is a rule then add each $\phi$'s literal as it's child and apply recurrently.
\end{enumerate}

The resulting tree has the rules as internal nodes, and facts as leaves. Sample trees generated with \textit{tictactoe.kif} rule file are presented in Figure~\ref{fig:tictactoeTrees}.

\begin{figure}[h]
  \centering
      \includegraphics[width=0.8\textwidth]{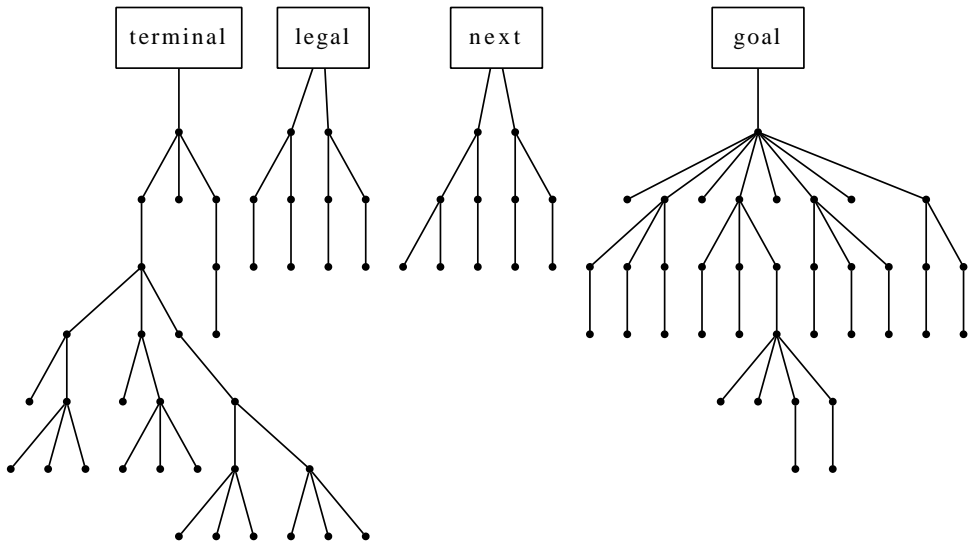}
  \caption{Sample reasoning trees for \textit{tictactoe.kif}. Each node represents a rule; node's children represent rules which may unify with literal from the rule's body.}
  \label{fig:tictactoeTrees}
\end{figure}

To illustrate size of typical reasoning trees, a few more examples are presented in Figures~\ref{fig:chessAsteroidsTrees},~\ref{fig:othelloTrees}, with labels removed for enhanced readability.
\\

\begin{sidewaysfigure}[p]
  \centering
  
  \begin{minipage}{\columnwidth}
  \subfloat[Chess]{
      \includegraphics[width=25cm]{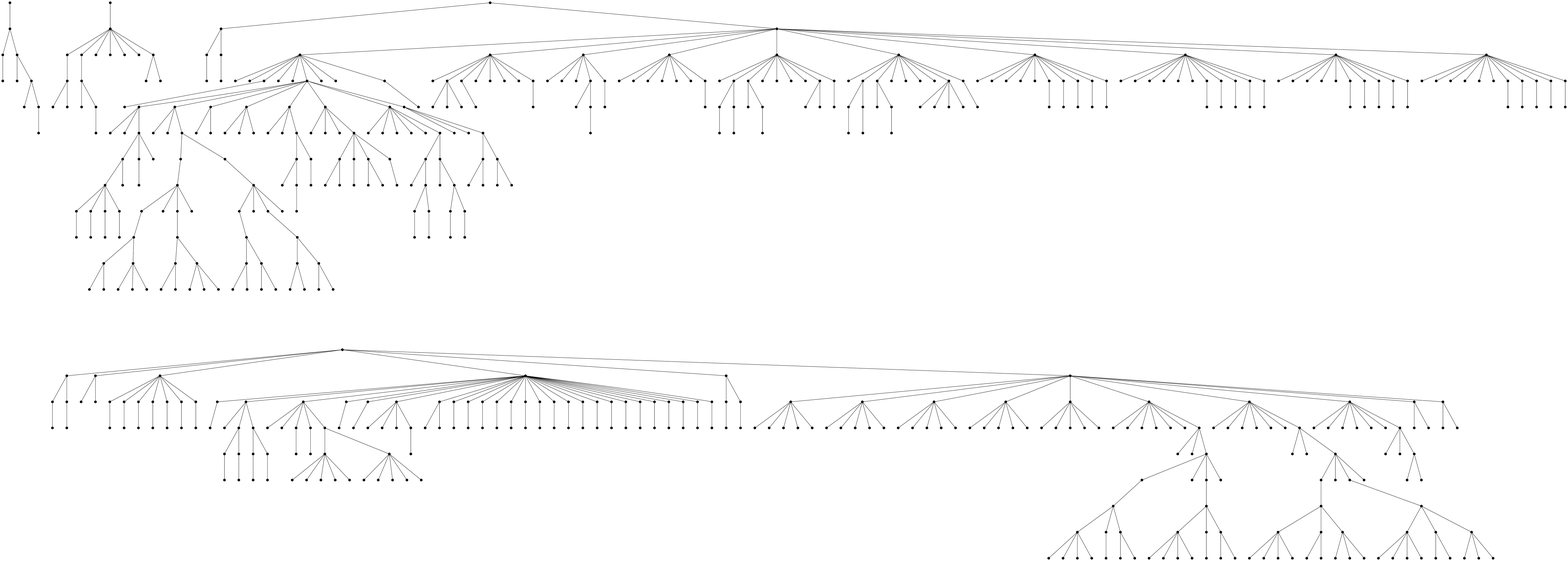}}
  \end{minipage}
 
  \vspace{1cm} 
 
  \begin{minipage}{\columnwidth}
  \subfloat[Asteroids]{
      \includegraphics[width=25cm]{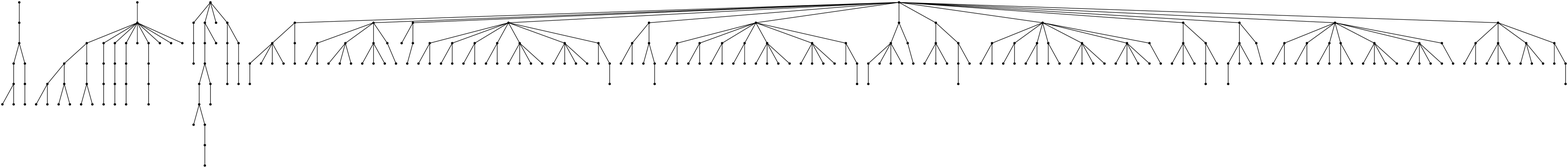}}
  % \caption{The reasoning trees for \textit{chess.kif}.}
  \end{minipage}
  \caption{Reasoning trees for \textit{chess.kif} and \textit{asteroidsserial.kif}}
  \label{fig:chessAsteroidsTrees}
\end{sidewaysfigure}

\begin{figure}[h]
  \centering
      \includegraphics[width=0.8\textwidth]{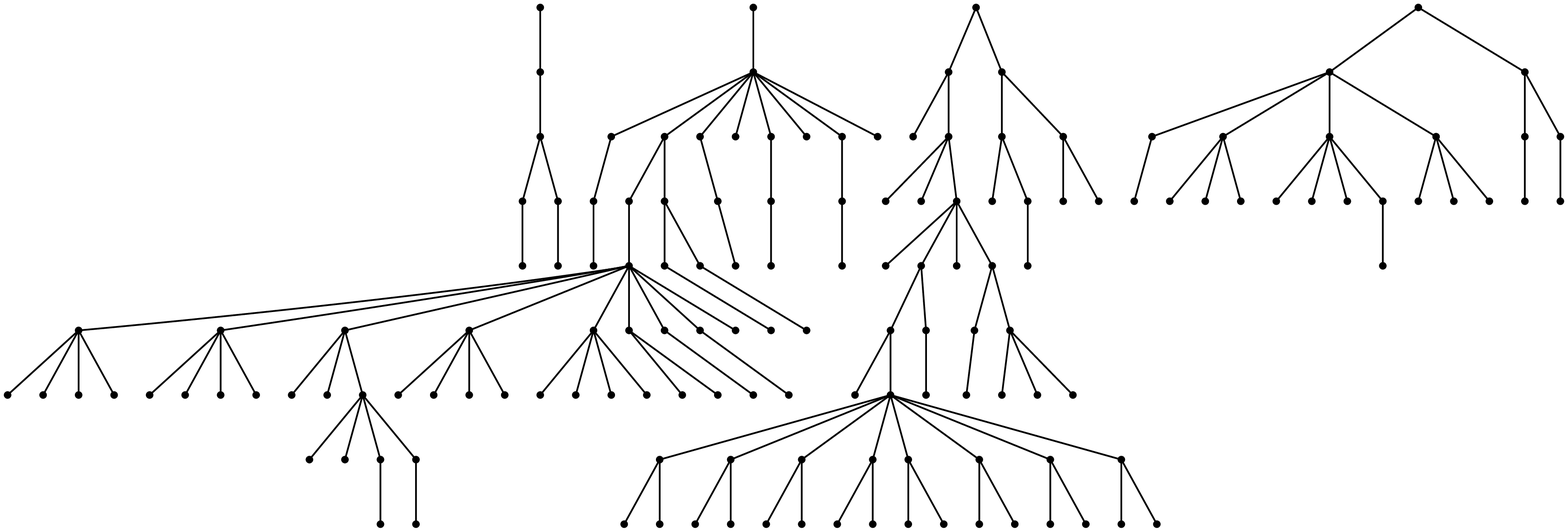}
  \caption{Reasoning trees for \textit{othello-comp2007.kif}}
  \label{fig:othelloTrees}
\end{figure}

\subsection{Leveraging Datalog}

	GDL was designed to describe games, which have natural representations as DFAs. But Datalog or even Datalog stripped of variables is expressive enough to describe them. In fact, a similar  concept was explored in the rule sheet instantiation scheme mentioned in Subsection~\ref{subsec:instantiation}.
\\

As the GDL specification~\cite{gdl_specification} states, it builds directly on top of Datalog, adding a few changes. The most meaningful is addition of function constants, what results in a possibility of occurrence of nested atomic sentences. It also complicates the reasoning; for instance, a sample query
\begin{verbatim}
   (legal xplayer ?move)
\end{verbatim}
can yield
\begin{verbatim}
   ?move : (play 1 3 x)
   ?move : noop
   ?move : (foo (bar (baz x))),
\end{verbatim}
and implies usage of an unification algorithm for inference.
\\
	
In practice, it is rarely handy to use nested atomic sentences. Whether is it worth having that extra syntactic sugar at the expense of complicated reasoning, is a mater of discussion. On the one hand, unification should not consume much time since the expressions to unify are supposed to be simple in most cases (and they almost never have nested function constants). On the other hand, it always adds the overhead. It depends heavily on the implementation of the reasoning engine.
\\

Of course, the matter of usefulness of nested function constants depends on what queries are to be made to the reasoning engine. If, for instance, the previous query would be reformulated to
\begin{verbatim}
   (legal xplayer (play ?x ?y ?p)),
\end{verbatim}
then the answer set would be
\begin{verbatim}
   ?x : 1
   ?y : 3
   ?p : x.
\end{verbatim}
An ordinary agent would have almost fixed set of possible queries throughout the game, specifically:
\begin{itemize}
    \item \texttt{next} queries,
    \item \texttt{legal} queries,
    \item \texttt{goal} queries,
    \item \texttt{terminal} queries.
\end{itemize}

As in Figure~\ref{fig:tictactoeTrees}, those would be the roots of the queries. Because of their arguments' semantics, \texttt{next} and \texttt{legal} rules usually introduce nested atomic sentences (like in \texttt{(<= (legal xplayer (mark ?x ?y x) ...)} for instance).
\\

It follows that it is a matter of making measured queries, to keep the reasoning fairly simple and have variables bind only to object constants. This realization led to the definition of a variant of GDL under the working name mGDL. It guarantees that variables can bind only to object constants, with function constants left merely as a syntactic sugar.

\subsection{mGDL}

As an attempt to simplify GDL, therefore simplifying the reasoning, mGDL has been created as a part of research for this publication. Since the definition of mGDL is almost identical to the original definition of GDL~\cite{gdl_specification}, only altered parts are given. Following the convention adopted from the mentioned document, changes are emphasized with a bold font.

\begin{definition}
\emph{(Satisfaction)}. Let $M$ be a model, and let the sentence in question be an explicitly universally quantified Datalog rule.
	\begin{itemize}
        \item $\models_{M} distinct(t_1, t_2)$ if and only if $t_1$ and $t_2$ are not the same term, syntactically.
        \item $\models_{M} p(t_1, \ldots, t_n)$ if and only if $p(t_1, \ldots, t_n) \in M$.
        \item $\models_{M} \neg \phi$ if and only if $\not\models_{M}\phi$.
        \item $\models_{M} \phi_1 \wedge \ldots \wedge \phi_n $ if and only if $\models_{M} \phi_i$ for every $i$.
        \item $\models_{M} h \Leftarrow b_1 \wedge \ldots \wedge b_n $ if and only if either $\not\models_{M} b_1 \wedge \ldots \wedge b_n$ or $\models_{M} h$ or both.
		\item $\models_{M} \forall X.\phi(X)$ \textbf{if and only if} $\models_{M} \phi(c)$ \textbf{for every object constant} $c$.
	\end{itemize}

\end{definition}

mGDL is somewhat between Datalog and GDL. Function constants have been left for notation consistency's sake, but variables are not allowed to range over sentences. It is also worth emphasizing that, while the semantics change, the syntax does not. Consider the rules
\begin{verbatim}
   (<= (foo a))
   (<= (foo (bar a)))
\end{verbatim}
A query \texttt{(foo ?x)} would yield the following results:
\begin{verbatim}
   ;; GDL
   ?x : a
   ?x : (bar a)

   ;; mGDL
   ?x : a
   ;; ?x cannot be unified with (bar a)
\end{verbatim}
Under the assumption of making measured queries, most of the popular rule sheets have the same semantics in both mGDL and GDL. It was confirmed in Section~\ref{subsec:mgdlTest}.
\\

The simplest form of unification-based~\cite{wiki:unification} reasoning is presented in Algorithm~\ref{alg:queryResolving}. As mentioned earlier, the minimal reasoning engine "interface" exposed to the agent consists of functions corresponding to reserved predicates. With reasoning trees rooted in those queries, the unification would occur multiple times in every single node of those trees, including the leaves.
\\

\begin{algorithm}[h]
 \SetAlgoLined
 \KwData{Sentence $\phi$}
 \KwResult{All substitutions satisfiable $\phi$ in the model}
 \BlankLine
    \If{$\phi = true(\psi)$}{
        \For{fact $f\in F_{dyn} \cup F_{stat}$}{
             $unify(\psi, f)$\;
        }
    }
    \For{rule $r\in R$}{
        \If{$unify(\phi, head(r))$}{
            $resolve(body(r))$\;
        }
    }
\caption{Outline of a simple, unification-based query resolving algorithm}
\label{alg:queryResolving}
\end{algorithm}

With mGDL, the unification is greatly simplified, for it concerns only "flat" atomic sentences. It requires only to check if the corresponding functors and arguments match - and that is linear in time with respect to the number of arguments. Of course, good unification algorithm would also perform in linear time with such data. However, by translating atomic sentences directly to C++ routines, the overhead is greatly reduced, and no complex tree-structured expressions (representing potential substitutions for variables) are being passed as arguments. Instead, language constructs are being leveraged, and function arguments are mapped to simple types only.

\subsection{The Database Aspect}
\label{subsec:database}

Falling back to Datalog-like reasoning with mGDL brings out the database aspect of reasoning. After all, Datalog is the language for deductive databases. Intensive research in that area was conducted starting from the late 80s, and Datalog is still in use to this day~\cite{Datalog11}. Many clever algorithms like the Magic Sets~\cite{MagicSets86} were developed, and there is certainly a room for improvement in developing compact, mGDL reasoning engines too.
\\

Throughout the game, the rules do not change. In a sense, they constitute relational algebra relations. Resolving a query can be seen as a chain of relational algebra operations, namely: projection, selection, natural join and cross join (Cartesian product). For instance, rule
\begin{verbatim}
   (<= (sibling ?x ?y)
       (child ?x ?x1)
       (child ?y ?y1)
       (married ?x1 ?y1))
\end{verbatim}
can be rewritten in relational algebra as
$$ Q = \pi_{X,Y} \left(\left(Child_{X,X1} \times Child_{Y,Y1} \right) \bowtie Married_{X1,Y1} \right).$$

When choosing an external reasoning engine, instead of reaching for Prolog because of similarities of syntax/semantics, it could be more efficient to pick an engine that closer matches the needs in terms of what it does. An example of such engine would be a one based on Datalog or Relational Algebra. mGDL could mitigate the difficulties in suiting systems of this kind to GGP. This publication, however, does not further pursue the idea.

\subsection{Table-Driven and Query-Driven Models}
\label{subsec:gtcModels}

As mentioned earlier, one of the goals of this publication was to pursue an efficient GDL to C++ translation scheme. The developed approach required translating each rule and fact to a C++ function. As the evaluation of an expression (query) proceeded in a top-down manner, the function matching the query was called. In the function's body, the literals were evaluated by calling the corresponding functions, and so on.
\\

Two models of carrying tree evaluation were chosen for implementation and testing, both sharing essentially the same components: a table-driven and a query-driven one. They varied in how variables were handled and passed down the tree.
\\

By enumerating all object constants in a particular rule file, they can be mapped to natural numbers. In mGDL, storage of potential variable values is straightforward, as they do not have to be polymorphic. At minimum, a set of values a variable might be bound to, can be represented by a linear container of primitive integer type.

\paragraph{The table-driven model}
In this model, a two-dimensional array of primitive type is stored in the memory. It might be though of as a simple table of valid variable substitutions. More specifically, each column corresponds to a GDL variable, and each record to a valid set of values.
\\

With such table, tree evaluation still proceeds top-down. The nodes are visited along a depth-first search path:
\begin{itemize}
\item Each time a node $v$ is visited, new variables might be declared.
\item Variables declared in $v$ are local only for the subtree rooted in $v$.
\end{itemize}
The second property allows to always add/remove columns in the "last in, first out" fashion. Simply speaking, the table will never have empty columns in between. An example is shown in Figure~\ref{fig:lilo}.

\begin{figure}[h]
	\centering
	\begin{tabular}{ccc}
	\subfloat[Entering a node in the reasoning tree]{\includegraphics[scale=0.6]{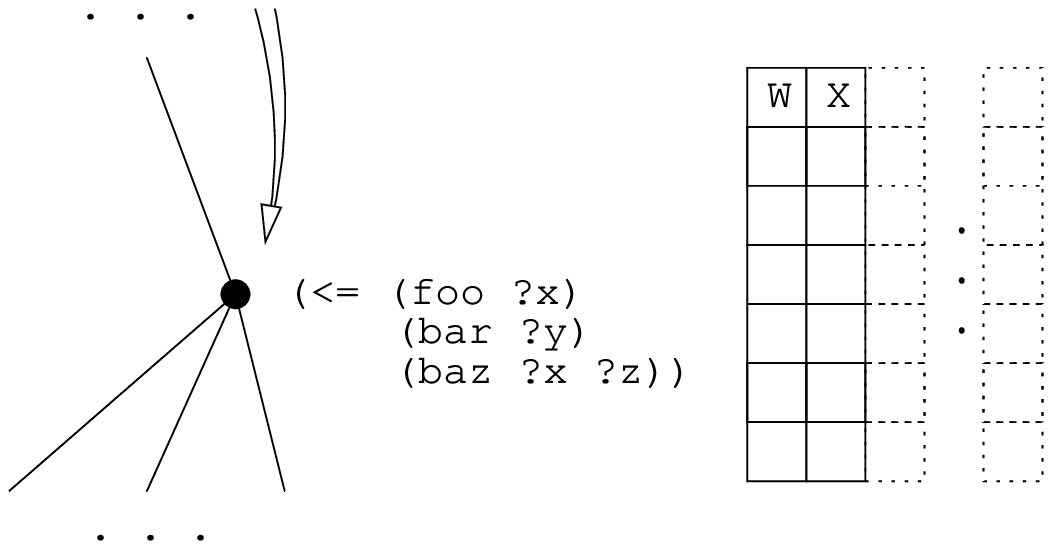}}
	\hspace{2cm}
	\subfloat[Progressing down the tree; room for variables local to the subtree reserved]{\includegraphics[scale=0.6]{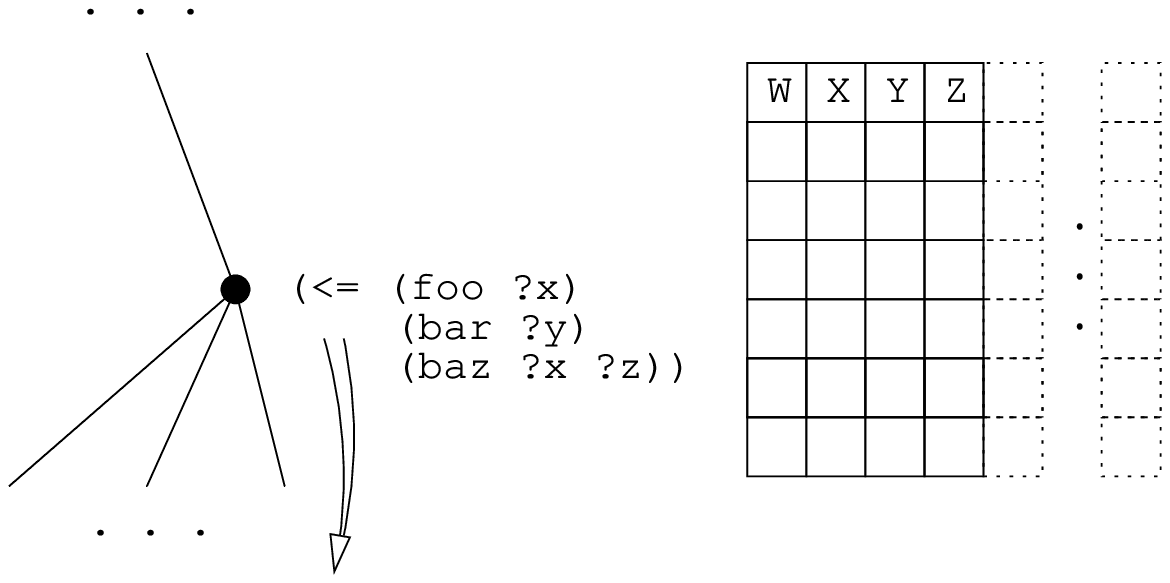}} \\
	\vspace{0.5cm}\\
	\subfloat[Collecting the results of the subtree; local variables erased]{\includegraphics[width=0.4\textwidth]{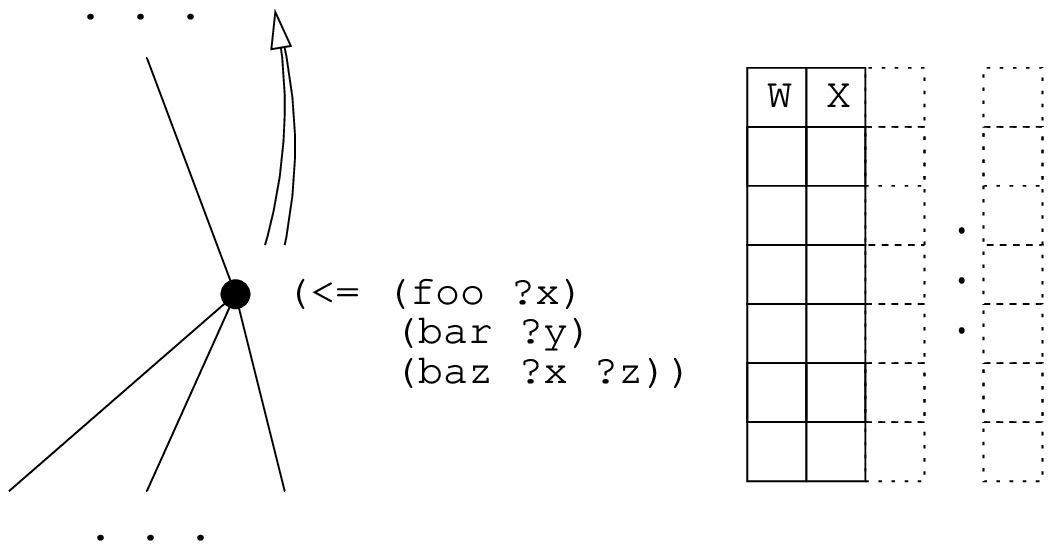}}
	\end{tabular}
\caption{Column management in the LIFO fashion}
\label{fig:lilo}
\end{figure}

However, the columns were not simply added; the join operation took place, as soon as new variables were bound to possible sets of values. In the example in Figure~\ref{fig:lilo}, let $VS$ denote the \texttt{VarStore} relation (table), and $Bar$, $Baz$ relations obtained from \texttt{(bar ?y)} and \texttt{(baz ?x ?z)} respectively. After \texttt{bar} has been evaluated, the new relation $VS'$ would be \[VS'=VS\times Bar\] and, after resolving \texttt{baz}, $VS''$ would be equal to \[VS''=VS'\bowtie Baz\] (with both relations joined on $X$).
\\

One advantage of this approach was the simplicity of the source code. Recalling from the GDL specification~\cite{gdl_specification} that a GDL rule is an implication
\[h \gets b_1 \wedge \ldots \wedge b_n,\] where the head $h$ is an atomic sentence and each $b_i$ in the body is a literal, it was translated roughly to
\begin{verbatim}
   bool h() {
       return b1() && ... && bn();
   }
\end{verbatim}
with the \texttt{VarStore} object passed by reference between the functions.
\\

Another advantage was the memory management. A large table variable was declared at run-time and kept throughout the whole game, solving the memory management problem mentioned in~\cite{cpp09}. The maximum number of variables (columns) that have to be stored simultaneously, bounding the table size in one dimension, can be determined by a simple DFS browsing of the reasoning tree.

\paragraph{The query-driven model}

To rule out join operations completely, the table-driven approach was combined with the idea from~\cite{cpp09}. Previously, processing a rule required evaluating each of the literals only once. The latter has literals evaluated multiple times in nested loops.
\\

The change shifts the cost, from manipulating a variable table, to evaluating the same literals multiple times (which also results in multiple function calls), but with less data.
\\

The query-driven model still used the variable table object - but only locally to a literal. A GDL rule was translated roughly to
\begin{verbatim}
   bool h() {
       vs_1 = b1();
       for (int i1 = 0; i1 < vs_1.len; ++i1) {
           r2 = b2();
           ...
               vs_n = bn();
               for (int in = 0; in < vs_n.len; ++in) {             
                   results.add_result();
               }
           ...
       }
   }
\end{verbatim}
where $\text{\texttt{r1}}, \ldots, \text{\texttt{rn}}$ denote \texttt{VarStore} objects.

\subsection{Generating the Source Code}

There are three main kinds of entities that might occur in a rule file Their summary is shown in Figure~\ref{tab:kifObjects}. Each kind has been converted to a piece of C++ source code in a slightly different manner.

\begin{table}[h]
  \caption{Summary of rules' entities with their counterparts in Prolog and GTC reasoning machines}
  \centering
  	\begin{tabular}{|l|l|p{6.5cm}|}
	\hline
	GTC C++ code & Prolog & GDL \\ 
	\hline
	container \& functions & static fact & ground atomic sentence, e.g. \texttt{(role xplayer)}\\ 
	\hline
	container \& functions & dynamic fact & ground atomic sentence $\phi$ such that $\models_M next(\phi)$ in the previous state, e.g. \texttt{(mark 1 1 x)} \\
	\hline
	function & static rule & rule, e.g. \texttt{(<= terminal (not open))} \\ 
	\hline
	\end{tabular}

  \label{tab:kifObjects}
\end{table}

\paragraph{Rules}

Each rule had a function associated (precisely, a set of overloads); functions' bodies were dependent on the chosen reasoning model.

\paragraph{Static and dynamic facts}

Following the notion adapted from Prolog-based agents, relation constants were assumed unique to either static facts, dynamic facts, or rules. With this assumption, static facts are easy to recognize; those are the ground atomic sentences hanging loosely in the KIF files, and they stay unchanged throughout the whole game.
\\

Dynamic facts are a little harder to distinguish; during the first turn, those are the ones decorated with the \texttt{init} functor. After the turn has passed, they should be replaced with new ones inferred with the \texttt{next} rules.
\\

Both kinds of facts had C++ containers and a set of querying functions associated.

\paragraph{Name mangling}

Because in mGDL, function constants have been syntactically left intact, it became less obvious how to translate nested atomic sentences to C++ functions. A simple name mangling scheme was used, so the name of a function could reflect the structure of the corresponding atomic sentence. The conversion was simple; it essentially required to flatten\footnote{The metaphor of flattening nested expressions was borrowed from Ruby's \texttt{flatten} method for arrays.} the sentences. For instance, rule head
\begin{verbatim}
   (<= (legal ?player (move ?x ?y ?piece))
\end{verbatim}
would be flattened to
\begin{verbatim}
   (<= (legal_ARG_LPAR_move_ARG_ARG_ARG_RPAR ?player ?x ?y ?piece)).
\end{verbatim}
The name reflects all terms and parentheses, read from left to right. The naming scheme is verbose in part for the human reader's sake. The finished, compiled reasoner did not break up the function name to analyze it during the runtime.

\paragraph{The procedure}

The complete procedure of generating the source code was as follows:
\begin{enumerate}
  \item The rule sheet, assumed to be written in mGDL, was preprocessed. Preprocessing consisted of identifying static/dynamic functors, removing \texttt{true} functors, and flattening the sentences with the name mangling scheme. \texttt{or/n} literals were isolated to $n$ separate sub-rules.
  \item For each fact and each rule, functions (along with the relevant overloads) were generated.
  \item The common interface was completed with interface functions for making precise queries to \texttt{legal}, \texttt{next}, \texttt{goal} and \texttt{terminal}.
  \item The code was compiled as a shared object.
\end{enumerate}

The common, pre-written code included the reasoning engine interface, auxiliary wrappers for saving/loading states, resetting the games etc., variable table, and fact containers.

\subsection{Optimizations}

The original paper describing translating GDL rule sheets to C++ code also featured description of optimizations. The following paragraphs provide a short commentary for those, with respect to the proposed new translation scheme.

\paragraph{Memorization}

Waugh~\cite{cpp09} employed memorization of state-independent query results (those, which do not rely on the \texttt{does} relations). In the query-driven model, with immediate substitution of variables with values, leaves with no variables are getting evaluated far more often than in the table-driven approach. Evaluating leaves with no variables, which defaults to a boolean query, can be realized simply by a hashset lookup - and it further adds to the performance. While memorization of fact queries might not be worth the cost, an efficient technique of memorizing rule queries (again, only state-independent) might yield further improvements.
\\

In implementation made for this publication, only memorization of fact-queries was investigated. Facts were called in the leaves of the reasoning tree; queries might have been either entirely boolean (with no variables), or they might have required a join operation.
\\

Because of the query-driven model characteristics, fact queries were called more often and with less variables, than in the table-driven approach. Particularly, profiling the output with Google CPU Profiler\footnote{Part of Google Performance Tools~\cite{googlePerfTools}.} revealed that for most of the games, fact queries with no variables took most of the time spent on fact queries. The amount varied from game to game, being usually around 5\% - 30\% of total runtime.
\\

To mitigate the issue, a hashset lookup was implemented with open addressing hashsets\footnote{This particular method allowed for quick calculation of the hash, based on the fact's arguments. It also took advantage of the sparse nature of the values.}. Double hashing  was used. Again, for each fact, memory for lookup set was reserved at startup. This form of caching was used for all fact containers, and yielded improvements for all of the tested games.

\paragraph{Relevant function overloads}

To reduce the size of output source files and compilation times, only the necessary function overloads (for rules and facts) were generated. Each function was overloaded with possible combinations of input arguments as constants/variables, but not all combinations were present in the reasoning tree.
\\

With a fixed set of possible queries, all possible function overloads have been identified. To achieve that, a DFS search in the reasoning tree was employed, substituting variables with labels \texttt{<var>} and \texttt{<const>}. In other words, trees have been flooded to see what kinds of arguments were to be expected in the rules' heads. 

By tracing the labels, overloads were obtained. It should be however noted, that such set is a superset of the sufficient set of overloads, as some literals and rule heads agreeing in functors and arities might still never unify during the game. For instance \texttt{(goal <var> <const>)} corresponding to \texttt{(goal <var> 100)} would still unify with \texttt{(goal ?x 50)}. However the supersets, even for demanding game rules like chess and checkers, turned out to be of acceptable size.
\\

Function overloads, prepared this way, fit well with the C++ overload constructs. Below are sample definitions of functions from the aforementioned Tic-tac-toe, obtained by the method:

\begin{figure}
\begin{SaveVerbatim}{VerbtttOverloads}
bool diag(VarStore *& vs, const Constant & c1);
bool line(VarStore *& vs, const Constant & c1);
bool not_open(VarStore *& vs);
bool mark(VarStore *& vs, Variable v1, Variable v2, Variable v3);
bool mark(VarStore *& vs, Variable v1, const Constant & c1,
          const Constant & c2);
bool mark(VarStore *& vs, const Constant & c1, Variable v1, 
          const Constant & c2);
bool mark(VarStore *& vs, const Constant & c1, const Constant & c2,
          const Constant & c3);
\end{SaveVerbatim}
	
	\centering
    \setlength{\fboxsep}{5mm}
    \fbox{\BUseVerbatim{VerbtttOverloads}}
    \caption{GTC C++ function definitions from generated \textit{tictactoe.kif} reasoning machine}
    \label{fig:tttOverloads}
\end{figure}

\paragraph{Reordering of literals}

\cite{cpp09} also employed reordering of literals by the number of unknown variables, but with respect to Datalog Rule's Safety property~\cite{gdl_specification}. Saffidine and Cazenave~\cite{ocaml11} however, has not reordered the literals, under an assumption that the creator of the rule file might have already ordered them, keeping efficiency in mind. Under the same assumption, in implementation made for this publication, literals also have not been reordered.
\\

\section{Spatial Game Knowledge}
\label{sec:knowledgeMining}

The following section describes an attempt made in this thesis, to transfer the ideas of obtaining feature-based game knowledge in Go, mentioned in Chapter~\ref{chap:sota}, to GGP. Numerous successful MCTS-based Go agents achieved professional rankings by taking advantage of such methods, which could be of great value to GGP as well.
\\

There are, however, legitimate reasons why the feature approach is not straightforward to be applied to GGP. To name a few:
\begin{itemize}
    \item Knowledge modules usually use game-specific features.
    \item An arbitrary game (in it's abstract form) not necessarily takes place on a 2-dimensional board, what further discards some universal features. Even if the game has some kind of a board, the agent is not aware of it's structure.
    \item Board pieces cannot be assumed to be persistent between states; 
          predictive power of patterns may be limited. For instance, a few games might be "mixed" into one and being played in parallel, alternating between entirely different board situations.
    \item There are no quality game records for general games, useful for mining features and weights.
    \item Resources devoted to developing the agent are limited in GGP (usually it is the startclock phase of a playout).
    \item No human intervention is assumed in GGP. The agent would have to tune all the parameters by itself, which is again hard to achieve reliably within the resource constrains.
\end{itemize}

This publication relaxes the limitations imposed by the GGP specification, so the developed GGP agent could be equipped with more advanced features in the likeness of Go-playing agents, and the other MCTS-based ones. However, the assumption of no human intervention during the whole process is retained, as the most crucial one.
\\

The following sections address the aforementioned issues in detail.

\subsection{The Euclidean Board Metric Extension of GDL}
\label{subsec:gdlExtension}
This section covers a proposed extension to GDL, introduced in order to reliably extract Go-like game features. Subsection~\ref{subsec:boardDiscovery} describes some of the tools allowing to infer the existence of the board and define the Euclidean metric for a particular GDL rule sheet. However, none of these tools are reliable; good tools would have to thoroughly understand a particular game and it's spatial arrangement, a task that might not be at all possible for some games. Although the class of games describable by GDL is enormously large, most of the real-world games translated to GDL are actually square-board games or at least they have a quantified, spatial representation. In other words, many of them are not purely abstract (like rock-paper-scissors in contrast to Tic-tac-toe) or do not happen in continuous domains (like the original Ms. Pac-Man).
\\

For this reason, this publication introduces an extension to the GDL specification, consisting of new reserved relations, describing multi-dimensional boards and the corresponding pieces. With the extension, underlying structures of most of the board games can be recognized unequivocally. The extension is not mandatory, and it requires only to supply the meta data in form of static facts in the KIF file.
\\

The system works with boards of hypercubical shape. Formally a board is a set:
$$\left\{
\begin{pmatrix}
r_1, \cdots, r_n
\end{pmatrix}^T
%\begin{pmatrix}
%r_1 \\
%\vdots \\
%r_n
%\end{pmatrix}
: d_{min} \leq r_1, \dots, r_n \leq d_{max} \right\},$$

which is described in GDL by the following relations:

\begin{description}
    \item[\texttt{boardboundaries/2}] - a range $[d_{min};d_{max}]$ for a single board dimension,
    \item[\texttt{boardrelation/1}] - a relation denoting a single board field, like \texttt{cell} or \texttt{mark},
    \item[\texttt{boardpattern/k}] - a pattern describing the meaning of the board relation's arguments, 
    \item[\texttt{playfunctor/1}] - a relation denoting placing a piece on a board,
    \item[\texttt{playpattern/l}] - a pattern describing the meaning of the play relation's arguments.
\end{description}
\noindent
Additional reserved constants used in the relations are:
\begin{description}
    \item[\texttt{piece}] - a piece argument,
    \item[\texttt{dim}] - a board dimension,
    \item[\texttt{skip}] - a meaningless argument.
\end{description}

\vspace{0.5cm}
The added extensions require an additional:
\begin{definition}
\emph{(GDL extensions restriction)}.
Each extension relation only appears in ground atomic sentences. Let $\{d_{min}, d_{max}\}$, $\{t_1,\ldots,t_k\}$, $\{r_1\}$, $\{t'_1,\ldots,t'_k\}$, $\{r_2\}$ be the arguments of, respectively, \texttt{boardboundaries}, \texttt{boardrelation}, \texttt{boardpattern}, \texttt{playfunctor}, \texttt{playpattern}. They should meet the following conditions:
\begin{itemize}
    \item $d_{min}$, $d_{max}$ are valid real numbers (represented by decimals with an optional, "."-delimited fraction part), $d_{min} \leq d_{max}$,
    \item all values of $\{t_1,\ldots,t_k\}$ and $\{t'_1,\ldots,t'_k\}$ are either \texttt{piece}, \texttt{dim}, \texttt{skip},
    \item $\{r_1\}$ and $\{r_2\}$ are valid relations appearing elsewhere in the rule sheet.
\end{itemize}
\end{definition}

The extension behaves a lot like the \texttt{role} predicate - it resides in the KIF file and serves as a reference for the agent. If the agent does not support it, it will be transparent, as the sentences simply will not take place in the reasoning. When adding the extension, one would also have to ensure no name collisions would occur. Figure~\ref{fig:gdlExtensions} presents an example of the extension from a rule file.

\begin{figure}
\begin{SaveVerbatim}{VerbgdlExt}
;; The extension.	

(boardboundaries 1 8)
(boardfunctor mark)
(boardpattern dim dim piece)

(playfunctor play)
(playpattern piece skip skip dim dim)
\end{SaveVerbatim}
	
	\centering
    \setlength{\fboxsep}{5mm}
    \fbox{\BUseVerbatim{VerbgdlExt}}
    \caption{The GDL spatial extension applied to \textit{chess.kif}}
    \label{fig:gdlExtensions}
\end{figure}

\subsection{Spatial Features}
\label{subsec:features}

A typical feature is a function: 

\[ \delta :(s_n, M_{n-1}, m_n) \rightarrow \{0,1\}, \]

where $s_i$ is the $i-$th game state, $m_i$ is a move to be made in this state by the investigated player, and $M_i$ is a set of moves for all the players. Simply put, the function checks if the intended move meets particular conditions (creates a pattern, takes part in a formation, makes a capture, etc.) and acts as an indirect prediction for the next state.
\\

The following system of simple features has been designed with arbitrary games in mind. The prerequisites were:
\begin{itemize}
\item a hypercubical board,
\item players making move by indicating coordinates on the board (by placing, moving, removing pieces, etc.),
\item relations describing boards and moves, following the conventions close enough to be described by the extension.
\end{itemize}

\noindent
Each feature had a weight associated, and the features later took part in the agents' plays. The features included:

\begin{description}
\item[Proximity between the moves] A distance between the last and the current move. If many players have made the move during last turn, it's a distance to the closest one. This is the proximity feature that is well known to pay out in Go.

\item[Nearest border distance] A distance to the closest board border, expressed with a single, real number.

\item[Absolute piecewise move] A move made by a player, regardless of a state or any other contributing factor, weighted and marked as a good or a bad one. The feature has been inspired by Cadiaplayer's MAST technique described in Subsection~\ref{subsec:incorporation}.

\item[Absolute piecewise move in area] Presence of a piece in a specific board area. If square board's fields $(1,1)$ and $(1,2)$ belong to the same area, facts \texttt{(mark 1 1 x)} and \texttt{(mark 1 2 x)} would yield the same feature.

\item[K-nearest neighbors] A list of $k$ closest neighboring pieces (in respect to the field where a move is to be made), ordered lexicographically.

\item[K-nearest neighbors in one dimension] Similarly, a list of $k$ nearest pieces only in 1 dimension, but ordered by the distance. They would correspond to columns and rows on a $2$-dimensional plane.

\item[Itemsets only] A dummy feature, meant to be backed up with itemsets, that are enough interesting alone. More information on itemsets is provided in Subsection~\ref{subsec:itemsets}.
\end{description}

\subsection{Board Areas}

Some of the features rely on the concept of board areas. A board was divided into square (or hypercubical in general) areas, so the common features for neighboring fields could be captured. Let $d = |d_{max}-d_{min}|$ be the board size. Variable area size was chosen as

$$s = \left\lfloor \frac{(d+1)}{log_2 \left(d + 1\right) + 1} \right\rfloor .$$

\noindent
Areas were indexed with natural numbers. An area $A_{t_1,\ldots,t_n},\; t_i\in\mathbb{N}$ was defined as

$$\left\{ \left(r_1,\ldots ,r_n\right)^T: d_{min} + st_i \leq r_i < d_{min} +s(t_i+1),\quad i\in\left\{1,\ldots,n\right\} \right\} .$$

The definition was tailored specifically to $8 \times 8$ boards, so it would result in $2 \times 2$ ares. A sample area is shown in Figure~\ref{fig:area}. The upper borders of the area are not inclusive (marked with dashed lines) - this way an area of size $n$ spans over $n$ integer points in each of the dimensions.

\begin{figure}[h]
  \centering
      \includegraphics[width=0.5\textwidth]{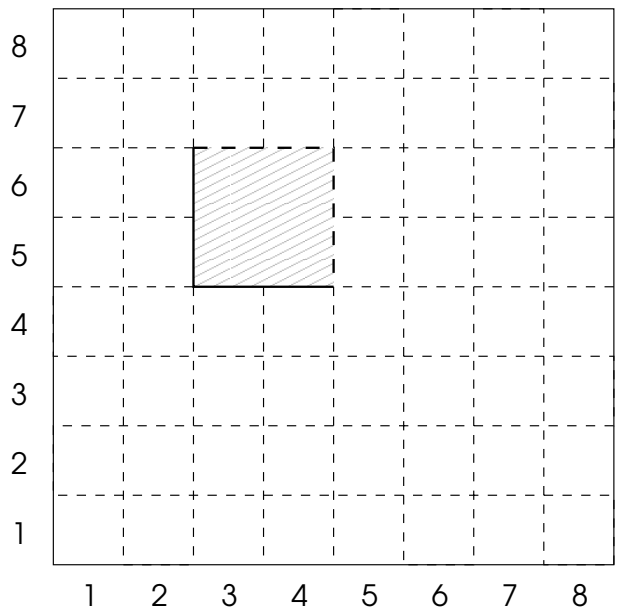}
  \caption{Area $A_{1,2}$ on an $8\times 8$ board. Edges marked with dashed lines do not belong to the area, therefore the areas do not overlap and cover the whole board.}
  \label{fig:area}
\end{figure}

\subsection{Meta Facts}

Apart from spatial features, which act like simple predictions for consecutive game states, a similar mechanism was used to express the statements about the current state only.
%The rationale is this: while some of the common features were backed with fact sets, the facts alone might have been too rare in some cases.
\\

Meta facts are realized by binary functions; a particular meta fact is either present in a state or not. They are aggregated into groups; each group $\psi$ is associated with a function 

% $$\psi : S \rightarrow \left\{(t_1,\ldots,t_n)^T : t_i\in \left\{0,1\right\} \right\} $$
\[ \psi : S \rightarrow \{0,1\}^n , \]

returning a binary vector depicting, if the corresponding features are present or not. A meta fact describing that something occurs in area $A_{i,j}$, would return a list of all the areas indicating, which are the interesting ones. Meta facts (groups) share the following properties:
\begin{itemize}
    \item $n \leq |S|$ (size of the vector is not greater than the number of facts),
    \item some of meta facts are more frequent than average facts.
\end{itemize}

Of course, far more sophisticated methods could have been employed to overcome the problem of finding associations between the pieces. However, such simple way of abstracting information about the game mimics the human ability of doing so.
\\

\noindent
So far, only two meta facts were used:

\begin{description}

    \item[Any piece in a field] A particular field having any piece on it. It might be useful for games with many pieces, but also for situations, when a distinction between pieces of the player and his opponent is not necessary.
    
    \item[Piece in area] Like the feature with a similar name, true when a certain piece lays within the area.
    
\end{description}

\section{Obtaining the Knowledge}
\label{sec:analyzer}

One of the goals of this thesis was to design a system, not only able to play an arbitrary (board) game, but also able to gradually learn it. Hence the assumption, that the rules are known for an undefined amount of time before the actual playout.
%Knowledge format was flexible enough, so the accurate parameters and features could be pre-computed even before the game, rather than during the startclock phase. 
For the computations were very time consuming, hours where expected, rather than minutes or seconds, as it is usually the case with GGP.
\\ 

A tool external to the agent, Knowledge Miner (also referenced as the analyzer), has been designed to perform the learning process, thus completing the spatial board knowledge approach made to address the learning issue in this publication. The following sections describe a variant of Simple Genetic Algorithm along with the feature finding scheme, used by the analyzer for developing a strong knowledge file.

\subsection{Self-Play Through Modified SGA}

With no game records available for an arbitrary game, a scheme with a population of competing agents has been created with the following goals in mind:
\begin{itemize}
    \item producing game records of greater quality, than with bare UCT agents,
    \item gathering meaningful features by analysis of those records,
    \item making crucial choices concerning utilization of features,
    \item fine-tuning knowledge parameters.
\end{itemize}

Simple Genetic Algorithm~\cite{Goldberg:1989:GAS:534133} has been used as a basis for the evolutionary algorithm. Skill of individuals was meant to gradually improve with every generation, as well as the quality of their game knowledge. The individuals have been scored according to their performance during inter-generation matches. Apart from regular evolutionary operators, separate feature mining took place (using game records gathered during the previous generation). Then, after applying evolutionary operators, the agents were allowed to update their knowledge with newly mined features, according to their individual feature-learning policies. The goal was to improve upon both the knowledge and game records, to finally develop the ultimate game knowledge. The algorithm is shown in Algorithm~\ref{alg:evoAnalyzer}.
\\

Mining features for each individual alone has been rejected as too time- and resource- consuming. Because of this shortcoming, at the knowledge mining phase, the knowledge has been mined from the entire population's game records combined. Of course having many agents sharing the same knowledge would be pointless; to bring back the diversity, each agent had a simple learning policy and a set of varying knowledge parameters. The parameters were also subjected to evolution.

\begin{algorithm}[h]
 \SetAlgoLined
 \KwData{GDL rule sheet}
 \KwResult{Knowledge file}
 \BlankLine

$P \gets Initial-Population(n)$\;
\While{not $Termination-Condition(P)$}{
    \For{$r$ rounds}{
        $G \gets Randomly-Schedule-Games(P)$\;
        \For{game $g\in G$}{
            $R = R \cup \{result(g)\}$\;
        }
    }
    $P_p \gets Parent-Selection(P,R)$\;
    $P_c \gets Crossover(P_p)$\;
    $Muation(P_c, \theta)$\;
    $Replacement(P, P_c)$\;
    $F \gets Mine-Features(R)$\;
    \For{individual $I\in P$}{
        $Update-Using-Policy(I,F)$\;
    }
}

\caption{The genetic algorithm based on SGA. GGP Agent knowledge objects (knowledge files) constitute as individuals. The files are scored depending on quality of plays between agents using them. Instead of taking raw mean of number of victories, Elo ranking system is used for better convergence. During the Mine-Features phase, game records from previous generations are being analyzed, revealing new candidate features. Knowledge files are being updated with candidate features according to their learning parameters.}
\label{alg:evoAnalyzer}
\end{algorithm}

\paragraph{Chromosome}

The basic chromosome has been made a string of $\mathbb{R}$, $\mathbb{N}$, and 0-1 values. The values represented different knowledge parameters (more on parameters in Section~\ref{subsec:gameKnowFormat}). Because of this heterogeneity of genes, uniform crossover has been employed. Each chromosome has been also paired with a feature list. The list were subjected only to the crossover operator: the resulting list was a list of $n_{max}$ best features (in respect to their weights) from a set of unique features, from both of the parents' lists combined. The number $n_{max}$ has been stored in the chromosome. Updates of feature lists (with the population's newly mined knowledge) have been carried in a similar fashion.

\subsection{Feature Mining}

Following the ideas from Go (Section~\ref{sec:goKnowledge}) of analyzing professional game records, a similar scheme has been employed for mining features. Moves can be associated with states, in which they were made, of both winning and loosing parties. The games were not assumed symmetric (or almost-symmetric); while playing with white and black pieces in chess is almost the same in theory, features for both colors were chosen separately.
\\

To begin mining on a set of records, all the records were loaded; for each player and each state, the occurring features were recognized and divided into sets $F_w$ and $F_l$, winning and loosing feature sets respectively.
\\

To recognize correlations of certain features with success or failure, the phi correlation coefficient was used. For each feature $\delta$, the following values were computed:
\\

\begin{center}
\begin{tabular}{ c|c|c|c| }
\cline{2-4}
& \# of winning & \# of loosing & \multirow{2}{*}{total}\\
& states & states & \\
\hline
\multicolumn{1}{ |c| }{ $\delta$ present } & $n_{0,0}$ & $n_{0,1}$ & $n_{0,*}$ \\ \hline
\multicolumn{1}{ |c| }{ $\delta$ absent } & $n_{1,0}$ & $n_{1,1}$ & $n_{1,*}$ \\ \hline
\multicolumn{1}{ |c| }{total} & $n_{*,0}$ & $n_{*,1}$ & $n_{*,*}$ \\ \hline
\end{tabular}
\end{center}

The phi coefficient for correlation between feature occurrence and winning was:

$$\phi_{good} = \phi = \frac{n_{0,0}n_{1,1}-n_{0,1}n_{1,0}}{\sqrt{n_{*,0}n_{0,*}n_{*,1}n_{*,1}}}.$$

Additionally, the coefficient between feature occurrence and loosing required swapping table columns; specifically:

$$\phi_{bad} = \frac{n_{0,1}n_{1,0}-n_{0,0}n_{1,1}}{\sqrt{n_{*,0}n_{0,*}n_{*,1}n_{*,1}}} = -\phi_{good}.$$

It follows that $-1 \leq \phi_{good}, \phi_{bad} \leq 1$.
\\

Features, that exceeded the confidence threshold $r$, that is $|\phi| > r$ were picked as correlated with winning or loosing (depending on the sign of $\phi$). The $|\phi|$ value also influenced the initial weights for those features.

\subsection{Apriori-Like Algorithm for Fact Sets}
\label{subsec:itemsets}

Mined features, while having $\phi > r$, underwent an additional fact set mining phase. The point of that phase was to back the features with frequent item sets which, while lowering the chance of their occurrence, would strengthen the correlation.
\\

Each game state was treated as a basket of facts it consisted of. A simple variation of the Apriori Algorithm~\cite{Agrawal:1994:FAM:645920.672836} was applied to find out which fact sets are associated with winning, and which with loosing (also having the mentioned feature present). Unfortunately, Apriori turned out to be cumbersome when working with the two pools; i.e. mining winning and loosing fact sets separately, and then subtracting the sets from each other, seems like a huge overkill. Perhaps it would be more natural to somehow subtract the states \textit{before} running Apriori. The encountered problem could be formulated as follows:

\begin{problem}
A basket is a set of objects. Given two sets of baskets: $D$ (\textit{desirable}) and $U$ (\textit{undesirable}), two constants: $d_{min}$, $u_{max}$, find all the itemsets, which have the support $\geq d_{min}$ in the $D$ set and $ \leq u_{max}$ in the $U$ set.
\end{problem}

The problem is rather hard to tackle; the Apriori itself might, under the right circumstances, return an exponentially large output, and a subpart of the algorithm is NP-complete~\cite{AprioriPerf04}. The hard part is, however, finding itemsets frequent and infrequent (in different sets) at the same time.

\subsubsection{The downward closure lemma}

The following auxiliary relations simplify the stated problem. Let $D=\{t_1,t_2,\ldots,t_n\}$ be the set of baskets, $I$ an itemset, $r_f,r_i\in \mathbb{R}$:
\begin{eqnarray}
\textit{freq}(I) \Leftrightarrow \textit{supp}(I) \geq r_f,\\
\textit{infreq}(I) \Leftrightarrow 1 - \textit{supp}(I) \geq r_i.
\end{eqnarray}
The problem comes down to finding itemsets both frequent in $D$ and infrequent in $U$. Finding frequent itemsets belongs to the Apriori Algorithm. The algorithm is based on the downward closure lemma, which guaranties that larger (candidate) itemsets may be obtained by extending smaller ones. However, extending itemsets does not work with infrequent itemsets, where  larger itemsets could be narrowed down to obtain smaller candidate itemsets. The following lemmas express these properties through \textit{freq}/\textit{infreq} relations:

\begin{lemma}
Let $I_n$ be an itemset of size $n$. For $I_{k-1} \subseteq I_k$, the following hold:
\begin{eqnarray}
\label{eq:lemma1.1}
\textit{freq}(I_k) \Rightarrow \textit{freq}(I_{k-1}),\\
\label{eq:lemma1.2}
\textit{infreq}(I_k) \Leftarrow \textit{infreq}(I_{k-1}).
\end{eqnarray}
\end{lemma}

\begin{proof}
From the definition of support:
$$\textit{supp}(I_k) = \frac{ |\{t\in D : I_k \subseteq t\}| }{|D|}.$$
Because $I_{k-1} \subseteq I_k$, then
$$\{t\in D: I_k \subseteq t\} \subseteq \{t\in D: I_{k-1} \subseteq t\}$$
so $\textit{supp}(I_{k-1}) \geq \textit{supp}(I_k)$. It follows from definitions of \textit{freq} and \textit{infreq} that
\begin{eqnarray*}
\textit{supp}(I_{k-1}) \geq \textit{supp}(I_k) \geq r_f,\\
1 - \textit{supp}(I_k) \geq 1 - \textit{supp}(I_{k-1}) \geq r_i.
\end{eqnarray*}
\end{proof}

\begin{lemma}
Relations inverse to~\ref{eq:lemma1.1} and~\ref{eq:lemma1.2}, that is
\begin{eqnarray}
\label{eq:lemma2.1}
\textit{freq}(I_k) \Leftarrow \textit{freq}(I_{k-1}),\\
\label{eq:lemma2.2}
\textit{infreq}(I_k) \Rightarrow \textit{infreq}(I_{k-1}),
\end{eqnarray}
do not hold.
\end{lemma}

\begin{proof}
Suppose (by way of contradiction) that the relations hold. Let $D=\{t\}$, where $|t|=m$, be a basket set, $I_1=\{e\}$ be an itemset s.t. $e\notin t$. It follows, that $\textit{supp}(I_1)=0$. For $I_m = t$, $\textit{supp}(I_m)=1$ and $\textit{freq}(I_m)$. It follows from~\ref{eq:lemma2.1} that
$$\textit{freq}(I_m) \Rightarrow \textit{freq}(I_m \cup I_1).$$
But because $\textit{supp}(I_1)=0$, then $\textit{supp}(I_m \cup I_1) = 0$, so $\neg\textit{freq}(I_m \cup I_1)$, which contradicts the above.
Similarly, it may be shown that, while $\textit{infreq}(I_m \cup I_1)$, it would follow from~\ref{eq:lemma2.2} that $\textit{infreq}(I_m)$, which is also a contradiction.
\end{proof}

\noindent
With those properties proved, the downward closure lemma can be introduced.
%Because (*), candidate itemsets can be build out of smaller itemsets. Because of (**), this is not feasible for \textit{infreq} relation.
%First of all, Apriori algorithm is based on a observation that all itemsets of size $k$ might be obtained by extending those of size $k-1$. It is commonly formulated as a following lemma:

\begin{lemma}[The downward closure lemma]
Let $D=\{t_1, t_2, \ldots, t_n\}$ be the set of transactions, and $L_i$ a set of all frequent itemsets of size $i$ in $D$. For
$$C_k = \{c:c=a\cup \{b\} \wedge a \in L_{k-1} \wedge b\in \bigcup L_{k-1} \wedge b\notin a \},$$
$L_k \subseteq C_k$, where $k \geq 2$.
\end{lemma}

\begin{proof}
$L_k$ can be formulated in terms of the \textit{freq} relation as:
$$L_k=\{I:|I|=k \wedge freq(I)\}.$$
Suppose $L_k \nsubseteq C_k$. Then there exist an itemset $I_k\in L_k$ s.t. $I_k \notin C_k$. Because $k \geq 2$, then $I_k=I_{k-2} \cup \{b_1, b_2\}$ for some $b_1,b_2\in I_k$. From~\ref{eq:lemma1.1} it follows that
$$\textit{freq}(I_k) \Rightarrow \textit{freq}(I_{k-2} \cup \{b_1\}),$$
so $I_{k-2} \cup \{b_1\} \in L_{k-1}$ and $b_1 \in \bigcup L_{k-1}$. Similarly
$$\textit{freq}(I_k) \Rightarrow \textit{freq}(I_{k-2} \cup \{b_2\}),$$
so $I_{k-2} \cup \{b_2\} \in L_{k-1}$. But it means that $(I_{k-2} \cup \{b_2\}) \cup \{b_1\} \in C_k$, which is a contradiction, taking into account the formula for $C_k$. There could be no $I_k\in L_k$ s.t. $I_k \notin C_k$, so $L_k \subseteq C_k$.
\end{proof}

Because~\ref{eq:lemma2.1} and~\ref{eq:lemma2.2} do not hold, infrequent itemsets cannot be build from smaller ones. An example would be a set of $20$ cars: $10$ red, and $10$ fast. While being red and being fast are certainly frequent characteristics, being both red and fast is not, as there is no such car in the set.
\\

Wu et al.~\cite{negativeAssociationRules04} solved a similar problem by narrowing down mined infrequent itemsets to only those, for which their every subset is also infrequent. In other words, the paper explored the possibility of applying the downward closure lemma to infrequent itemsets regardless of the aforementioned issue, arguing that such infrequent itemsets were less likely accidental. Itemsets found through this approach satisfy the relation:
$$\textit{infreq}(I_k) \Leftrightarrow \textit{infreq}(I_{k-1}).$$

In this publication, taking into account that desired fact sets were rather small (so they could be quickly computed), and assuming small baskets (number of elements close to the number of board fields/pieces), the problem has been approached with a naive solution for smaller itemsets, and the downward closure lemma for larger ones. It is presented in Algorithm~\ref{alg:myApriori}.
\\
\begin{algorithm}[h]
 \SetAlgoLined
 \KwData{Basket sets $D,U$, \textit{freq}/\textit{infreq} thresholds $\epsilon_d,\epsilon_u$}
 \KwResult{Itemsets both frequent in $D$ and infrequent in $U$}
 \BlankLine
 $n \gets MaxBruteforceFactsetSize(D, U)$\;
 \For{$k=1 \to n$}{
    \For{basket $b \in D$}{
        $C \gets \{c:c\subseteq b \wedge |c| = k\}$\;
        \For{candidate $c \in C$}{
            $count_d[c] \gets count_d[c] + 1$\
        }
    }
    \For{basket $b \in U$}{
        $C \gets \{c:c\subseteq b \wedge |c| = k\}$\;
        \For{candidate $c \in C$}{
            $count_u[c] \gets count_u[c] + 1$\;
        }
    }
    $L_k \gets \{c:count_d[c] \geq \epsilon_d \wedge count_u[c] \leq \epsilon_u\}$\;
 }
 \For{$k=n+1 \to n_{max}$}{

    $C \gets \{c:c=a\cup \{b\} \wedge a \in L_{k-1} \wedge b\in \bigcup L_{k-1} \wedge b\notin a \}$\;
    \For{candidate $c \in C$}{
    
        \If{$|\{b: c\subseteq b \wedge b\in D \}| \geq \epsilon_d \wedge |\{b:c\subseteq b \wedge b \in U\}| \leq \epsilon_u$}{   
            $L_k \gets L_k \cup \{c\}$\;
        }
    }
 }
\Return{$\bigcup_k L_k$}
\caption{Apriori-like feature set mining algorithm, which works simultaneously on two transactions sets $D$ and $U$. The point of algorithm is to find itemsets frequent in $D$ and infrequent in $U$ at the same time. It works in two phases: small itemsets are found with a naive approach. Then, larger itemsets are being build out of smaller ones as in regular Apriori, though the approach is not accurate with infrequent itemsets.}
\label{alg:myApriori}
\end{algorithm}

The algorithm explores the aforementioned idea of applying the downward closure lemma to infrequent itemsets, even though it would only return some of the itemsets. It begins with a brute force search for small itemsets (size $k\leq3$ was intended). Starting from $k+1$, it then applies the downward closure lemma. The mining takes place in baskets $D$ and $U$ in parallel.

\chapter{Testing Platform}
\label{chap:testingPlatform}

This chapter introduces GGP Spatium - a GGP framework created particularly for this publication. It also covers the environment, in which the framework has been employed for evaluation of the approach presented in the previous chapter.

\section{GGP Spatium Framework}

An extensive framework, named GGP Spatium, has been implemented to evaluate the ideas presented in Chapter~\ref{chap:keyIdeas}. Targeting Unix platforms, it consists of nearly $35000$ lines of code, written mostly in C++ and Python, and has been released under GNU LGPL license~\cite{lgpl} as an addition to this publication. It should be noted that there exist similar projects, like Java-based GGP Galaxy~\cite{ggpGalaxy13} or RL-GGP~\cite{rlggp13}, which is also Java-based and constitutes a testbed for reinforced learning algorithms for GGP. However, due to different principles of those projects and technologies involved, it was more convenient to create an independent framework.
\\

The main components of the system are:
\begin{description}
  \item[GGP Agent] - (C++), an UCT-based agent capable of carrying regular GGP matches using the protocol,
  \item[GDL To C++ Generator] - (C++, Python), a tool generating reasoning machine C++ source code out of mGDL rule sheets,
  \item[Knowledge Miner] - (Python), a tool, based on an evolutionary algorithm, which conducts plays and analyzes game records of population of agents in order to produce a knowledge file,
  \item[Game Runner] - (C++), a small tool reusing most of the agent's code in order to accurately measure raw performance of supplied reasoning engines,
  \item[Game Library] - a simple database of rule sheets, GTC engine source code files, compiled .so libraries and game knowledge, maintained automatically by the agents.
\end{description}
The framework also includes miscellaneous scripts, i.e. mGDL conformance checker or a knowledge file scorer. Layer architecture diagrams, illustrating the key components, are shown in Figure~\ref{fig:arch}. UML diagrams of those components are presented in Appendix~\ref{apdx:uml}. The following sections describe them in detail.
\\

\begin{figure}[h]
\centering
\begin{tabular}{cc}
\subfloat[GGP Agent]{\includegraphics[width=0.45\textwidth]{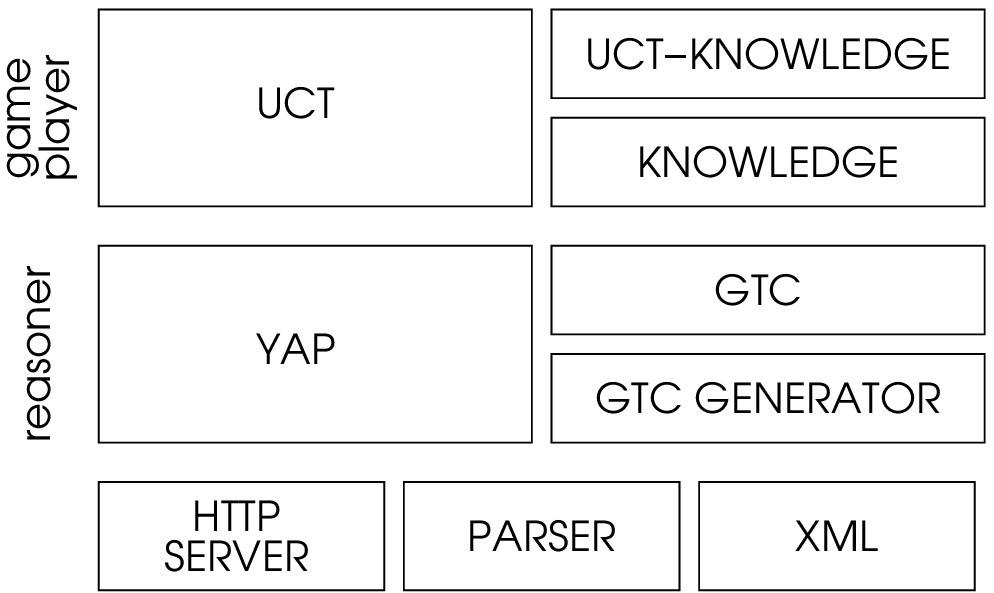}} 
\hspace{1.5cm}
\subfloat[Knowledge Miner]{\includegraphics[width=0.4\textwidth]{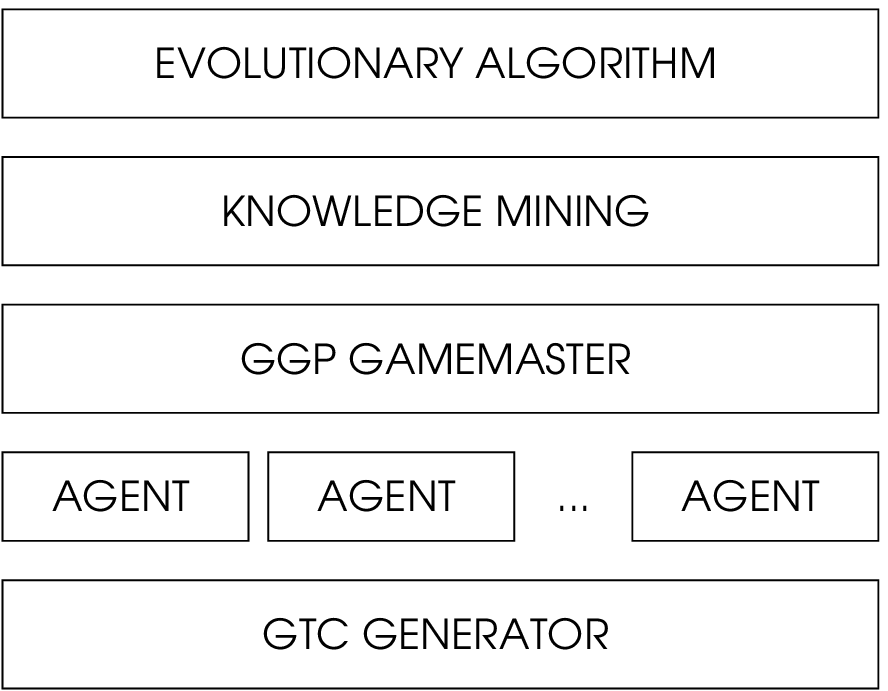}} \\
\vspace{0.5cm}
\end{tabular}
\subfloat[GTC Generator (GDL To C++ Generator of reasoning engines)]{\includegraphics[width=0.4\textwidth]{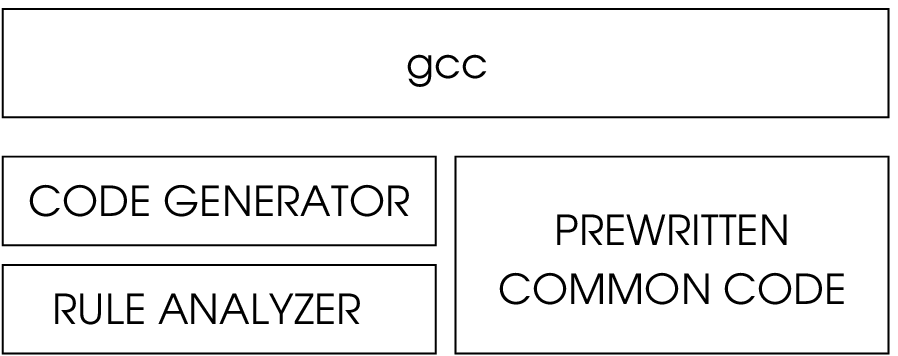}}
\caption{Layer diagrams of key components of the system. UML diagrams are presented in Appendix~\ref{apdx:uml}.}
\label{fig:arch}
\end{figure}

\subsection{GGP Agent}

The agent follows the popular design described in Section~\ref{sec:agentArchitecture}. It employs UCT, has a regular network interface to handle GGP game protocol, and an interface for playing algorithms (UCT, UCT with knowledge, purely random etc.).
\\

What may distinct the agent from other similar ones, is the implementation of concepts presented in Chapter~\ref{chap:keyIdeas}, with a few additions for testing purposes:
\begin{itemize}
	\item two reasoning engines supplied by default: YAP and GTC,
	\item a common interface for different reasoning engines,
	\item a common interface for playing algorithms,
	\item transposition table, simple and linked,
	\item an Euclidean, feature-based knowledge module, able to load XML knowledge files (described later in this chapter),
    \item ability to save XML game records (Figure~\ref{fig:xmlrecord}) along with some internal data about the plays,
	\item options for artificially limiting transposition table size, altering game clocks, enabling CPU profiling.
\end{itemize}

The system of swapping reasoning engines deserves a broader attention. At default, at the beginning of each game, the agent tries to generate, compile and load a GTC reasoning engine. On failure, it falls back to YAP Prolog. Generated GTC files are stored in the library of known games, so a rule sheet would not have to be processed more than once. It is also possible to supply a custom-written C++ engine, along with a rule sheet it corresponds to. Lastly, the agent can be compiled statically with one, chosen game engine - be that GTC or a custom one.
\\

New playing algorithms can also be developed, by extending appropriate classes. Game clock multiplier is useful either for testing agents with uneven computation times, or conducting plays with clocks under 1~s with high-performance reasoning engines. Along with possibility of limiting transposition table memory in a sensible way, the framework provides means for creating agents not only as resource-intensive GGP contestants, but also compact, general-purpose problem solvers with network interfaces.
%Such AIs might find application even to regular, desktop applications.
\\

\begin{figure}
\begin{SaveVerbatim}{Verbxmlrecord}
<Match Id="MATCH001">
  <Player>
    <Role>black</Role>
    <Score>6</Score>
  </Player>
  <Player>
    <Role>red</Role>
    <Score>58</Score>
  </Player>
  <State Number="0">
    <Fact>(cell 4 4 black)</Fact>
    <Fact>(cell 4 5 red)</Fact>
    <Fact>(cell 5 4 red)</Fact>
    <Fact>(cell 5 5 black)</Fact>
    <Fact>(control black)</Fact>
    <Move>
      <Role>black</Role>
      <MoveFact>(mark 3 5 BLACK)</MoveFact>
    </Move>
    <Move>
      <Role>red</Role>
      <MoveFact>noop</MoveFact>
    </Move>
  </State>
\end{SaveVerbatim}
	
	\centering
    \setlength{\fboxsep}{5mm}
    \fbox{\BUseVerbatim{Verbxmlrecord}}
    \caption{Sample excerpt from a game record XML file}
    \label{fig:xmlrecord}
\end{figure}

\subsection{GDL to C++ Generator}

GTC Generator is responsible for development and compilation of GTC reasoning machines. The operation is meant to be carried in the startclock phase of a match. Generator takes a rule sheet on input; it computes it's hash sum and stores the generated code in the game library. Afterwards, it attempts compilation.
\\

Common parts of code involve data structures and realization of the agent's reasoning engine interface. For a particular GDL rule sheet, only C++ source code files containing specific routines are being generated to complete the GTC engine, along with a GNU Make makefile for the ease of later compilation. On successful compilation, a shared object dynamic linked .so library\footnote{A common library format for Unix systems.} is being created and stored in the Gamelib.
\\

The generator is capable of generating reasoning machines based on both schemes presented in Subsection~\ref{subsec:gtcModels}. Said section also provides sample code excerpts for both methods.

\subsection{Evolutionary Knowledge Miner}

Principal of operation of the knowledge miner has been presented in depth in Section~\ref{sec:knowledgeMining}. However, gathering a vast number of game records is not straightforward, and requires an additional comment.
\\

The goal of knowledge miner is to develop a single game knowledge XML file. The miner is designed to work in a Local Area Network, where it also acts as a server coordinating client workstations. All workstations, along with the server, are assumed to have access to shared Network Attached Storage. Sample network diagram illustrating devices participating in mining the knowledge is shown in Figure~\ref{fig:lan}.
\\

\begin{figure}[h]
  \centering
  \includegraphics[width=0.7\textwidth]{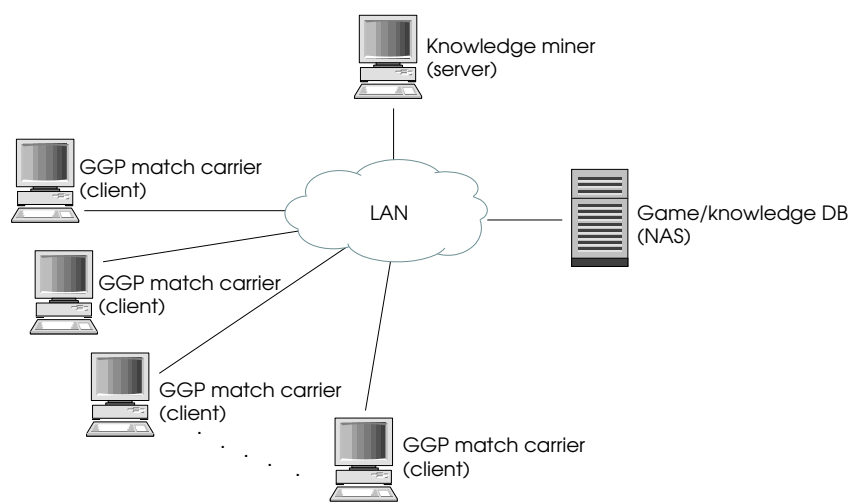}
  \caption{Sample network setup for carrying GGP matches and mining the knowledge with the Knowledge Miner}
  \label{fig:lan}
\end{figure}

Client workstations spawn GGP Game Managers along with GGP Agent instances, carry out GGP matches, and save XML game logs to the shared space. The miner stores intermediate agents' knowledge in XML files in the shared space, so the agents could access it. No extra locking mechanisms are needed, because the server and clients write the shared data in different phases of the mining algorithm. Client-server communication takes place over the network, using TCP/IP protocol.
\\

A special case of described network setup is running the server and clients on a single machine. In that case, the communication takes place through the loopback interface, and the shared space might be i.e. the computer's hard drive.

\subsection{Transposition Table}

The goal of the transposition table (TT for short) is to speed up the game graph search and keep the agent from consuming too much resources. The following design is an extended version of the one in~\cite{holt08}. The paper gave an outline of a Java-based GGP agent; because for demanding games TT would take up most of available memory, the maximum size of TT was limited by the memory available to the whole Java Virtual Machine. Whenever the VM would go out of RAM, it would split the Transposition Table in half (thus deleting a "random" half of states).
\\

The transposition table used in this publication is similar, with an exception of using a linked hash set\footnote{Terminology borrowed from the \texttt{java.util} Java package.} rather than an ordinary one. A linked hash set stored an additional list of nodes, which was used for better memory management. Because states also had pointers to other states, the transposition table became an interesting combination of three data structures:
\begin{itemize}
    \item a hash set,
    \item a directed graph,
    \item a linked list.
\end{itemize}
Every reasoning engine implementing the agent's reasoning interface shipped it's own state class, equipped with a hashing function. The state's hash was it's key in the hash set.
\\

The linked list served as a priority queue. A new state inserted to the transposition table was also being inserted at the front of the queue. Each state was moved to the front of the queue upon visiting. If the maximum number of states had been reached, the last item queue item was being deleted, which turns out to be the one visited last. This strategy worked well with UCT; after all, UCT keeps the search "fair", so even average states get visited from time to time, and thus brought up to the front of the queue. With large enough memory limits, the game graph held only relevant parts of the game tree.
\\

Some game's GAs are graphs rather than trees (i.e. in chess the same state might be reached a few times since the players might be moving back and forth with the same pieces). Instead of flushing the transposition table or pruning the game graph with every turn, they were left intact. The old, possibly unreachable states, were deleted in the first place, since they were last in the queue anyway.
\\

\begin{figure}[p]
\centering
\begin{tabular}{cc}
\subfloat[Initial transposition table: a hash set, a linked list and a graph]{\includegraphics[width=0.44\textwidth]{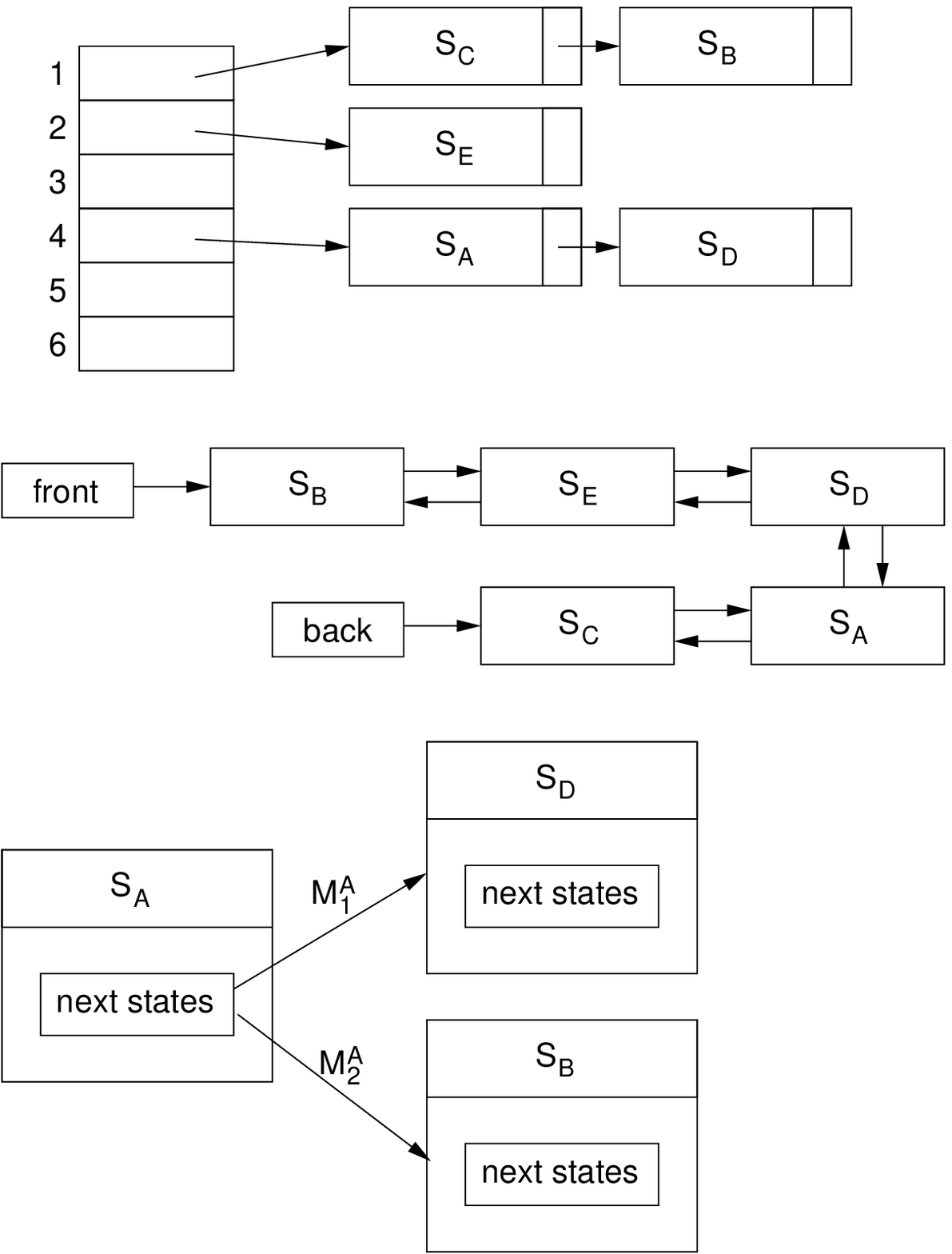}} 
\hspace{1cm}
\subfloat[After accessing state $S_D$ it is being brought to the front of the list]{\includegraphics[width=0.44\textwidth]{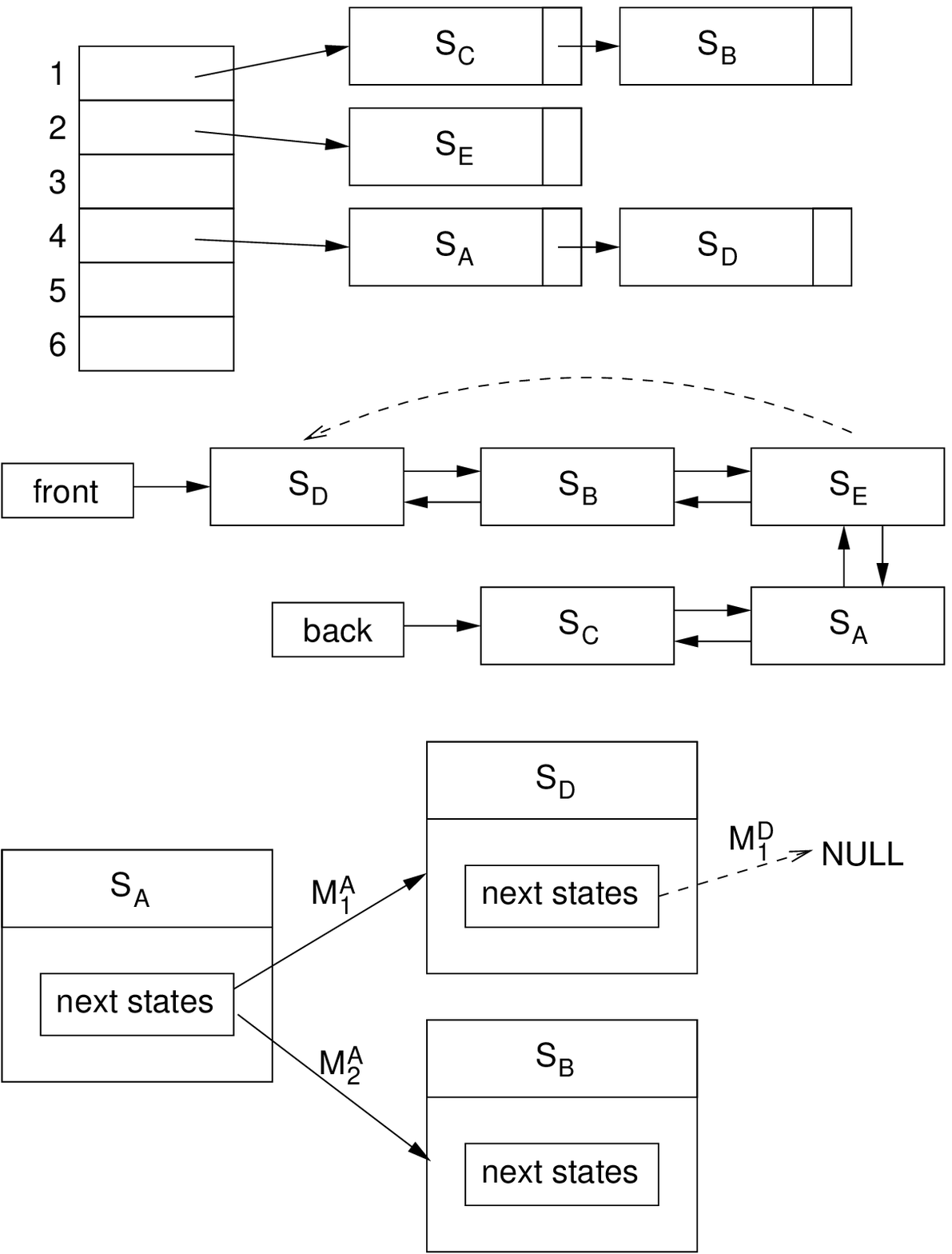}}\\
\vspace{0.5cm}
\end{tabular}
\subfloat[$S_F$ is a new state inferred from $(S_D, M_1^D)$; if the limit is reached, the last state from the list ($S_C$ in this case) is being deleted]{\includegraphics[width=0.50\textwidth]{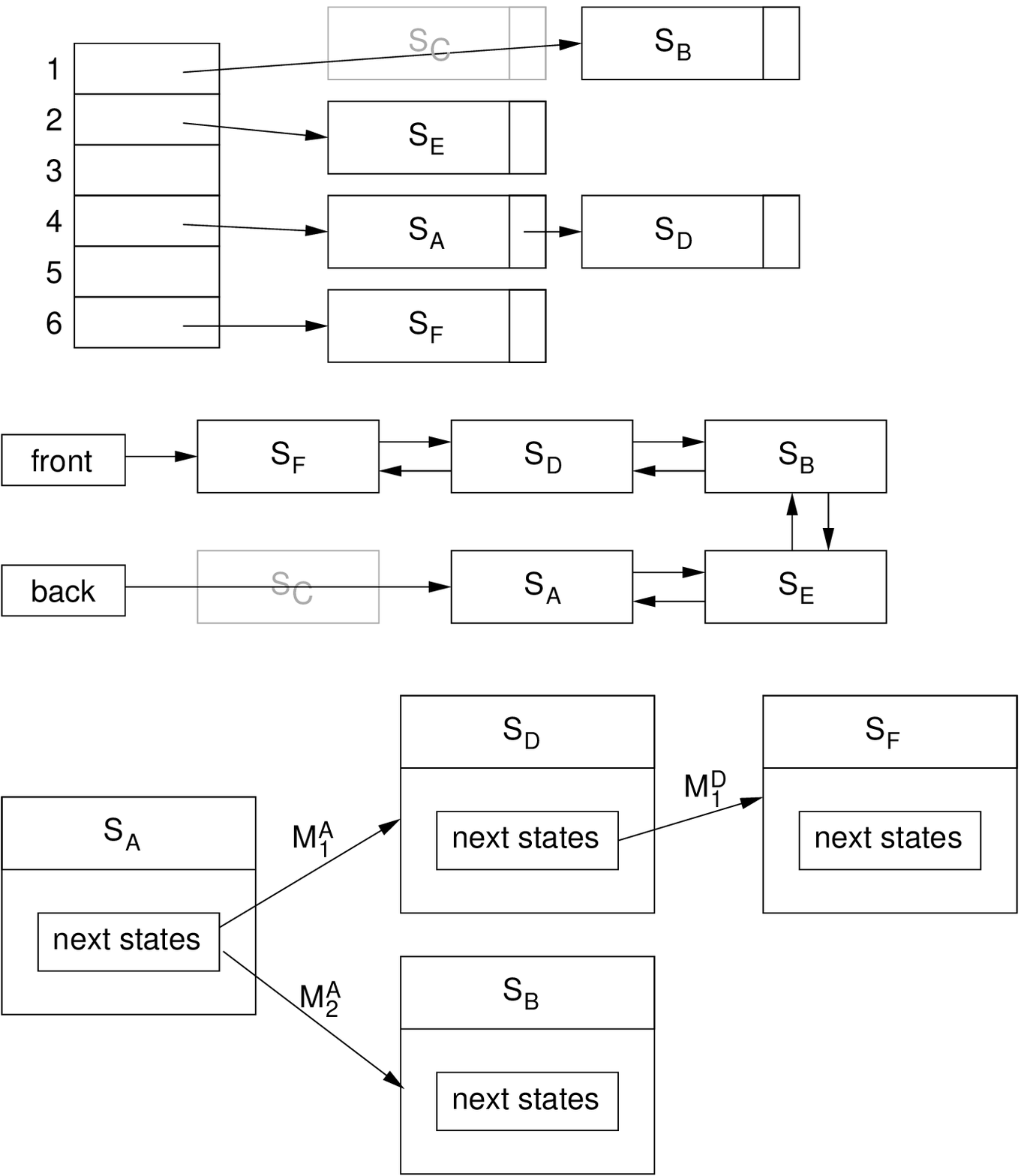}}
\caption{Transposition table behavior upon accessing/adding a state}
\label{fig:tt}
\end{figure}

Figure~\ref{fig:tt} shows an example of transposition table's behavior when accessing/adding a state. Each state contains pointers to reachable, already visited states (indexed with moves). Upon visiting, state $S_D$ is being moved to the front of the linked list. The move set $M^D_1$ is to be made during the UCT search phase. There is no path from $S_D$ labeled with $M^D_1$, so a new state is created. If a state with such hash value has already been stored in the hash set, the missing path is added; otherwise the state $S_F$ is inserted to TT. $S_F$ is brought to the front of the linked list. The size limit has been reached, so the last state in the linked list, $S_C$, is being deleted. The design involves usage of some kind of smart pointers, so no additional updates on the game graph would be necessary.
\\

Finally, the maximum size of the TT during runtime can set as simply the maximum number of states, or more precisely as a rough size in bytes by estimating an average size of a state during runtime. The design provides an interface for states having self-size estimating function.

\subsection{Game Knowledge Format}
\label{subsec:gameKnowFormat}

Recalling Section~\ref{sec:knowledgeMining}, the proposed spatial game knowledge consisted of features and meta facts, and relied on the underlying board as well as artificial board areas. Furthermore, if possible, features could be backed with frequent item sets, improving their quality.
\\

Although a few ways of improving UCT with features were discussed throughout this publication, no specific was chosen for the proposed approach. Weights for feature classes, number of features and other different parameters susceptible to evolution were considered parts of the knowledge as well. The proposed evolution algorithm worked with individuals, conceptually being the agents, but practically being game knowledge instances their possessed. The self-play scheme, a part of the evolution, required those agents to be ready to conduct matches during every generation, right after the knowledge update phase.
\\

For the stated reasons, per-game knowledge has been designed as a single XML file, encapsulating features, item sets, and knowledge parameters. Other data about individuals, such as their evolution history or quality in respect to the objective function, were stored in small auxiliary files. This way, all agents shared one single executable, with different knowledge files supplied to be loaded.
\\

The structure of said XML knowledge files was simple:
\begin{verbatim}
    <Knowledge>
       <Parameters>
          ...
       </Parameters>
       <Player role="red">
          <WinningFeatures>
             ...      
          </WinningFeatures>
          <LoosingFeatures>
             ...
          </LoosingFeatures>
       </Player>
       <Player role="black">
          ...
       </Player>
    </Knowledge>
\end{verbatim}

\paragraph{Knowledge parameters}

The parameters section had the aforementioned knowledge parameters: from ways of using the knowledge, to particular weights. Of course the agent has been programmed to correctly interpret those. Winning and loosing feature lists were, possibly lengthy, lists of features correlated with winning and loosing respectively. Thus, the knowledge files could range from several kilobytes, to a few megabytes in size, depending on the number of features. An excerpt from a sample knowledge file, showing features and item sets, is shown in Figure~\ref{fig:xmlknow}.
\\

The parameters, which were coded as chromosome genes, included:
\begin{multicols}{2}
\begin{itemize}
    \item maximum knowledge size
    \item weights for particular feature classes
    \item learning factor
    \item weights on winning/loosing features
    \item weights specifying the overall impact of knowledge
    \item parameters for progressive widening
    \item boolean values for:
    \begin{itemize}
        \item feature classes
        \item feature classes' item sets
        \item using item sets in selection and simulation UCT phases
        \item progressive widening
        \item first-feature scoring
        \item using features in the selection phase
    \end{itemize}
\end{itemize}
\end{multicols}

The basic form of using the knowledge, is for defining probability distribution for moves during pseudo-random playouts, according to the formula
$$P(m_k|s_n) =  \frac{\textit{score}(m_k)}{\sum_{i\in I} \textit{score}(m_i)},$$
with $\textit{score}(m_i)=\textit{baseValue} + \textit{winningFeatures}(m_i) - \textit{loosingFeatures}(m_i)$, and $s_n$ denoting the state, $m_k$ denotes one of the legal moves. First-feature scoring is the method of, when being in state, scoring a move with the first matching feature (from a sorted list). The default knowledge scoring method goes over the whole list and sums weights of matching features. The parameter for using features also in the selection phase, employs Formula~\ref{formula:progressiveUCT} (Subsection~\ref{subsec:incorporation}).

%Because the correlations might not have been that strong, or that many as expected, the lists were complemented with features backed with \textbf{frequent itemsets} of state facts. Such features were marked as "heavy"\footnote{Meaning: heavy to compute.}, in opposition to the remaining "light" features. Also, the means were provided for ignoring certain classes of features, with a further distinction for light and heavy.
%\\
%
%In the search phase, the heavy features were used in progressive widening (described in Subsection~\ref{subsec:incorporation}); widening parameters were also supplied by the knowledge file. In the playout phase, light features might have been used either as a probability distribution (determined by their weights with some scaling parameters adjustable), or as an ordered list of expert rules, where the first matching move was picked. A shift between good/bad features, which corresponds to offensive/defensive style of playing, could also be set.

\begin{figure}
\begin{SaveVerbatim}{Verbxmlknow}
<Knowledge>
...
   <Player role="black">
      <WinningFeatures>
       <FeatureRelEuclidKNearest weight="0.5">
          <K>3</K>
          <Pieces> redpiece redpiece redpiece </Pieces>
          <Itemset>
             <MetafactPieceInArea>
                <AreaSize>2</AreaSize>
                <AreaDimensions>7 1</AreaDimensions>
                <Piece>blackpiece</Piece>
             </MetafactPieceInArea>
             <MetafactAnyPieceInField>
                <Position>7 1</Position>
             </MetafactAnyPieceInField>
             ...
          </Itemset>
       </FeatureRelEuclidKNearest> 
       <FeatureRelEuclidProximity weight="0.3">
          <Distance>1</Distance>
       </FeatureRelEuclidProximity>
       <FeatureAbsEuclidBorderDist weight="0.1">
          <Distance>0</Distance>
          <Lower>True</Lower>
          <Dimension>2</Dimension>
       </FeatureAbsEuclidBorderDist>
          ...
\end{SaveVerbatim}
	
	\centering
    \setlength{\fboxsep}{5mm}
    \fbox{\BUseVerbatim{Verbxmlknow}}
    \caption{A sample excerpt from a knowledge XML file}
    \label{fig:xmlknow}
\end{figure}

\section{Performance Tunning}

The following section explains important choices made during system setup and tests. It also provides estimates of overhead introduced by the tested ideas. Such things are important in correct interpretation of test results presented in Chapter~\ref{chap:tests}, as well as in evaluation of the system.

\subsection{Choosing Game Clocks}

Since GGP agents are assumed to adapt to an arbitrary game, they are naturally assumed to adapt to an arbitrary game length (expressed as startclock and playclock). However, when doing any kind of game-specific agent adaptation, be that mining knowledge, training an evaluation function, etc., it may be important to choose right game clocks. Keeping in mind that UCT converges in time to the same result as minimax, the impact of using additional knowledge (which is usually not that accurate) might diminish in time, eventually worsening the performance.
\\

As explained in Section~\ref{sec:reasoning}, the reasoning engine plays a crucial role in the agent's computational performance. By profiling the GGP agent made for this publication with Google Performance Tools~\cite{googlePerfTools}, it turned out that reasoning alone accounts for most of the time spent by the agent during each turn in a bare UCT agent. However, performance of reasoning engines may vary greatly, even by orders of magnitude. For instance, a GTC-based agent can do as much reasoning during 1~s as, elsewhere identical, Prolog-based one in 44~s in 8puzzle.\footnote{Detailed performance comparison is shown in Subsection~\ref{subsec:gtcPerformance}.}
\\

Another factor to consider is the complexity of the GDL rule sheet. Reasoning in games, described with rule sheets large resulting in large game trees (chess, checkers, Othello),
tends to be much slower in terms of games/s or states/s, than in case of simpler ones (Blocks World, Tic-tac-toe). For instance, GTC-based agent can conduct about 3/s full random playouts of Othello, whereas for Connect Four the same agent would conduct as many as 2300/s, even though Othello playouts are approximately three times long in the number of turns. Additionally, duration of random playouts made by a human player would probably be far closer in both cases.
\\

One simple approach, that seems to be favored during various GGP competitions, is to choose game clocks similar to those which human players would use. Computations made for this publication involved thousands of regular GGP matches. The chosen clocks were fairly low, usually both startclock and playclock were set as 10~s. As stated above, clocks are very relative, especially when taking advantage of GTC. To investigate the impact of knowledge, on games with game graphs heavily explored by UCT, as well as those with graphs explored only a little, the set of games chosen for evaluation included games with efficient and inefficient reasoning engines. When developing knowledge with a particular application in mind, all said factors should be considered when choosing right clocks for simulations.

\subsection{Used Hardware}

Even though the participants of a GGP match (agents and Gamemaster) are able to play over network, and the knowledge mining involves many computers connected to a local area network, the preferred method of conducting matches is to have the agents and the Gamemaster run on one machine and communicate through the loopback interface. This approach has strong advantages:
\begin{itemize}
\item it evens out RAM usage on participating machines, since even a low number of Gamemasters might deplete the resources,
\item match fault rate is lower, because players are not susceptible to unexpected delays in communication,
\item agents running on the same processor should obtain roughly the same amount of processing power, than when run on different hardware.
\end{itemize}

Computations made for this publication where conducted in LAN consisted of 28 workstations with the following configuration:
\begin{itemize}
\item Intel Core2 Quad Q8300 @ 2.50~GHz, 4~GB RAM, Debian 7.0, gcc 4.7.1 (14 workstations),
\item Intel Core i5-2400 @ 3.10~GHz, 4~GB RAM, Debian 7.0, gcc 4.7.2 (14 workstations).
\end{itemize}

Both CPU models present on those machines were equipped with 4 cores. With one computer being the miner/server, there were 54 matches of a 2-player game possible to conduct in parallel (4 agents occupying each of the remaining 27 workstations). With 10~s clocks, the computations usually took 1-3~h/generation, depending on the game.
\\

Further improvements to the computation time might be easily achieved by rewriting the mining component in C++ and/or distributing the mining on more workstations.
\\

The current design of the agent is single-threaded. To achieve the best utilization of resources, the number of players run in parallel on one machine should match the number of core. However, operating system's scheduling policy should also be taken into account.

\subsection{Knowledge Overhead}

Evaluation of algorithms and computation methods naturally should take their implementation into account. The GGP Agent is designed to work with different reasoning engines - each supplying it's \texttt{Fact} class. To keep things simple, the knowledge module operates on facts where object constants are strings. However, GTC uses internally simple types (\texttt{int}s) instead, to keep the reasoning faster and lower memory usage. The knowledge module makes frequent conversions to coordinates, i.e. for a fact \texttt{(mark 1 2 x)} it will return a vector $\rowvec{2}{1}{2}^\top$. With GTC, it has to make frequent unnecessary intermediate conversions to strings (which is not the case with YAP Prolog, where strings are being returned by default by YAP library calls).
\\

Of course, the mentioned overhead takes up precious turn time, which could be used for more iterations of MCTS, thus artificially limiting available turn time. The real performance of knowledge-based players, with more efficient implementation, should be much higher than reported. To overcome this problem, the framework could be further modified to have the reasoning engine supply also it's own metric module.

\subsection{Fitting GTC Within Startclock}

Even though this publication assumes, that large amounts of time are available for knowledge development before the start of the  game, it is still important to evaluate application of GTC to ordinary GGP matches. In such matches, the GTC module should manage to analyze game rules, generate C++ source code, compile it to a custom reasoning machine, and have the agent load it, during the startclock phase.
\\

Table~\ref{tab:gtcTimes} summarizes resources required for GTC to complete for various GDL rulesheets. For tested games, GDL managed to finish roughly within 10~s startclock with an optimized version, and within 5~s startclock with an unoptimized one, on an ordinary machine. Both code generation and compilation were single-threaded, therefore the agent could run a few mirrors of GTC Generator (with different optimization levels set) in parallel. They would still share IO, but the tests revealed that most of the time spent on generation and compilation was spent in the user space, so it should not be an issue. Moreover, the agent could even temporarily load Prolog and begin the computations using remaining cores from the beginning, with an option to replace the reasoning engine on the fly, to achieve the best utilization of resources.

\begin{table}[h]
  \caption{Resource usage during generation and compilation of GTC source code on an Intel Core i5 2400 @ 3.10~GHz machine. The tests were carried using single-thread programs (framework's GTC Generator and gcc 4.7.2). Measured times are the total runtime times (spent in user space, system space and waiting for IO). Source code size is the total size, including the files common to all reasoning machines. The files do not undergo any compression and are in a human-readable form.}
	\centering
	\begin{tabular}{ |c||c|c|c|c| }
	\hline
	\multirow{2}{*}{Game} & \multirow{2}{*}{Generation time} & \multicolumn{2}{c|}{Compilation time} & \multirow{2}{*}{Source code size}\\ %\cline{2-6}
	& & non-optimized & -O3 enabled &\\ \hline
	\textit{8puzzle.kif} & 0.154 s & 3.139 s & 4.722 s & 144 KB \\
	%\hline
	\textit{amazing.kif} & 0.204 s & 3.243 s & 5.201 s & 180 KB \\
	%\hline
	\textit{blocks.kif} & 0.154 s & 3.087 s & 4.742 s & 128 KB \\
	%\hline
	\textit{checkers.kif} & 0.354 s & 3.876 s & 8.663 s & 368 KB \\
	%\hline
	\textit{chess.kif} & 0.554 s & 3.949 s & 8.862 s & 404 KB \\
	%\hline
	\textit{chinesecheckers6.kif} & 0.204 s & 3.387 s & 5.986 s & 224 KB \\
	%\hline
	\textit{connectfour.kif} & 0.154 s & 3.194 s & 5.109 s & 164 KB \\
	%\hline
	\textit{crisscross.kif} & 0.204 s & 3.294 s & 5.339 s & 200 KB \\
	%\hline
	\textit{lightsout.kif} & 0.155 s & 3.123 s & 4.611 s & 132 KB \\
	%\hline
	\textit{nim.kif} & 0.104 s & 3.101 s & 4.762 s & 124 KB \\
	%\hline
	\textit{othello.kif} & 0.204 s & 3.364 s & 5.694 s & 212 KB \\
	%\hline
	\textit{pancakes.kif} & 0.154 s & 3.152 s & 4.771 s & 144 KB \\
	%\hline
	\textit{pawntoqueen.kif} & 0.204 s & 3.301 s & 5.388 s & 188 KB \\
	%\hline
	\textit{peg.kif} & 0.154 s & 3.207 s & 5.047 s & 160 KB \\
	%\hline
	\textit{sum15.kif} & 0.154 s & 3.073 s & 4.724 s & 128 KB \\
	%\hline
	\textit{tictactoe.kif} & 0.154 s & 3.088 s & 4.838 s & 136 KB \\
	\hline
	\end{tabular}

  \label{tab:gtcTimes}
\end{table}

\section{UCT Architectural Issues}
\label{sec:agentArchitecture}

This section presents some technical aspects of a typical UCT-based General Game Playing agent, in respect to the agent implemented in GGP Spatium. As such agents have dominated the AAAI Competition\footnote{In the AAAI Competitions held in years 2007-2012, with non-simulation-based agent being among the finalists only once (Fluxplayer, 2007).}, UCT might be considered a de facto standard for constructing such systems.

\subsection{The $C$ Constant}

The first obvious adjustment is a proper $C$ value (often denoted also as $k$).
Recalling the original UCT formula~\ref{eq:ucb}, the value of a move is made of two components: it's winning average and the UCB bonus. The longer the move is omitted, the higher the bonus, so it keeps the agent from forgetting about unpromising but not thoroughly investigated moves.
\\

The bonus is multiplied by the $C$ factor which, as many note~\cite{mogoFirstGoUct06}, serves as a exploration/exploitation parameter. Higher values increase the bonus thus emphasizing exploration, whereas low values favor exploitation.~\cite{CadiaPlayer09, holt08} used an empirically set constant $C=40$, 2007 version of Ary used $C=50$~\cite{Ary08}. The value was later changed to $C=40$ for multiplayer games, and $C=100$ for single player~\cite{Ary10} (obviously to favor exploration in search problems). Those values were also determined experimentally. Also, the authors report that attempts to set the constant dynamically after the rules analysis did not bring good results.
\\

\cite{tron10, mctsPacMan11} used $C=\sqrt{2}$, which makes UCT use the exact UCB1 formula. \cite{holt08} notes that values close either 0 or 100 gave significantly worse results than a value in between.
\\

Finally,~\cite{MCAmazons07, amazonsDiscover08} point out that, when applying UCT/MCTS to the Game of Amazons (or any particular game), the constant should be accordant with the problem; in one of the studies, empirically chosen $C=0.35$ was used.
\\

The agent of GGP Spatium used a fixed value $C=40$, based on the claim that changing the value through learning does not bring satisfactory results. Chapter~\ref{chap:tests} unveils, however, an interesting shift in number of visited nodes under the influence of game knowledge. It might indicate that shorter simulation episodes take place, because of frequently choosing same paths in the selection phase, and the agent focuses more on exploitation.

% It turns out, that far less nodes are being visited in the simulation phase than in the selection phase, when compared to the baseline UCT agent. The tree expanding policy remains the same in both cases; thus, the agent equipped with knowledge 
\subsection{Internal Policies}

\paragraph{Selection phase}

An extended policy, UCB1-TUNED~\cite{multiarmed02} is reported to perform even better than UCB1 policy~\cite{InriaMCTSPatternGo06}. The upper confidence bound for UCB1-TUNED is
 
\[ \sqrt{\frac{\ln n}{n_j}\min\left\{\frac{1}{4},V_j(n_j)\right\}} .\]

Many other policies have been developed over the years. UCB2~\cite{LaiRobbins85} for instance has a better confidence bound, but is more complex. Other widely-known policies include $\epsilon_n-GREEDY$ and UCB1-NORMAL.
\\

Also, the first policy with an upper confidence value, developed by Lai and Robbins in 1985~\cite{LaiRobbins85} for the multiarmed bandit problem, was better than UCB in terms of the expected bound, but was proved only for a certain class of distributions\footnote{Including popular ones like the Gaussian Distribution or the Poisson Distribution.}. In~\cite{policyLearning12} an interesting approach of learning the best policy for a given problem was examined.

\paragraph{Tree node expansion}

Two-phase tree searching in MCTS, consisting of the selection phase and the simulation phase, has the advantage of memory conservation. Only the first phase requires storing information about the states, such as detailed descriptions, hash values, visit counters, average payouts, etc. It is, however, unclear how to proceed with extending game tree with new nodes. Ideally, stored nodes should be the ones visited the most. Coulom~\cite{Coulom06efficientselectivity} mentions a best-first tree growing policy: % taken word-by -word
\begin{itemize}
\item Start with a root only
\item Whenever a random simulation goes through a node, which was already visited once, create a new node to the next move in that simulation
\end{itemize}
Using the policy, the search tree becomes unbalanced (in terms of branches' heights), therefore it adapts to more promising states
\\

Some of the leading agents like Ary~\cite{Ary10} and CadiaPlayer~\cite{CadiaPlayer09} use a slightly simplified version of that policy, where a node is being added once for each playout, and it is the first one encountered on a search path, which is not already stored.
\\

It gets complicated for agents with highly efficient reasoning engines, hard-coded or at least not Prolog-based, which are not doing any real reasoning apart from calculating states and moves. Since they rapidly switch game states, the node-per-playout policy, especially for heavily branched games, might lead to quick exhaustion of available memory. For instance, the Game of Amazons agent~\cite{amazonsDiscover08} mentioned earlier, has a policy of adding a node after as many as 40 visits to a parent node. Interestingly, the threshold was not set to just save the memory. As noted in the paper, setting the threshold to 20 yielded slightly worse results, while decreasing it further to 5 lowered the winning percentage of the agent (of plays versus the test agent) from 80\% to 65\%.
\\

The agent implemented for this publication featured a fixed policy of one node for each playout. Resulting transposition tables were rather small (in the number of states), yet they have been prepared to selectively adapt to size restrictions. Section~\ref{sec:testsTT} discusses sizes of transposition tables obtained during tests.

\paragraph{Opponent modeling}

At default, every opponent is assumed to play with the same UCT strategy. In the minimax heuristic, the algorithm has to alternate between the players on consecutive plies of the tree. In UCT, every opponent can be modeled with the same strategy. For each possible move a state should thus hold not the average payout, but the average payouts for all players. It has a reflection in game knowledge, since features are mined and stored for each player separately.

\chapter{Validation}
\label{chap:tests}

This chapter contains detailed description of tests made to evaluate the ideas proposed in Chapter~\ref{chap:keyIdeas}. They have been carried out, using GGP Spatium framework, in a local area network (the testing environment is discussed in Chapter~\ref{chap:testingPlatform}).

\section{Tested Games}
\label{sec:testedGames}

All rule sheets mentioned in this publication (with the exception of rule sheets modified with the Euclidean extension) come from Dresden GGP Server~\cite{dresdenGGPServer} game database. Only a few games were chosen for the most time-consuming tests (preparing game knowledge through evolution): \textit{amazons.kif}, \textit{checkers.kif}, \textit{connectfour.kif}, \textit{cubicup.kif}, \textit{othello.kif}. They have been enhanced with the proposed Euclidean GDL extension. The extension, however, does not guarantee immediate compatibility with a rule sheet. It took minor changes to make three of them conform:
\begin{description}
\item[amazons.kif] - A turn consists of moving a piece and shooting an arrow. Turns have been split in half, letting the player move and shoot in separate, though subsequent turns. The opponent is forced to wait (by playing \texttt{noop}) for two subsequent turns. This change altered the game automaton, while preserving semantics. It also sped up the reasoning and considerably lowered branching factor from around 500 to $\sqrt{500}$, at the cost of doubling the length of the game.
\item[checkers.kif] - The rule sheet required specifying a piece location and destination to make a move. Moreover, for double and triple jumps, intermediate fields also had to be specified, i.e. \texttt{(doublejump bp 3 2 3 3 3 4)} moves a piece from field $(3,2)$ to field $(3,4)$. Move functors have been made into arguments, and dummy variables have been added, so all move sentences would have equal number of arguments and a common functor \texttt{mark}. The previous example would become \texttt{(mark doublejump bp dummy dummy 3 2 3 3 3 4)}.
\item[connectfour.kif] - It suffices to chose a column to make a move. The second dimension, as well as piece type, was implicitly added to the move as well. \texttt{(drop ?x)} became \texttt{(play ?x ?y ?piece)}.
\end{description}
To avoid confusion in the remaining sections, names of the rule sheets modified to meet the needs of the Euclidean extensions end with .ext.kif, e.g. \textit{tictactoe.ext.kif}.
\\

\begin{table}
  \caption{Games for which knowledge files have been developed with Knowledge Miner}
	\centering
	\begin{tabular}{ |c||c|c|c|c|c| }
	\hline
	\multirow{2}{*}{Game} & Branching & Average & Random & \multirow{2}{*}{Board} & \# of \\
	& factor & length  & games/s & & dimensions \\
	\hline
	\textit{amazons.ext.kif}  		& 22 	& 131	& 51.8	& square $10\times 10$ & 2-d \\
	\textit{checkers.ext.kif} 		& 10 	& 99	& 44.2		& square $8\times 8$ & 2-d \\
	\textit{cubicup.ext.kif} 		& 4 	& 56	& 278.23	& tetrahedron (edge 6) & 3-d \\
	\textit{connectfour.ext.kif} 	& 7 	& 22	& 2330.49	& square $8\times 6$ & 2-d \\
	\textit{othello.ext.kif} 		& 7 	& 61	& 2.99		& square $8\times 8$ & 2-d \\
	\hline
	\end{tabular}

  \label{tab:testedGames}
\end{table}

Table~\ref{tab:testedGames} summarizes rule files chosen for knowledge evolution. Descriptions and complete rule sheets of less-known tested games are presented in Appendix~\ref{apdx:rulesheets}.

\section{Compiled Code}

\subsection{mGDL Conformance}
\label{subsec:mgdlTest}

Algorithm~\ref{alg:mgdl-check} has been used to resolve conformance with mGDL for a particular rule sheet. The algorithm tries to unify literals found in rules' tails with other rules' heads, first checking functors and arities. Unificators are being examined; if all variables unify with constants or other variables, the rule sheet is mGDL compatible. If there are exceptions, the result of the algorithm is inconclusive, for there is no guarantee that such an unification would ever occur during the real reasoning.
\\

\begin{algorithm}[h]
 \SetAlgoLined
 \KwData{GDL rule sheet}
 \KwResult{\texttt{true} if conforming to mGDL or \texttt{inconclusive} when uncertain}
 \BlankLine
	$R \gets$ set of rules\;
	\For{rule $r \in R'=R\setminus\left\{terminal, goal, legal, next, \ldots\right\}$}{
		\For{literal $l \in tail(r)$}{
			\For{rule $r_2 \in R$}{
				\lIf{$head(r_2)$ unifies with $l\; \wedge \;\exists a\in args(l)$ s.t. $a$ unifies with a complex expr}	{			
					\Return{(inconclusive)}
				}
			}
		}
	}
	\Return{true}
\caption{Basic mGDL conformance check}
\label{alg:mgdl-check}
\end{algorithm}

A set of 270 GDL games from Dresden GGP Server~\cite{dresdenGGPServer} has been checked with the aforementioned algorithm for mGDL conformance. Table~\ref{tab:mgdlCompat} presents the results. Rule sheets, which do not immediately conform to mGDL, are easy to alter; in all cases, it would take only minor changes, not significantly extending nor complicating the rule sheets.
\\

\begin{table}
  \caption{Conformance of GDL v1 rule sheets with mGDL, obtained from Dresden GGP Server~\cite{dresdenGGPServer}.}
	\centering
	\begin{tabular}{c|c|c|}
	\cline{2-3}
	& \textbf{rule sheets} & rule sheets \\
	& \textbf{conforming to mGDL} & not conforming to mGDL \\ 
	\hline 
	\multicolumn{1}{ |c| }{number} & \textbf{253} & 17 \\ 
	\hline 
	\multicolumn{1}{ |c| }{percentage} & \textbf{94\%} & 6\%\\
	\hline 
	\end{tabular}

  \label{tab:mgdlCompat}
\end{table}

\subsection{GTC Performance}
\label{subsec:gtcPerformance}

The raw performance of reasoning engines has been measured with GTC Game Runner, tool made solely for benchmarking purposes. It measures the average performance of a particular reasoning engine during random plays. With this tool, some of GDL v1 games have been tested. It should be noted that GTC reasoning engines failed to build for some games, because of minor incapabilities of used implementation of the GTC translation scheme. Table~\ref{tab:gtcperformance} summarizes achieved results.
\\

\begin{table}
  \caption{Average GTC performance measured on an AMD Athlon 64 X2 4000+ @ 2.1~GHz machine, compiled with gcc 4.6.3 (32-bit). Tests were carried using GTC Game Runner, which minimizes the overheads. Measured times are the average times from a 10~s run. Prolog tests were carried with YAP Prolog 6.2.2 and a C++ module loading YAP as a shared library.}
    \centering
	\begin{tabular}{ |c||c||c|>{\columncolor[gray]{.9}}c||c|>{\columncolor[gray]{.9}}c| }
	\hline
	\multirow{2}{*}{Game} & YAP Prolog & \multicolumn{2}{c||}{GTC} & \multicolumn{2}{c|}{GTC2}\\ %\cline{2-6}
	 & games/s & games/s & \% of YAP & games/s & \% of YAP\\ \hline
	\textit{8puzzle.kif} & 132.69 & 531.04 & 400 & 5849.11 & 4408 \\
	%\hline
	\textit{amazing.kif} & 295.34 & 374.16 & 127 & 6054.98 & 2050 \\
	%\hline
	\textit{blocks.kif} & 3830.79 & 233648.83 & 6099 & 283470.64 & 7400 \\
	%\hline
	\textit{checkers.kif} & 5.51 & 8.82 & 160 & 44.20 & 802 \\
	%\hline
	\textit{chess.kif} & 1.13 & 1.51 & 134 & 2.46 & 218 \\
	%\hline
	\textit{chinesecheckers6.kif} & 46.90 & 54.50 & 116 & 874.65 & 1865 \\
	%\hline
	\textit{connectfour.kif} & 236.54 & 785.92 & 332 & 2330.49 & 985 \\
	%\hline
	\textit{crisscross.kif} & 179.82 & 148.24 & 82 & 1951.87 & 1085 \\
	%\hline
	\textit{lightsout.kif} & 128.82 & 535.44 & 416 & 4060.41 & 3152 \\
	%\hline
	\textit{nim.kif} & 762.65 & 867.13 & 114 & 19230.91 & 2522 \\
	%\hline
	\textit{othello.kif} & 2.34 & 0.80 & 34 & 2.99 & 128 \\
	%\hline
	\textit{pancakes.kif} & 485.51 & 393.82 & 81 & 33954.03 & 6993 \\
	%\hline
	\textit{pawntoqueen.kif} & 35.30 & 37.77 & 107 & 119.67 & 339 \\
	%\hline
	\textit{peg.kif} & 75.07 & 77.52 & 103 & 254.33 & 339 \\
	%\hline
	\textit{sum15.kif} & 1060.32 & 5721.91 & 540 & 14856.77 & 1401 \\
	%\hline
	\textit{tictactoe.kif} & 906.37 & 4532.34 & 500 & 34390.02 & 3794 \\
	\hline
	\end{tabular}

  \label{tab:gtcperformance}
\end{table}

GTC v1 yielded satisfactory results, but only for shallow and wide reasoning trees. Google CPU Profiler\footnote{Part of Google Performance Tools~\cite{googlePerfTools}.} was used to determine the cause of this behavior. As it turned out, each time a fact was processed, a join operation took place (exemplary join might be found in Subsection~\ref{subsec:database}), and the overall cost of all the join operations was the biggest bottleneck of all. The approach turned successful only for small games, where variable tables were short. More efficient data structures are required for more complex games with GTC v1.

\section{Game Knowledge Development}

This section presents tests concerned with spatial features and the evolutionary algorithm. Impact of features on the agent's computational performance and quality of plays is established, along with purposefulness of the evolution. At the end, crucial evolved parameters are summed up.

\subsection{Evolution Parameters}

Because the genetic algorithm utilized by the knowledge miner is based on SGA, parameters with regular to SGA values have been employed:
\begin{multicols}{2}
\begin{itemize}
    \item population size: 24
    \item crossover rate: 70-80\%
    \item elitism: 15\%
    \item new random individuals: 10\%
    \item mutation rate: 0.01-0.02
    \item ranked selection
    \item matches with no knowledge: 300
    \item matches per individual during tournaments: 50
    \item number of generations: 5-20
\end{itemize}
\end{multicols}
\noindent
Feature weights were assigned based on their $\phi$ coefficients.

\subsection{Objective Function} % stability and rank convergence}

The point of mining the knowledge is to develop the knowledge file, which would improve the agent's play. It is not clear how to artificially measure the agent's skill; the chosen approach was to make a bare UCT agent (that is a one having no knowledge) a baseline agent, with whom scoring matches would be played. Then, after a certain number of plays, a conclusion about the agent's performance might have been drawn.
The objective function has been thus formulated, as an average win percentage of the agent, in matches versus a bare UCT agent. Of course, the accuracy depends on the number of conducted matches.
\\

Reliable scoring of the agents is very resource-consuming. Matches usually end with a binary result; thus, the score might be modeled with the binomial distribution. It takes nearly 400 matches, to achieve a margin of error of circa 5\% when scoring one single agent.\footnote{Assuming 95\% confidence interval and estimating the interval with a normal distribution.} Since the impact of an agent on the next generation diminished with it's score, it was crucial to score reliably the fittest individuals.
\\

To conserve resources, the agents were scored in rounds. During each round, $m$ matches were conducted between each individual and the bare UCT agent. Only the best performing half of the agents qualified for the next round. The tournament ended, when there was one individual left. Rounds allowed to conduct over half of the matches less, than when scoring all the agents with the same number of matches. Figure~\ref{fig:tournamentConvergence} shows agent's fitness during an exemplary tournament.
\\

\begin{figure}
  \centering
  \includegraphics[width=1.0\textwidth]{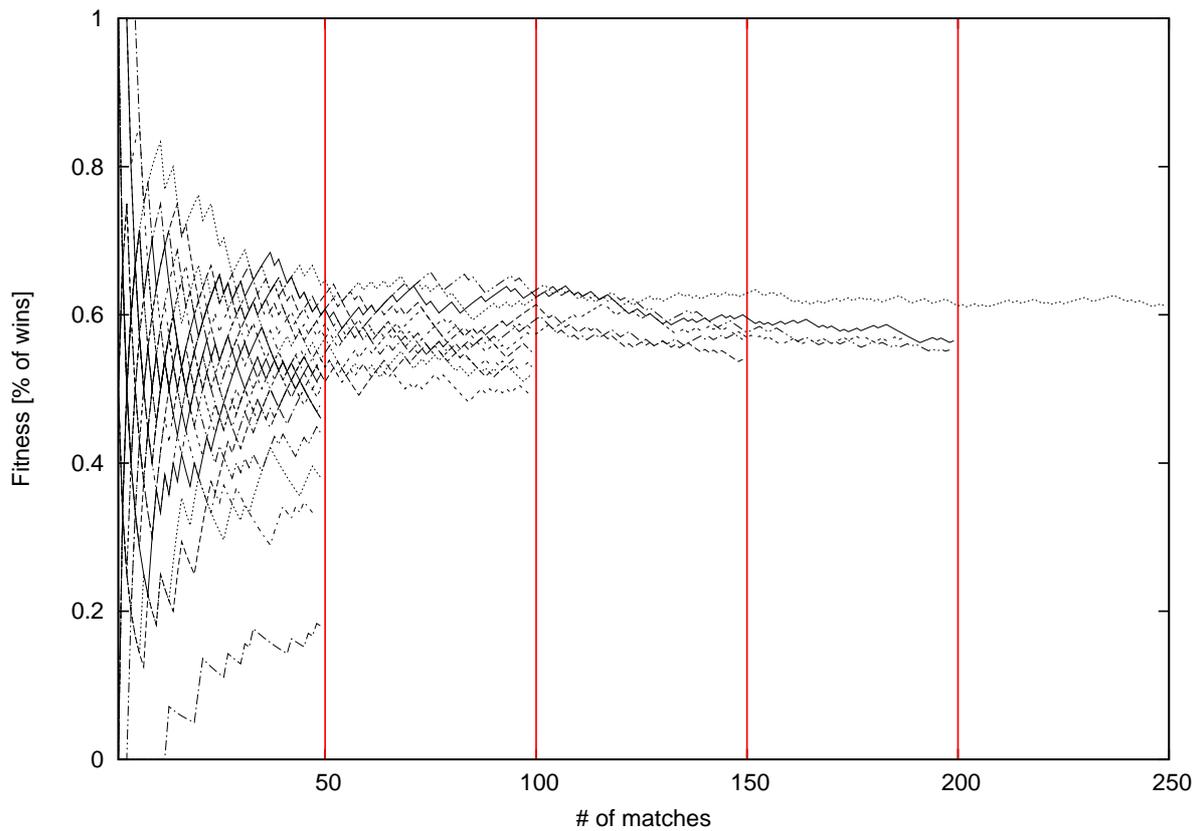}
  \caption{Convergence of agents' fitness throughout the tournament in \textit{connectfour.ext.kif}. Each polyline represents agent's fitness, expressed by win percentage against baseline UCT agent. The weakest ones are being discarded with each round. Vertical lines separate subsequent rounds of the tournament.}
  \label{fig:tournamentConvergence}
\end{figure}

Another investigated way of scoring the agents was to have the agents play between themselves. It also had a similar benefit of cutting the number of matches in half (for 2-player games, since every match gave feedback to both of the players). Early test, however, gave significantly worse results in this approach. Game records were of lower quality, but most importantly, the practical confidence interval was much wider. Although there was a strong correlation between the results of the tournaments, and results of plays versus a bare UCT agent, the correlation was not perfect and added to the confidence intervals. The approach worked well on early generations, where individuals differed between each other. In the later ones, the error was to great to reliably score the agents, with only minor differences in scores.

\subsection{Particular Evolution Cases}

To evaluate the system of evolutionary feature mining and utilization in games as a whole, complete evolution episodes were carried for games listed in Section~\ref{sec:testedGames}. Early tests reported poor performance; CPU profiling revealed a major overhead on memory management in the knowledge module (compared with regular, GTC-based agent). The overhead can be mitigated in the GTC version of the reasoner, by integrating the knowledge module, so that during each random playout no extra state objects would be created and passed from the reasoner to the agent. To manage with the overhead, a simple handicap has been introduced: each bare-UCT player had half the time available to the knowledge-based one. The handicap has been chosen based on rough overhead estimates. The following sections contain information about the obtained results.

\subsubsection{The Game of Amazons}

The Game of Amazons becomes very costly to test in the modified version, were turns are split in half and average random game has length of 120 turns. Initially, only 10 generations were carried; early generations did not seem to bring any significant gain. Figure~\ref{fig:amazonsQ} presents best agent's fitnesses. It is worth noticing, that the final best individual did not feature progressive widening; after all, branching factor of the tested version has been lowered.
\\

\begin{figure}
  \centering
  \includegraphics[width=0.7\textwidth]{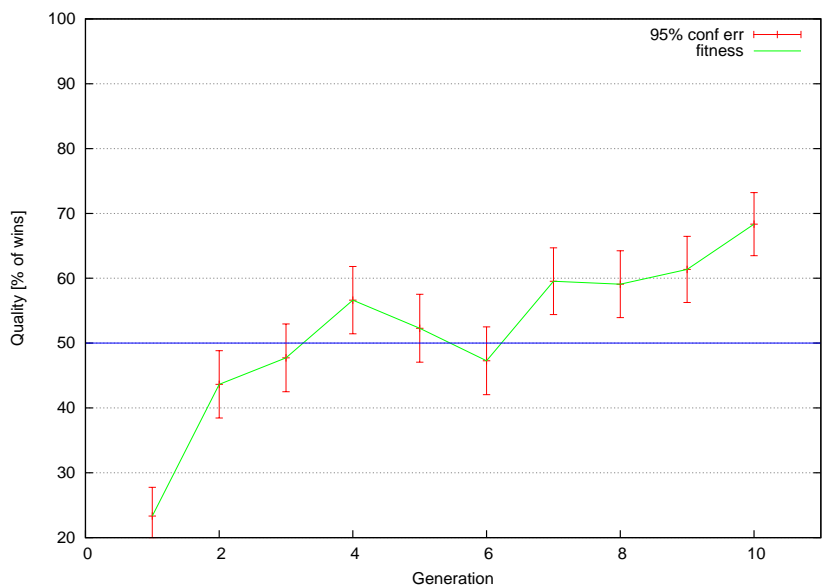}
  \caption{Fitness of best individual during game knowledge evolution - \textit{amazons.kif}}
  \label{fig:amazonsQ}
\end{figure}

\subsubsection{Checkers}
As it is a case with rule sheets having highly inefficient corresponding reasoning engines, even though GTC has improved the performance, UCT tree search was shallow. Under this conditions, spatial features gave immediate gain, which conversely has not been further improved by the evolution (Figure~\ref{fig:checkersQ}).

\begin{figure}
  \centering
  \includegraphics[width=0.8\textwidth]{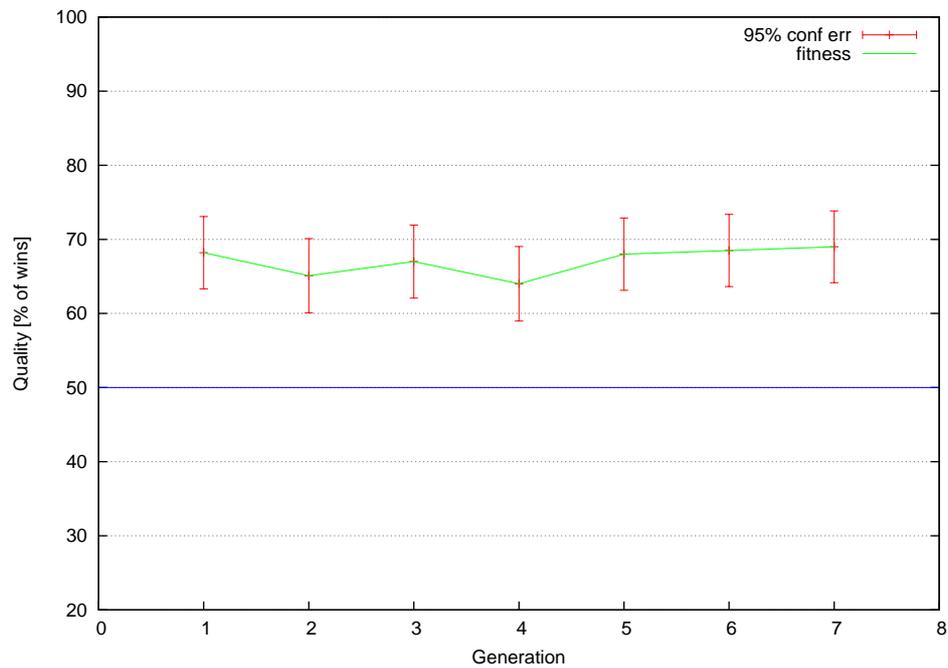}
  \caption{Fitness of best individual during game knowledge evolution - \textit{checkers.kif}}
  \label{fig:checkersQ}
\end{figure}

\subsubsection{Connect Four}

Short match lengths in Connect Four encouraged longer evolution. However, it failed to bring any significant gain within the set handicap (Figure~\ref{fig:connect4Q}). While the opponent had half the time as startclock and playclock, final transposition tables had on average  71\% states less in the individual, indicating that the knowledge overhead caused more serious slowdowns than expected.
\\
\begin{figure}
  \centering
  \includegraphics[width=0.8\textwidth]{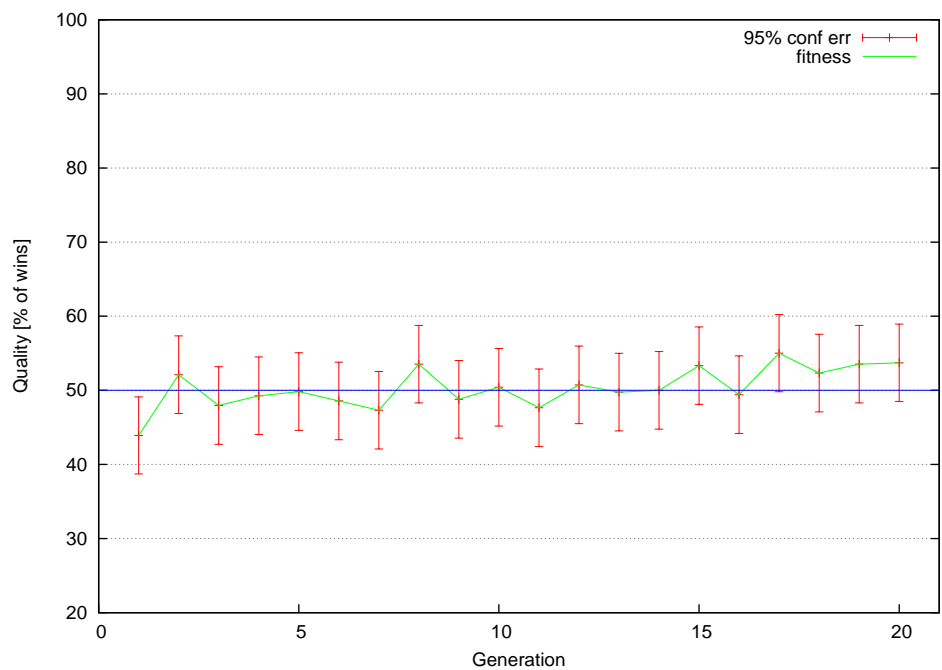}
  \caption{Fitness of best individual during game knowledge evolution - \textit{connectfour.kif}}
  \label{fig:connect4Q}
\end{figure}

\subsubsection{Cubicup}

Spatial features gave an immediate rise in quality of plays of around 5-10\% in the first generation, which was evolved to around 15\% (Figure~\ref{fig:cubicupQ}). Transposition tables of the best knowledge-equipped individuals had about 40-70\% states of the baseline agent's TTs, with 44\% in the final individual, what speaks in favor of the 0.5 handicap.

\begin{figure}[htbp]
  \centering
  \includegraphics[width=0.8\textwidth]{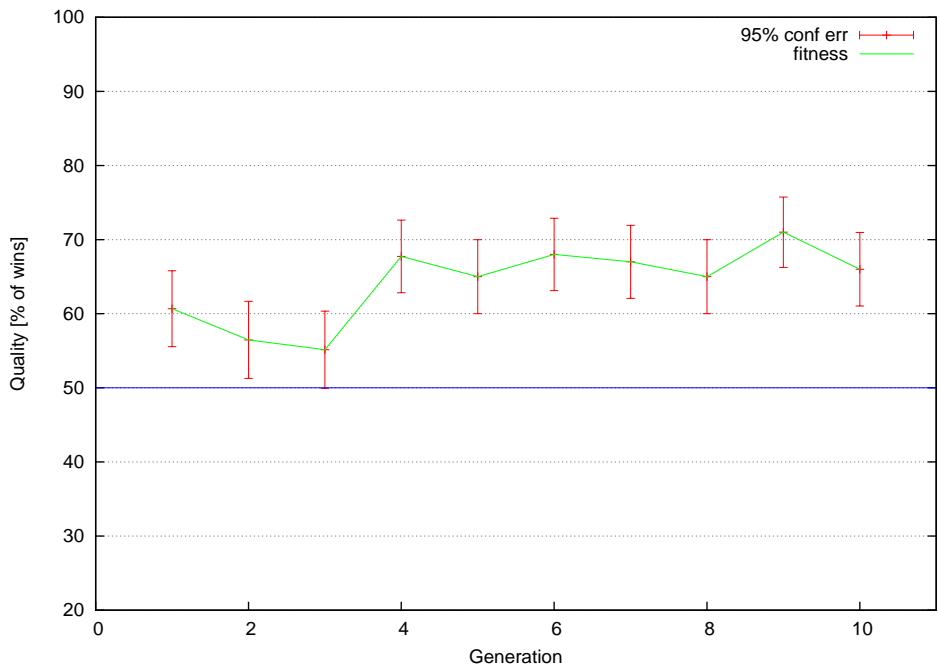}
  \caption{Fitness of best individual during game knowledge evolution - \textit{cubicup.kif}}
  \label{fig:cubicupQ}
\end{figure}

% To further investigate the cause of aforementioned performance gain, average number of features in a generation was measured. Results are presented in Figure~\ref{fig:cubicupFeatureNum}. Knowledge files become saturated with features around the 10th generation. 

\subsubsection{Othello}

With games shorter than 64 turns, Othello has been tested with slightly larger game clocks. Nonetheless, unreasonably slow reasoning engine resulted in shallow tree search. Feature system seemed to just make up for the overhead - average transposition table sizes were around 45\% smaller in knowledge agents than in the baseline agent. Curiously, among mined features were some features common among Othello players and programs, like those promoting playing in the corners or along the edges. Those few features alone should be responsible for most gains from the knowledge. It might be the case, that other features, which seem relevant in the mining phase, actually worsened agent's plays.

\begin{figure}[htbp]
  \centering
  \includegraphics[width=0.8\textwidth]{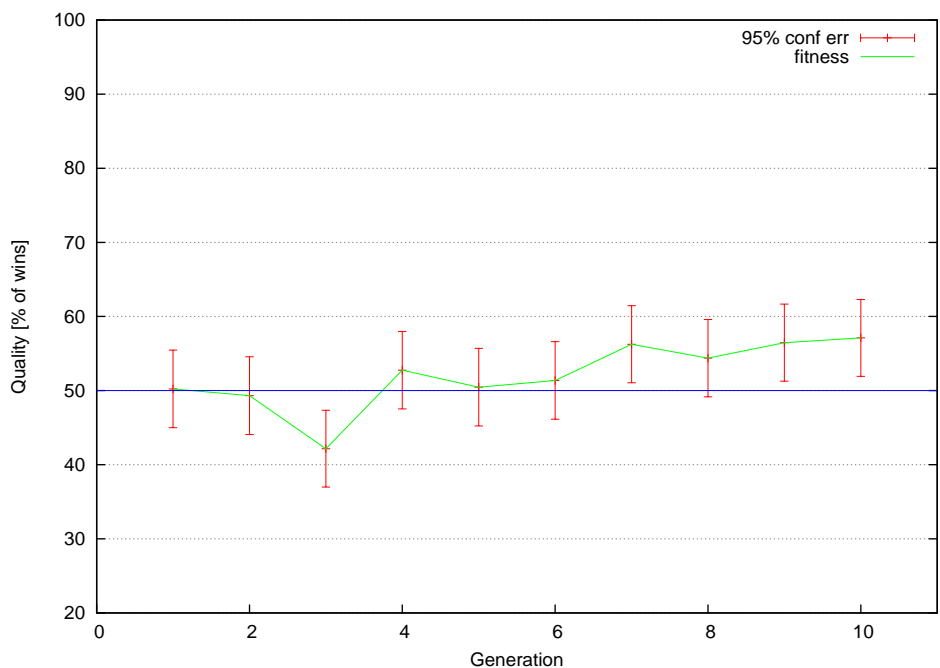}
  \caption{Fitness of best individual during game knowledge evolution - \textit{othello.kif}}
  \label{fig:othelloQ}
\end{figure}

\subsection{Feature Utilization}

With no doubt, well evolved feature set gives an advantage to the agent. The purpose of this section is to determine how the features are being used.
\\

As stated in Subsection~\ref{subsec:features}, the features fall into 7 categories. Distribution of those categories among the best players is shown in Figure~\ref{fig:featureDistrib}. All feature classes occur in the fittest individuals. However, currently there is no method to measure the real per-feature impact on playouts (how frequently a feature is matched, what is the average payout of the most frequent and infrequent matches, etc.). With such data at hand, it would be feasible to fine-tune feature list by tweaking their weights.
\\

\begin{figure}[htbp]
  \centering
  \includegraphics[width=0.8\textwidth]{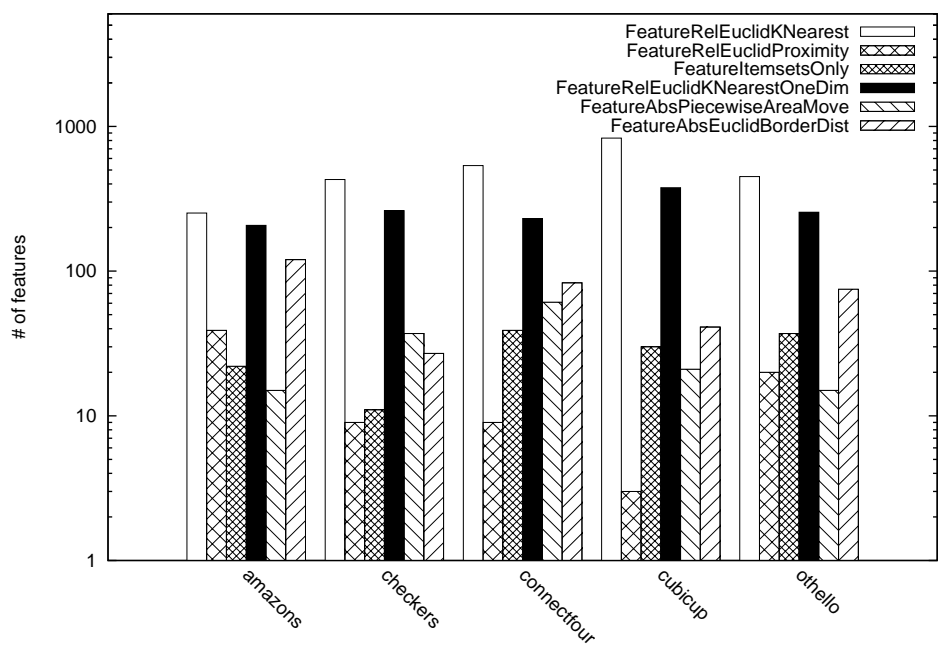}
  \caption{Distribution of feature classes in different games. Features come from the best individual in evolution. The Y-axis is in logarithmic scale.}
  \label{fig:featureDistrib}
\end{figure}

Number of features also plays an important role; recalling poor performance of the knowledge module, higher counts of features might not report the expected fitness. The maximum number of features, which is susceptible to evolution, ranges 0-500 for all the players. For instance, for a 2-player game, the total number ranges 0-2000, as each player has both good and bad feature lists. Figure~\ref{fig:featureNumConnect4} shows the evolution of an average number of features in population in Connect Four. The number grows in time, until the limit is reached. The maximum number of features increases after the knowledge files have been saturated. It may be expected that this trend would continue with next generations. % In both cases, the average feature capacity, was not affected by the evolution up to the point of saturation; it also a premise, that no premature convergence occurs.

\begin{figure}[htbp]
  \centering
  \includegraphics[scale=1.0]{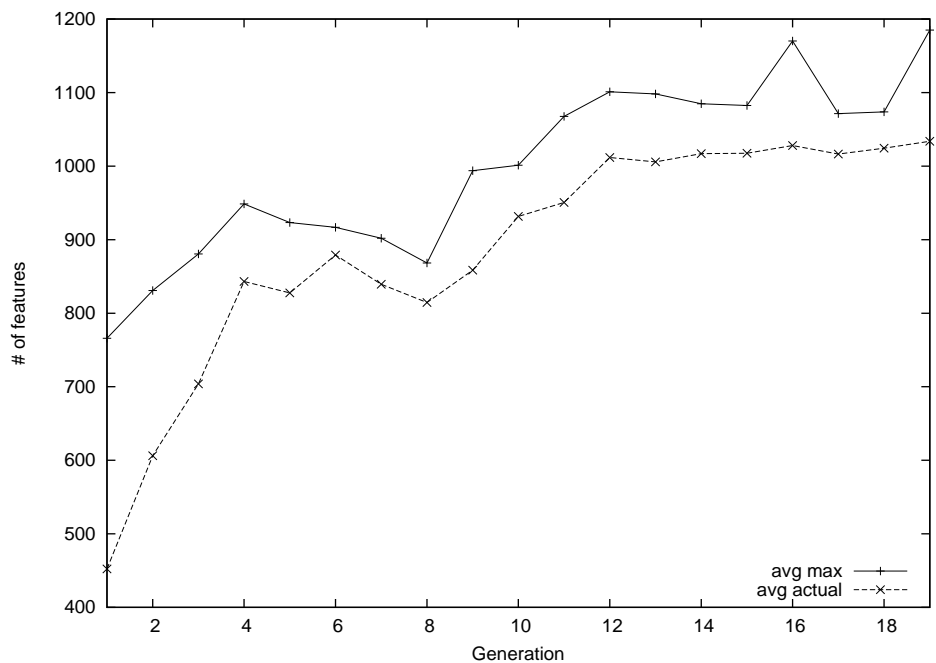}
  \caption{Average maximal and actual number of population's features in \textit{connectfour.kif}. After achieving saturation, knowledge size increases evolutionary, to permit new features.}
  \label{fig:featureNumConnect4}
\end{figure}

\subsection{Knowledge Summary}

The impact of the knowledge turned out hard to measure; frequent passing game states to the knowledge module resulted in memory management overhead of around 45\% to 71\%, seriously crippling the agent's performance (Table~\ref{tab:testedGamesResults}). Nonetheless, after applying a handicap to the baseline agent, the evolution took off. Spatial features, considering performance issues, brought competitive gains, especially in harder games. For instance in checkers, poor performance of the reasoning engine resulted in extremely shallow tree search. Spatial features improved the play substantially, again considering the applied handicap.
\\

However, preliminary tests with larger game clocks brought significantly worse results. This is expected, since such a small number of game features, limits their expressive power, therefore skewing payout distributions obtained through random playouts. Perhaps it would be a good idea to limit the impact of knowledge in time, not only in the selection phase, but in the simulation phase as well. Another obvious improvement would be extending the system with new feature classes, and proper implementation adding knowledge support on the level of the reasoning machine.
\\

\begin{table}
  \caption{Agent performance with evolved knowledge. Win percentage is an average value from matches played against the baseline, bare-UCT agent. Confidence interval were constructed at a confidence level of 95\%. Selection and simulation node percentage is the percentage of nodes visited during those matches by the baseline agent.}
	\centering
	\begin{tabular}{ |c||c|c|c|c| }
	\hline
	\multirow{2}{*}{Game} & Startclock/ & Win		 & UCT & Tree\\
						& playclock		& percentage & constant & expands \\
	\hline
	\textit{amazons}  		& 10s/10s	& $68.3\pm4.87\%$	& 56.7	& 14/10 playouts\\
	\textit{checkers} 		& 2s/2s 	& $68.8\pm4.85\%$	& 36.4	& 12/10 playouts\\
	\textit{cubicup} 		& 5s/5s		& $66.1\pm4.96\%$	& 15.4	& 5/10 playouts\\
	\textit{connectfour} 	& 5s/5s		& $53.7\pm5.22\%$	& 33.0	& 8/10 playouts\\
	\textit{othello} 		& 10s/10s	& $57.1\pm5.19\%$	& 38.4	& 9/10 playouts\\
	\hline
	\end{tabular}
  \label{tab:testedGamesResults}
\end{table}

\section{Transposition Tables}
\label{sec:testsTT}

The purpose of the following section was to determine usefulness of linked transposition tables. To achieve the goal, matches of bare UCT players were carried for different games with various transposition table sizes. Both types of agents, that is with linked and ordinary transposition tables played against each other. Obtained results have been presented in Table~\ref{tab:ttPerformance}.
\\
\begin{table}
  \caption{Performance of bare UCT agents utilizing linked transposition tables against agents with randomly-emptied transposition tables, under size restrictions}
	\centering
	\begin{tabular}{ |c||c|c|c| }
	\hline
	Game & required TT size & restricted TT size (\%) & Win percentage\\
	\hline
	\textit{amazons.kif}  		& 23500	& 1\% 	& $51.5\pm6.92\%$ \\
	\textit{checkers.kif} 		& 24000	& - 	& - \\
	\textit{cubicup.kif} 		& 16000	& 1\% 	& $48.7\pm6.93\%$ \\
	\textit{connectfour.kif} 	& 11500	& 0.5\% & $57.1\pm6.86\%$ \\
	\textit{othello.kif} 		& 30000	& 0.3\% & $57.9\pm6.84\%$ \\
	\hline
	\end{tabular}
  \label{tab:ttPerformance}
\end{table}

The bound on transposition table size has to be much lower than the maximal obtained size during unbounded simulations, to see any difference between the two types of TTs. It should be even lower, than an average number of states expanded during a single turn. It is costly to test reliably with hundreds of matches, games with long clocks that could better illustrate the effect. Thus, small game clocks were tested. Linked transposition tables gave a small improvement over randomly-emptied ones. Google CPU Profiler~\cite{googlePerfTools} reported negligible cost of handling the additional linked list.
\\

The difference in performance between the two types of TTs might be expected to increase with further limiting the size of tables or increasing turn time. However, as long as the reasoning engine is weak, memory exhaustion should not be an issue. Furthermore, MCTS suited to the Game of Amazons in~\cite{amazonsDiscover08} achieved better results with expansions policy as high as expanding by one node every 40 visits, so even with efficient reasoning, bounding the memory might not be necessary. On the other hand, it is easy to imagine a chess player, equipped with a custom engine, that could benefit from linked TTs a lot.
\\

\section{Validation Summary}

The approach proposed in this thesis proved to be valuable. Generated C++ reasoning machines offered significant speed ups of 2300\% on average over YAP Prolog, while retaining compatibility and fairly low resource usage. With improved pace of matches, it became more feasible to conduct long computations. The evolutionary algorithm, along with simple mining methods, working on a very basic set of spatial board features, were competitive with the standard UCT algorithm and noticeably improved the quality of plays. Linked transposition tables, while introducing a negligible overhead, brought performance gains only with restrictive limits on their sizes. However, they still might be useful for agents playing longer matches.

\chapter{Summary}
\label{chap:summary}

\section{Conclusions}

This publication pursues a goal of designing and evaluating an efficient General Game Playing agent, at the same time relaxing some limitations imposed by GGP specification. Perhaps, the crucial change is carrying time-consuming computations, like mining game knowledge, before the actual match. With this approach in mind, GGP can find it's application in developing efficient problem solvers, rather than versatile but weak players.
\\

Chapter~\ref{chap:sota} presents the latests and most important research in the field of GGP/MCTS. Since MCTS is strongly connected to GGP as the main method of choice, the research made for particular games is highly relevant to the topic.

Chapter~\ref{chap:problem} contains formulation of the problem and initial discussion of it's purposefulness and methods.

Chapter~\ref{chap:keyIdeas} introduces the new approach, formulated in the thesis as a solution for the problem. It consists of a revised, GDL to C++ translation scheme, proposition of spatial feature-based general game knowledge and an evolutionary algorithm for mining features and tunning knowledge parameters.

Chapter~\ref{chap:testingPlatform} starts with an introduction of GGP Spatium - a versatile framework created to evaluate the approach presented in this publication. It consists of many components, parts of which can be easily modified, like playing algorithms and reasoning machines swapped in the agent.

Chapter~\ref{chap:tests} finally presents the results of tests constructed in order to verify the approach. GTC translation scheme proves valuable and shows a substantial improvement over the original approach. Additionally, spatial board knowledge and the evolutionary algorithm, while simple, improve the agent's performance.
\\

All objectives of this thesis are met. The developed agent's leanness manifests itself through efficient reasoning with GTC, possible memory conservation with linked transposition tables, and fewer dependencies, by making Prolog libraries optional. It also brings a shift in the philosophy, by separating learning from playing.

The GDL translation scheme starts with an introduction of mGDL - a simplified version of GDL, yet much easier to handle during the translation process. Though mGDL is closer to Datalog than GDL, high compatibility with existing GDL rule sheets has been retained; 97\% of over 250 tested rule sheets meet the conditions without any modifications. GDL to C++ (GTC) translation scheme proves valuable and shows a substantial improvement over the original approach from~\cite{cpp09}, with performance gains ranging from 28\% to 7300\%, with 2300\% on average over YAP Prolog. However, the complex games - while managing to undergo translation and compilation under few seconds - are the ones show the smallest improvements. Still, GTC is competitive with other leading methods of performance reasoning for GDL rule sheets~\cite{propnets09, instantiating10}. 

The agent also features linked transposition tables - self-managing transposition tables, which seem to suit MCTS's adaptive searching nature. Linked TTs bring slight performance gains over the very basic design from~\cite{holt08}, but only under restrictive size limits. Nonetheless, the agent might act either as a thorough problem solver, when run with large clocks and TT, or as a compact agent doing shallow search within small memory, which only adds to it's versatility.
\\

Learning the game before conducting a match has been successfully realized through the evolutionary algorithm extending SGA~\cite{Goldberg:1989:GAS:534133}, knowledge mining through $\phi$ coefficient with modified Apriori algorithm~\cite{Agrawal:1994:FAM:645920.672836} based loosely on~\cite{negativeAssociationRules04}, and spatial board knowledge. Despite low, yet steady quality gains of slightly over 50\% to 70\% matches won against the baseline agent, the real potential is much greater. Spatial features might give satisfactory results as soon as in the first generation. With efficient implementation, more feature classes and accurate system of scoring the features and choosing the best ones, the knowledge-based players should largely outperform bare UCT ones. One other interesting phenomenon is the discrepancy of evolved UCT constants: they range from 15.4 to 56.7, with an average value of 35.98. As for tree expanding policies, there seems to be no universal tendency towards more frequent or seldom expands.

Game knowledge relies on the concept of a board, which in turn requires yet another small modification to GDL, allowing to supply game meta data. Feature system, while simple, allows abstraction of move characteristics inspired by computer Go~\cite{KNearestPatternsGo05, InriaMCTSPatternGo06, progressiveMCTSGo08}, otherwise hard to grasp. Only few feature and meta data classes were described and implemented; in reality, tens or even hundreds of features might be added to the system.
\\

%... SGA-based evolutionary algorithm for conducting self-plays and iteratively improving upon them. It aids in developing game knowledge. Also, mining methods are described: features are selected based on their correlation with winning or loosing (expressed through the $\phi$ coefficient). They can be backed with frequent fact sets, mined through a variation of Apriori algorithm - one which mines sets frequent in one database, while infrequent in other at the same time.
%\\

As the tests indicate, GGP Spatium meets the criteria formulated as the problem addressed in this publication. The framework presents a different approach than existing ones~\cite{rlggp13, ggpGalaxy13}, and constitutes an extensive platform for developing game playing agents and problem solvers. A few weakness of GDL have been revealed during tests: the lack of meta data mechanism, arithmetics, or perhaps hasty inclusion of function constants, which unnecessarily complexes the language. Performance of reasoning machines, while getting better, still falls behind custom-written code. It could be more convenient to supply rule sheets along with pre-written libraries with bindings for the most popular programming languages. This way, the quality of plays should substantially improve, but more importantly, it would refocus research on different methods closer to real-world applications. Lastly, the mechanism of guiding the machine in abstracting information with carefully crafted features seems to have potential; the system presented in this publication is small and simple, but could be easily extended, mimicking various human abstraction skills.

\section{Perspectives}

This publications, through the proposed approach and obtained results, points interesting perspectives for further research.

\paragraph{Further GDL development}
GDL would benefit greatly from simple additions like arithmetics or a system of meta data. Such meta data might be easily added and maintained by the GGP community. Agents, which do not support such data, would simply ignore it. The data might describe various properties: from the mentioned boards to branching factors, different move classes, and other characteristics. To keep things simple, all meta data might be supplied within a single new restricted functor, i.e. \texttt{meta/1}.

\paragraph{High-performance reasoning engines}
There is also much room for improvement in reasoning engines. Instantiation is certainly of great value, since the instantiated output - being itself valid GDL - could also be supplied to other engines, i.e. to GTC. Partial instantiation heuristics could play great role in such methods. Another directions is parallelization - not only for many CPUs, but also GPUs. A temporary solution might be, however, as simple as providing human-written libraries for rapid state switching, along with associated rule sheets.

\paragraph{Extensive system of spatial features}
Spatial features clearly have potential for improving agents' performance. It might be the case, that only the initial population would suffice to develop good enough knowledge, without doing much costly computations. The spatial feature system requires a much thorough investigation, as this publication nearly outlines the idea. More feature classes should be implemented, and a system estimating each feature's usefulness might be more accurate, than the genetic algorithm which scores the agent merely as a black box. Also, a proper implementation taking advantage i.e. of GTC's memory-efficient internal state representation is crucial for further development.

\paragraph{On-demand AI}
GGP has brought a vision of AIs, created on-demand for specific problems, closer than even before. GGP Spatium also attempts on being a toolkit that pursues this point of view. It makes possible to evolve, generate and compile the agent itself entirely as a library, without any external dependencies - a network-interfaced, on-demand problem solver.

\appendix
\chapter{Rule Sheets with the Euclidean Metric Extension}
\label{apdx:rulesheets}

\section{The Game of Amazons}
The Game of Amazons has been invented in 1988 by Walter Zamkauskas. It takes place on a $10\times 10$ board. Each player controls four pieces, initially placed symmetrically in designated places along board edges. Players alternate taking turns; each turn consists of moving a piece and shooting an arrow. A piece can do a queen move (that is, make either a horizontal, vertical or a diagonal move along empty fields). An arrow is an additional piece, which starts it's move on the field of a firing, non-arrow piece. The arrow can also make a queen move. As the game continues, more arrows accumulate, limiting the possibilities of making valid moves; the first player, which cannot make a valid move, loses.

\begin{scriptsize}
    \begin{multicols}{2}
\begin{verbatim}
; amazons (from http://games.stanford.edu:4441/
; spectator/Showallrules?Match=AMAZONS)

(boarddimnum 2)
(boardboundaries 1 10)
(boardfunctor cellholds)
(boardpattern dim dim piece)
(playfunctor play)
(playpattern skip skip dim dim piece)

(role white)
(role black)
(init (cellholds 4 1 white))
(init (cellholds 7 1 white))
(init (cellholds 1 4 white))
(init (cellholds 10 4 white))
(init (cellholds 1 7 black))
(init (cellholds 10 7 black))
(init (cellholds 4 10 black))
(init (cellholds 7 10 black))
(init (control white))
(init (movetype white moves))

(<= (legal ?player noop)
    (role ?player)
    (not (true (control ?player))))

; shoot
(<= (legal ?player (play ?x1 ?y1 ?x2 ?y2 arrow))
    (role ?player)
    (true (control ?player))
    (true (movetype ?player shoots))
    (true (cellholds ?x1 ?y1 ?player))
    (queenmove ?x1 ?y1 ?x2 ?y2))

; move
(<= (legal white (play ?x1 ?y1 ?x2 ?y2 white))
    (role white)
    (true (control white))
    (true (movetype white moves))
    (true (cellholds ?x1 ?y1 white))
    (queenmove ?x1 ?y1 ?x2 ?y2))

(<= (legal black (play ?x1 ?y1 ?x2 ?y2 black))
    (role black)
    (true (control black))
    (true (movetype black moves))
    (true (cellholds ?x1 ?y1 black))
    (queenmove ?x1 ?y1 ?x2 ?y2))

(<= (next (control white))
    (true (control black))
    (true (movetype black shoots)))

(<= (next (control white))
    (true (control white))
    (true (movetype white moves)))

(<= (next (control black))
    (true (control white))
    (true (movetype white shoots)))

(<= (next (control black))
    (true (control black))
    (true (movetype black moves)))

(<= (next (movetype white shoots))
    (true (movetype white moves)))

(<= (next (movetype black moves))
    (true (movetype white shoots)))

(<= (next (movetype black shoots))
    (true (movetype black moves)))

(<= (next (movetype white moves))
    (true (movetype black shoots)))

; pieces not involved in the move from previous turn
(<= (next (cellholds ?xs ?ys ?state))
    (true (cellholds ?xs ?ys ?state))
    (does ?player (play ?x1 ?y1 ?x2 ?y2 ?piece))
    (distinctcell ?xs ?ys ?x1 ?y1))

; old position persists if it was a shoot turn
(<= (next (cellholds ?xs ?ys ?state))
    (true (cellholds ?xs ?ys ?state))
    (does ?player (play ?xs ?ys ?x2 ?y2 arrow)))

; target position persists no matter what it was
(<= (next (cellholds ?x2 ?y2 ?piece))
    (does ?player (play ?xs ?ys ?x2 ?y2 ?piece)))

(<= terminal
    (not (haslegalmove white)))
(<= terminal
    (not (haslegalmove black)))
(<= (goal white 100)
    (not (haslegalmove black)))
(<= (goal white 0)
    (not (haslegalmove white)))
(<= (goal black 100)
    (not (haslegalmove white)))
(<= (goal black 0)
    (not (haslegalmove black)))
(<= (haslegalmove ?player)
    (role ?player)
    (true (cellholds ?x ?y ?player))
    (queenmove ?x ?y ?xany ?yany))

(<= (queenmove ?x1 ?y1 ?x2 ?y2)
    (horizontalmove ?x1 ?y1 ?x2 ?y2))
(<= (queenmove ?x1 ?y1 ?x2 ?y2)
    (horizontalleftmove ?x1 ?y1 ?x2 ?y2))
;(<= (queenmove ?x1 ?y1 ?x2 ?y2)
;    (horizontalmove ?x2 ?y2 ?x1 ?y1))

(<= (queenmove ?x1 ?y1 ?x2 ?y2)
    (verticalmove ?x1 ?y1 ?x2 ?y2))
(<= (queenmove ?x1 ?y1 ?x2 ?y2)
    (verticaldownmove ?x1 ?y1 ?x2 ?y2))
;(<= (queenmove ?x1 ?y1 ?x2 ?y2)
;    (verticalmove ?x2 ?y2 ?x1 ?y1))

(<= (queenmove ?x1 ?y1 ?x2 ?y2)
    (diagonalmovene ?x1 ?y1 ?x2 ?y2))
(<= (queenmove ?x1 ?y1 ?x2 ?y2)
    (diagonalmovenw ?x1 ?y1 ?x2 ?y2))
;(<= (queenmove ?x1 ?y1 ?x2 ?y2)
;    (diagonalmovene ?x2 ?y2 ?x1 ?y1))

(<= (queenmove ?x1 ?y1 ?x2 ?y2)
    (diagonalmovese ?x1 ?y1 ?x2 ?y2))
(<= (queenmove ?x1 ?y1 ?x2 ?y2)
    (diagonalmovesw ?x1 ?y1 ?x2 ?y2))
;(<= (queenmove ?x1 ?y1 ?x2 ?y2)
;    (diagonalmovese ?x2 ?y2 ?x1 ?y1))

(<= (horizontalmove ?x1 ?y ?x2 ?y)
    (++ ?x1 ?x2)
    (index ?y)
    (emptycell ?x2 ?y))
(<= (horizontalmove ?x1 ?y ?x3 ?y)
    (++ ?x1 ?x2)
    (index ?y)
    (emptycell ?x2 ?y)
    (horizontalmove ?x2 ?y ?x3 ?y))

(<= (horizontalleftmove ?x1 ?y ?x2 ?y)
    (-- ?x1 ?x2)
    (index ?y)
    (emptycell ?x2 ?y))
(<= (horizontalleftmove ?x1 ?y ?x3 ?y)
    (-- ?x1 ?x2)
    (index ?y)
    (emptycell ?x2 ?y)
    (horizontalleftmove ?x2 ?y ?x3 ?y))

(<= (verticalmove ?x ?y1 ?x ?y2)
    (++ ?y1 ?y2)
    (index ?x)
    (emptycell ?x ?y2))
(<= (verticalmove ?x ?y1 ?x ?y3)
    (++ ?y1 ?y2)
    (index ?x)
    (emptycell ?x ?y2)
    (verticalmove ?x ?y2 ?x ?y3))

(<= (verticaldownmove ?x ?y1 ?x ?y2)
    (-- ?y1 ?y2)
    (index ?x)
    (emptycell ?x ?y2))
(<= (verticaldownmove ?x ?y1 ?x ?y3)
    (-- ?y1 ?y2)
    (index ?x)
    (emptycell ?x ?y2)
    (verticaldownmove ?x ?y2 ?x ?y3))

(<= (diagonalmovene ?x1 ?y1 ?x2 ?y2)
    (++ ?x1 ?x2)
    (++ ?y1 ?y2)
    (emptycell ?x2 ?y2))
(<= (diagonalmovene ?x1 ?y1 ?x3 ?y3)
    (++ ?x1 ?x2)
    (++ ?y1 ?y2)
    (emptycell ?x2 ?y2)
    (diagonalmovene ?x2 ?y2 ?x3 ?y3))

(<= (diagonalmovenw ?x1 ?y1 ?x2 ?y2)
    (-- ?x1 ?x2)
    (++ ?y1 ?y2)
    (emptycell ?x2 ?y2))
(<= (diagonalmovenw ?x1 ?y1 ?x3 ?y3)
    (-- ?x1 ?x2)
    (++ ?y1 ?y2)
    (emptycell ?x2 ?y2)
    (diagonalmovenw ?x2 ?y2 ?x3 ?y3))

(<= (diagonalmovese ?x1 ?y1 ?x2 ?y2)
    (++ ?x1 ?x2)
    (++ ?y2 ?y1)
    (emptycell ?x2 ?y2))
(<= (diagonalmovese ?x1 ?y1 ?x3 ?y3)
    (++ ?x1 ?x2)
    (++ ?y2 ?y1)
    (emptycell ?x2 ?y2)
    (diagonalmovese ?x2 ?y2 ?x3 ?y3))

(<= (diagonalmovesw ?x1 ?y1 ?x2 ?y2)
    (-- ?x1 ?x2)
    (++ ?y2 ?y1)
    (emptycell ?x2 ?y2))
(<= (diagonalmovesw ?x1 ?y1 ?x3 ?y3)
    (-- ?x1 ?x2)
    (++ ?y2 ?y1)
    (emptycell ?x2 ?y2)
    (diagonalmovesw ?x2 ?y2 ?x3 ?y3))

(<= (emptycell ?x ?y)
    (cell ?x ?y)
    (not (true (cellholds ?x ?y arrow)))
    (not (true (cellholds ?x ?y white)))
    (not (true (cellholds ?x ?y black))))
(<= (distinctcell ?x1 ?y1 ?x2 ?y2)
    (cell ?x1 ?y1)
    (cell ?x2 ?y2)
    (distinct ?x1 ?x2))
(<= (distinctcell ?x1 ?y1 ?x2 ?y2)
    (cell ?x1 ?y1)
    (cell ?x2 ?y2)
    (distinct ?y1 ?y2))
(<= (cell ?x ?y)
    (index ?x)
    (index ?y))
(index 1)
(index 2)
(index 3)
(index 4)
(index 5)
(index 6)
(index 7)
(index 8)
(index 9)
(index 10)
(++ 1 2)
(++ 2 3)
(++ 3 4)
(++ 4 5)
(++ 5 6)
(++ 6 7)
(++ 7 8)
(++ 8 9)
(++ 9 10)
(-- 2 1) 
(-- 3 2) 
(-- 4 3) 
(-- 5 4) 
(-- 6 5) 
(-- 7 6) 
(-- 8 7) 
(-- 9 8) 
(-- 10 9)
\end{verbatim}
    \end{multicols}
\end{scriptsize}

\section{Cubicup}

Each player has 28 pieces. Players take turns while building a pyramid, whose base is a triangle (side length of 6 pieces). Small cubes are used as pieces in such way, that a cube can be placed on top of three other ones. During his turn, a player has to place a cube. If he creates a special formation of three cubes of his color on the same level, called \textit{chalice}, his opponent has to first place a cube on top the chalice, before making his move, thus loosing an additional cube. The first player to run out of cubes loses.

\begin{scriptsize}
    \begin{multicols}{2}
\begin{verbatim}

(boarddimnum 3)
(boardboundaries 0 6)
(boardfunctor cube)
(boardpattern dim dim dim piece)
(playfunctor put)
(playpattern dim dim dim piece)

(role yellow)
(role red)

;;;;;;;;;;;;;;

(init (control yellow))

(init (cubes yellow 28))
(init (cubes red 28))

(init (cube 0 0 0 base))
(init (cube 1 0 0 base))
(init (cube 1 1 0 base))
(init (cube 2 0 0 base))
(init (cube 2 1 0 base))
(init (cube 2 2 0 base))
(init (cube 3 0 0 base))
(init (cube 3 1 0 base))
(init (cube 3 2 0 base))
(init (cube 3 3 0 base))
(init (cube 4 0 0 base))
(init (cube 4 1 0 base))
(init (cube 4 2 0 base))
(init (cube 4 3 0 base))
(init (cube 4 4 0 base))
(init (cube 5 0 0 base))
(init (cube 5 1 0 base))
(init (cube 5 2 0 base))
(init (cube 5 3 0 base))
(init (cube 5 4 0 base))
(init (cube 5 5 0 base))
(init (cube 6 0 0 base))
(init (cube 6 1 0 base))
(init (cube 6 2 0 base))
(init (cube 6 3 0 base))
(init (cube 6 4 0 base))
(init (cube 6 5 0 base))
(init (cube 6 6 0 base))

(lastcube 6 6 6)

;;;;;;;;;;;;;;

(<= (legal ?player (put ?x ?y ?z ?p))
	(true (control ?player))
	(true (control ?p))
	(opponent ?player ?opponent)
	(open_cubicup ?opponent ?x ?y ?z))

(<= (legal ?player (put ?x ?y ?z ?p))
	(true (control ?player))
	(true (control ?p))
	(opponent ?player ?opponent)
	(not (any_open_cubicup ?opponent))
	(filled ?xp ?yp ?zp)
	(succ ?xp ?x)
	(filled ?x ?yp ?zp)
	(succ ?yp ?y)
	(filled ?x ?y ?zp)
	(succ ?zp ?z)
	(not (filled ?x ?y ?z)))
	
(<= (legal ?player noop)
	(opponent ?player ?opponent)
	(true (control ?opponent)))
;;;;;;;;;;;;;;

(<= (cubicup ?player ?x ?y ?z)
	(true (cube ?xp ?yp ?zp ?player))
	(succ ?xp ?x)
	(true (cube ?x ?yp ?zp ?player))
	(succ ?yp ?y)
	(true (cube ?x ?y ?zp ?player))
	(succ ?zp ?z))

(<= (open_cubicup ?player ?x ?y ?z)
	(cubicup ?player ?x ?y ?z)
	(not (filled ?x ?y ?z)))

(<= (any_open_cubicup ?player)
	(open_cubicup ?player ?x ?y ?z))

;;;;;;;;;;;;;;

(<= (next (control ?opponent)) 
	(true (control ?player))
	(opponent ?player ?opponent)
	(not (any_open_cubicup ?opponent))
	(not (true (cubes ?opponent 0))))

(<= (next (control ?player))
	(true (control ?player))
	(opponent ?player ?opponent)
	(true (cubes ?opponent 0)))

(<= (next (control ?player))
	(true (control ?player))
	(opponent ?player ?opponent)
	(any_open_cubicup ?opponent))

;;;;;;;;;;;;;;

(<= (next (cube ?x ?y ?z ?color))
	(true (cube ?x ?y ?z ?color)))

(<= (next (cube ?x ?y ?z ?player))
	(does ?player (put ?x ?y ?z ?p)))

;;;;;;;;;;;;;;

(<= (next (cubes ?player ?n))
	(true (cubes ?player ?n))
	(not (true (control ?player))))

(<= (next (cubes ?player ?n))
	(true (cubes ?player ?n1))
	(true (control ?player))
	(succ ?n ?n1))

;;;;;;;;;;;;;;;

(<= (goal ?player 100)
	(lastcube ?x ?y ?z)
	(true (cube ?x ?y ?z ?player))
	(opponent ?player ?player2)
	(not (cubicup ?player2 ?x ?y ?z)))

(<= (goal ?w 50)
	(role ?w)
	(lastcube ?x ?y ?z)
	(true (cube ?x ?y ?z ?player1))
	(opponent ?player1 ?player2)
	(cubicup ?player2 ?x ?y ?z))

(<= (goal ?player2 0)
	(lastcube ?x ?y ?z)
	(true (cube ?x ?y ?z ?player))
	(opponent ?player ?player2)
	(not (cubicup ?player2 ?x ?y ?z)))

(<= terminal
	(lastcube ?x ?y ?z)
	(filled ?x ?y ?z))
	
;;;;;;;;;;;;;;;

(<= (filled ?x ?y ?z)
	(true (cube ?x ?y ?z ?color)))

;;;;;;;;;;;;;;;

(opponent yellow red)
(opponent red yellow)

(succ 0 1)
(succ 1 2)
(succ 2 3)
(succ 3 4)
(succ 4 5)
(succ 5 6)
(succ 6 7)
(succ 7 8)
(succ 8 9)
(succ 9 10)
(succ 10 11)
(succ 11 12)
(succ 12 13)
(succ 13 14)
(succ 14 15)
(succ 15 16)
(succ 16 17)
(succ 17 18)
(succ 18 19)
(succ 19 20)
(succ 20 21)
(succ 21 22)
(succ 22 23)
(succ 23 24)
(succ 24 25)
(succ 25 26)
(succ 26 27)
(succ 27 28)
\end{verbatim}
    \end{multicols}
\end{scriptsize}

\chapter{UML Diagrams of Key System Components}
\label{apdx:uml}

Chapter~\ref{chap:testingPlatform} outlines GGP Spatium with it's components; Layer diagrams are presented in Figure~\ref{fig:arch}, to give a sens of mutual relations between them. To understand the applications on the code level, UML diagrams are particularly useful. This appendix gives basic diagrams for GGP Agent (Figure~\ref{fig:ggp-agent-uml}), Knowledge Miner (Figure~\ref{fig:analyzer-uml}) and GTC Generator (Figure~\ref{fig:gtc-uml}). The diagrams are stripped of methods, showing only the most significant classes.

\vspace{1cm}

\begin{figure}[htbp]
  \centering
  \includegraphics[scale=0.3]{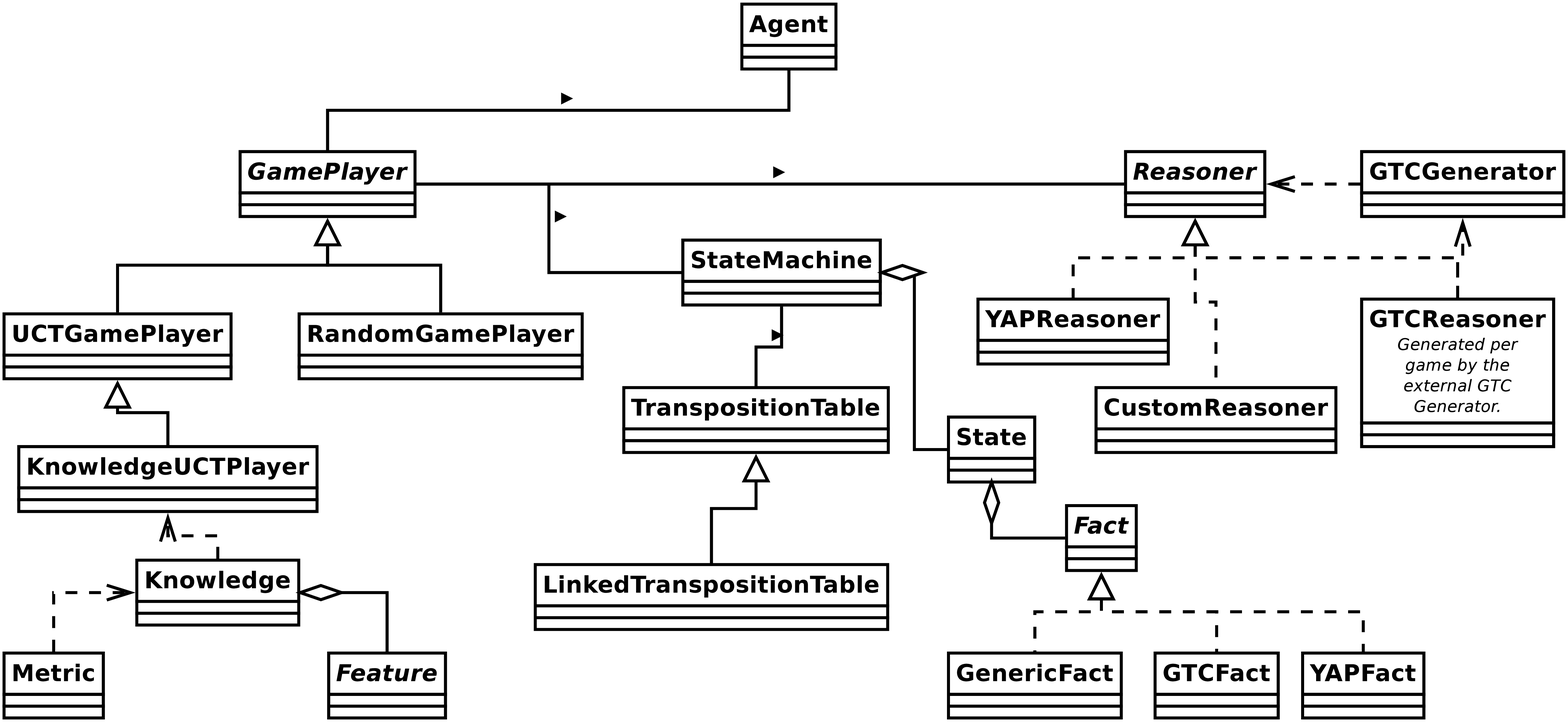}
  \caption{Partial UML class diagram of GGP Agent - part of GGP Spatium framework. It is possible to extend the player with custom \texttt{GamePlayer} and \texttt{Reasoner} classes. The class inheriting from \texttt{Reasoner} should supply it's own \texttt{Fact} extension as well, or use the \texttt{GenericFact} class. The \texttt{Knowledge} class aggregates \texttt{Feature} instances in two ways: feature classes are being registered, and particular instances (with different parameters) are aggregated for scoring moves.}
  \label{fig:ggp-agent-uml}
\end{figure}

\begin{figure}[htbp]
	\centering
	\includegraphics[scale=0.3]{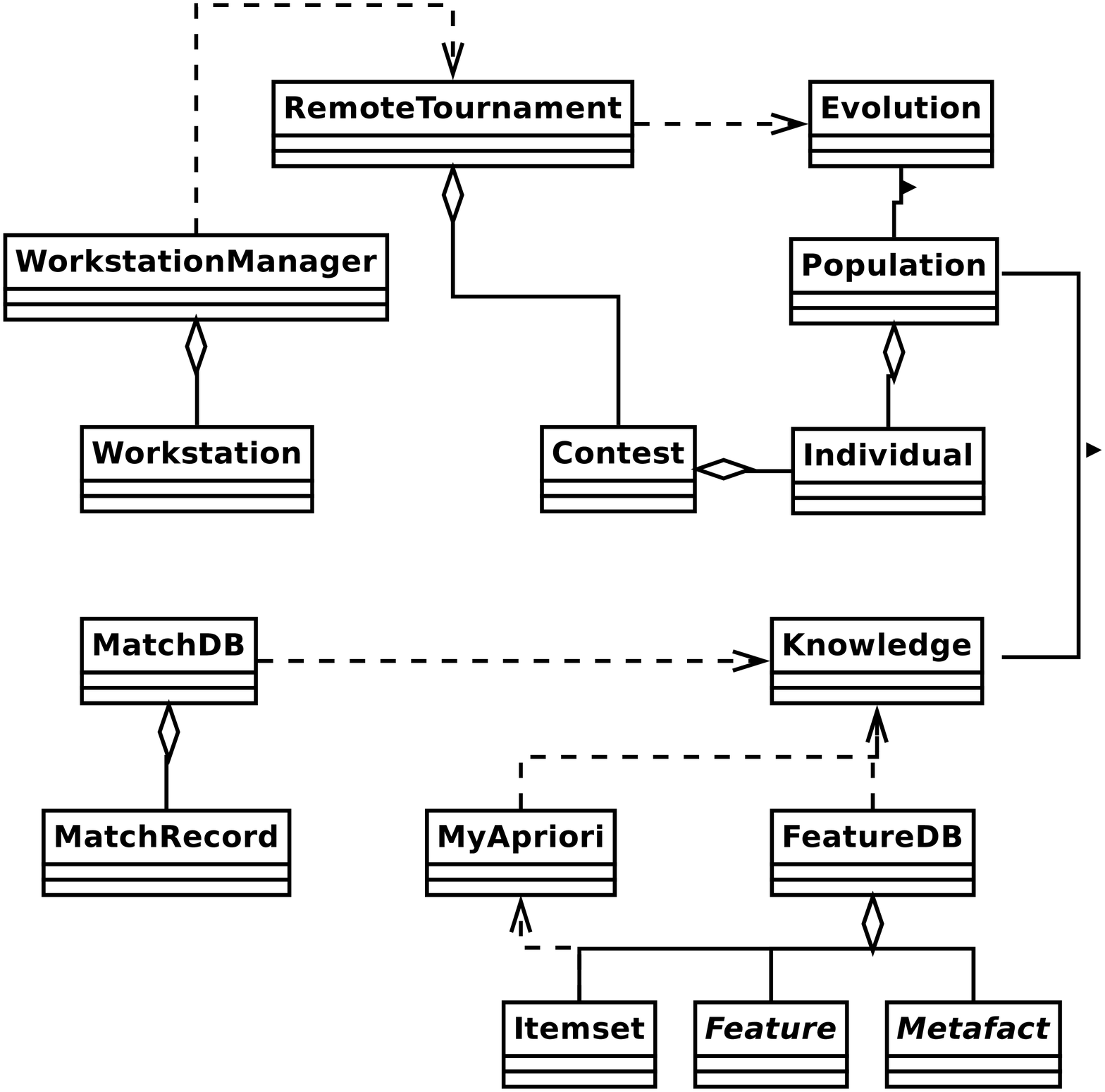}
	
\caption{Partial UML class diagram of Knowledge Miner - part of GGP Spatium framework. \texttt{Evolution} and \texttt{Population} are the top-level classes. \texttt{RemoteTournament} consists of matches, delegated to remote workstations. The \texttt{Knowledge} module is responsible for both triggering knowledge mining and providing particular knowledge instances.}
\label{fig:analyzer-uml}
\end{figure}

\begin{figure}[htbp]
	\centering
	\includegraphics[scale=0.3]{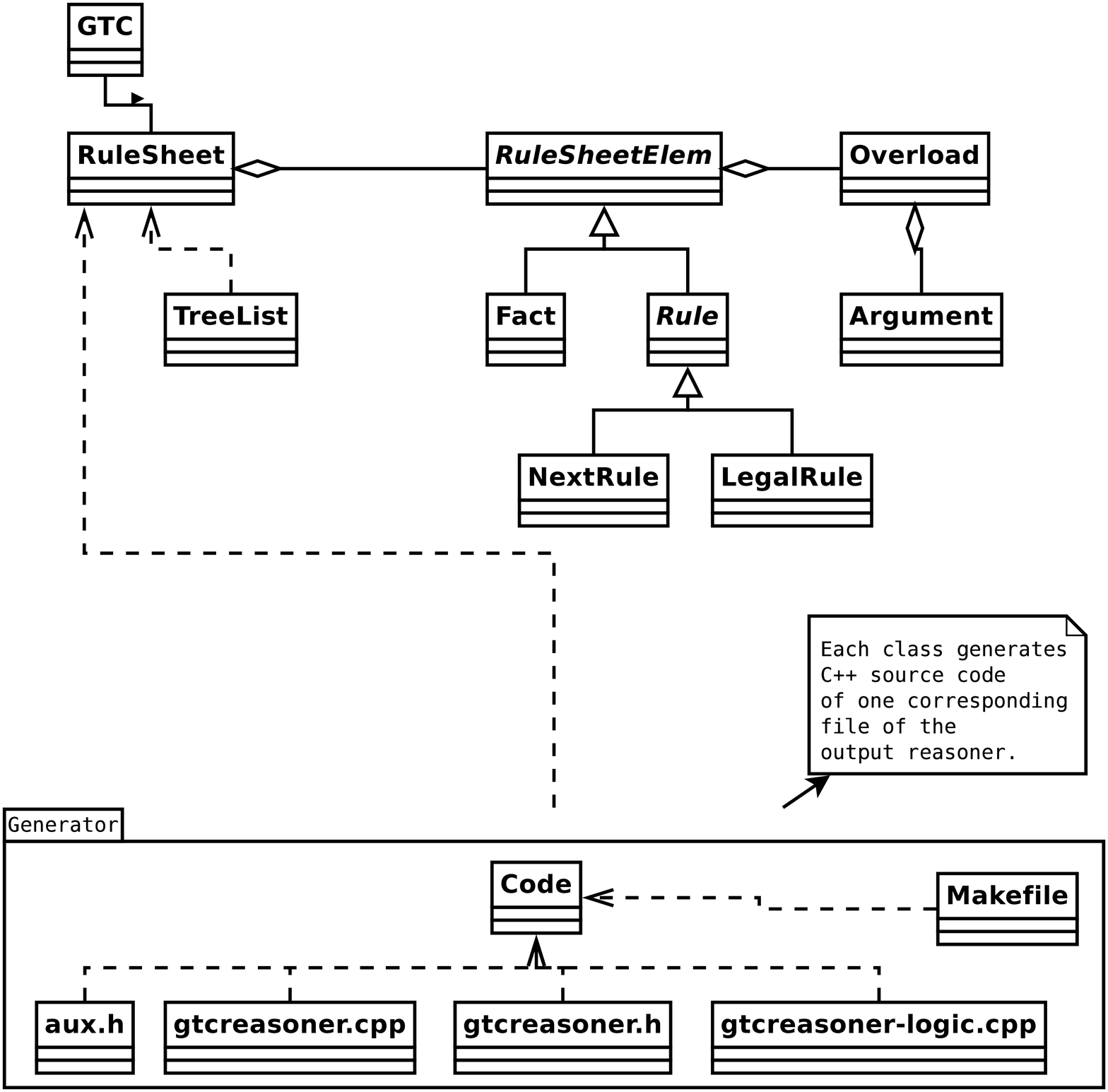}
	
\caption{Partial UML class diagram of GTC Generator - part of GGP Spatium framework. A GDL rule sheet is divided into separate expressions (trees), which undergo basic transformations (name mangling, handling of \texttt{not} and \texttt{or} functors, etc.), to finally have the expressions accurately classified in a \texttt{RuleSheet} object. Such an object serves as the basic input for the \texttt{Generator} package, where each module is responsible for a single output file. The \texttt{Code} class aids in code generation by providing methods for controlling loops and indentation.}
\label{fig:gtc-uml}
\end{figure}

\cleardoublepage
%\addcontentsline{toc}{section}{Literatura}
\addcontentsline{toc}{chapter}{Bibliography}
\bibliography{bibliography}{}
\bibliographystyle{plain}

\end{document}